\documentclass[twoside,11pt]{article}
\usepackage{jmlr2e}
\usepackage{natbib}
\usepackage{amssymb, amsmath}
\usepackage{verbatim, float}
\usepackage{mathrsfs}
\usepackage{graphics}
\usepackage{subfig}
\usepackage[usenames,dvipsnames]{pstricks}
\usepackage{epsfig}
\usepackage{pst-grad} 
\usepackage{pst-plot} 
\usepackage{cases}
\usepackage{multirow}
\usepackage{url}
\usepackage{array}
\usepackage{tabu}
\usepackage{hyperref}
\hypersetup{breaklinks=true,colorlinks=true,linkcolor=black,citecolor=red}
\usepackage{tabu}
\usepackage{lastpage}

\ShortHeadings{Quadratic Decomposable Submodular Function Minimization}{Li, He, and Milenkovic}
\firstpageno{1}

\begin{document}


\title{Quadratic Decomposable Submodular Function Minimization: Theory and Practice}

\author{\name Pan Li \email panli@purdue.edu\\
       \addr Department of Computer Science\\
       Purdue University\\
       West Lafayette, IN 47907, USA
       \AND
       \name Niao He \email niaohe@illinois.edu\\
       \addr Department of Industrial and Enterprise Systems Engineering\\
       University of Illinois at Urbana-Champaign\\
       Champaign, IL 61820, USA
       \AND
       \name Olgica Milenkovic \email milenkov@illinois.edu \\
       \addr Department of Electrical and Computer Engineering\\
       University of Illinois at Urbana-Champaign\\
       Champaign, IL 61820, USA\\
       }
       
\editor{Andreas Krause}

\maketitle

\begin{abstract}
We introduce a new convex optimization problem, termed \emph{quadratic decomposable submodular function minimization} (QDSFM), which allows to model a number of learning tasks on graphs and hypergraphs. The problem exhibits close ties to decomposable submodular function minimization (DSFM) yet is much more challenging to solve. We approach the problem via a new dual strategy and formulate an objective that can be optimized through a number of double-loop algorithms. The outer-loop uses either random coordinate descent (RCD) or alternative projection (AP) methods, for both of which we prove linear convergence rates. The inner-loop computes projections onto cones generated by base polytopes of the submodular functions via the modified min-norm-point or Frank-Wolfe algorithms. We also describe two new applications of QDSFM: hypergraph-adapted PageRank and semi-supervised learning. The proposed hypergraph-based PageRank algorithm can be used for local hypergraph partitioning and comes with provable performance guarantees. For hypergraph-adapted semi-supervised learning, we provide numerical experiments demonstrating the efficiency of our QDSFM solvers and their significant improvements on prediction accuracy when compared to state-of-the-art methods.
\end{abstract}


\section{Introduction}
Given $[N]=\{1,2,...,N\}$, a submodular function $F: 2^{[N]} \rightarrow \mathbb{R}$ is a set function that for any $S_1, S_2 \subseteq [N]$ satisfies $F(S_1) + F(S_2) \geq F(S_1\cup S_2) + F(S_1 \cap S_2)$~\citep{edmonds2003submodular}. Submodular functions are ubiquitous in machine learning as they capture rich combinatorial properties of set functions and provide useful regularizers for supervised and unsupervised learning objectives~\citep{bach2013learning}. Submodular functions are endowed with convex Lov{\'a}sz extensions~\citep{lovasz1983submodular}, which provide the means to establish connections between combinatorial and continuous optimization problems~\citep{fujishige2005submodular}. 

Due to their versatility, submodular functions and their Lov{\'a}sz extensions are frequently used in applications such as learning on directed/undirected graphs and hypergraphs~\citep{zhu2003semi,zhou2004learning,hein2013total}, image denoising via total variation regularization~\citep{osher2005iterative, chambolle2009total}, and maximum a posteriori probability (MAP) inference in high-order Markov random fields~\citep{kumar2003discriminative}. 
In many optimization settings involving submodular functions, one encounters the convex program
\begin{align} \label{eq:problem}
\min_{x} \sum_{i\in [N]}(x_i - a_i)^2 + \sum_{r\in[R]} \left[f_r(x)\right]^p,
\end{align}
where $a\in \mathbb{R}^N$, $p\in\{1,2\}$, and for all $r$ in some index set $[R]$, $f_r$ is the Lov{\'a}sz extension of a submodular function $F_r$ that describes some combinatorial structure over the set $[N]$. For example, in image denoising, the parameter $a_i$ may correspond to the observed value of pixel $i$, while the functions $\left[f_r(x)\right]^p$ may be used to impose smoothness constraints on the pixel neighborhoods. One of the main difficulties in solving the above optimization problem comes from the nondifferentiability of the second term: a direct application of subgradient methods leads to convergence rates as slow as $1/\sqrt{k}$, where $k$ denotes the number of iterations~\citep{Shor:1985:MMN:3585}. 

In recent years, the above optimization problem for $p=1$, referred to as \emph{decomposable submodular function minimization (DSFM)~\citep{stobbe2010efficient}}, has received significant attention. The motivation for studying this particular setup is two-fold: first, solving the convex optimization problem directly recovers the combinatorial solution to the submodular min-cut problem~\citep{fujishige2005submodular} which takes the form $\min_{S\subseteq [N]} F(S)$, where
$$F(S)=\sum_{r\in [R]} F_r(S) - 2\sum_{i\in S}a_{i}.$$
Second, minimizing a submodular function decomposed into a sum of simpler components  $F_r$, $r\in [R],$ is significantly easier than minimizing an unrestricted submodular function $F$ over a large set $[N]$.  
There are several milestone results for solving the DSFM problem:~\citet{jegelka2013reflection} first tackled the problem by considering its dual and proposed a solver based on Douglas-Rachford splitting.~\citet{nishihara2014convergence} established linear convergence rates of alternating projection methods for solving the dual problem.~\citet{ene2015random,ene2017decomposable} presented linear convergence rates of coordinate descent methods and subsequently tightened known results via the use of submodular flows. The work in~\citet{li2018revisiting} improved previous methods by leveraging incidence relations between the arguments of submodular functions and their constituent parts.

We focus on the case $p=2$ and refer to the underlying optimization problem as \emph{quadratic DSFM} (QDSFM). QDSFM appears naturally in a wide spectrum of applications, including learning on graphs and hypergraphs, and in particular, semi-supervised learning and PageRank analysis, to be described in more detail later. Moreover, it has been demonstrated both theoretically~\citep{johnson2007effectiveness} and empirically~\citep{zhou2004learning,hein2013total} that using quadratic regularizers in~\eqref{eq:problem} improves the predictive performance for semi-supervised learning when compared to the case $p=1$. 
Despite the importance of QDSFM, it has not received the same level of attention as DSFM, neither from the theoretical nor  algorithmic perspective. To our best knowledge, only a few reported works~\citep{hein2013total,zhang2017re} provided solutions for \emph{specific instances} of QDSFM with sublinear convergence guarantees. For the general QDSFM problem, no analytical results are currently known.

Our work takes a substantial step towards solving the QDSFM problem in its most general form by developing a family of algorithms with \emph{linear convergence rates} and \emph{small iteration cost}, including the randomized coordinate descent (RCD) and alternative projection (AP) algorithms. 
Our contributions are as follows. 
\begin{enumerate}
\item First, we derive a new dual formulation for the QDSFM problem since an analogue of the dual transformation for the DSFM problem is not applicable. Interestingly, the dual QDSFM problem requires one to find the best approximation of a hyperplane via a product of cones as opposed to a product of polytopes, encountered in the dual DSFM problem. 
\item Second, we develop linearly convergent RCD and AP algorithms for solving the dual QDSFM problem. Because of the special underlying conic structure, a new analytic approach is needed to prove the weak strong-convexity of the dual QDSFM problem, which essentially guarantees linear convergence. 
\item Third, we develop generalized Frank-Wolfe (FW) and min-norm-point methods for efficiently computing the conic projection required in each step of RCD and AP and provide a $1/k$-rate convergence analysis. These FW-type algorithms and their corresponding analyses do not rely on the submodularity assumption and apply to general conic projections, and are therefore of independent interest.
\item Fourth, we introduce a novel application of the QDSFM problem, related to the PageRank (PR) process for hypergraphs. The main observation is that PR state vectors can be efficiently computed by solving QDSFM problems. We also show that many important properties of the PR process for graphs carry over to the hypergraph setting, including mixing and local hypergraph partitioning results.
\item Finally, we evaluate our methods on semi-supervised learning over hypergraphs using synthetic and real datasets, and demonstrate superior performance both in convergence rate and prediction accuracy compared to existing general-purpose methods.
\end{enumerate}

\section{Problem Formulations}
We first introduce some terminology and notation useful for our subsequent exposition.

\paragraph{Lov{\'a}sz extension and base polytope.} For a submodular function $F$ defined over the ground set $[N]$, the \emph{Lov{\'a}sz extension} is a convex function $f: \mathbb{R}^{N} \rightarrow \mathbb{R},$ defined for all $x\in\mathbb{R}^{N}$ according to
\begin{align}\label{lovasext}
f(x) = \sum_{k=1}^{N-1} F(\{i_1,...,i_k\})(x_{i_k} - x_{i_{k+1}})  + F([N])x_{i_{N}},
\end{align} 
where $x_{i_1} \geq x_{i_2} \geq \cdots \geq x_{i_{N}}$, and $i_1,...,i_N$ are distinct indices in $[N]$.  
For a vector $x\in\mathbb{R}^N$ and a set $S\subseteq[N]$, define $x(S) = \sum_{i\in S} x_i$, where $x_i$ stands for the value of the $i$-th dimension of $x$. The \emph{base polytope} of $F$, denoted by $B$, is defined as 
\begin{align}\label{basepolytope}
B = \{y\in \mathbb{R}^{N}| y(S) \leq F(S),\; \forall S\subset [N],\,y([N]) = F([N])\}.
\end{align} 
Using the base polytope, the Lov{\'a}sz extension can also be written as $f(x) =\max_{y\in B} \langle y, x\rangle.$ We say that an element $i\in [N]$ is \emph{incident} to $F$ if there exists a $S\subset [N]$ such that $F(S) \neq F(S \cup \{i\})$. 

\paragraph{Notations.} Given a positive diagonal matrix $W\in \mathbb{R}^{N\times N}$ and a vector $x\in \mathbb{R}^{N}$, we define the $W$-norm according to $\|x\|_{W} = \sqrt{\sum_{i=1}^N W_{ii} \, x_i^2},$ and simply use $\|\cdot\|$ when $W=I$, the identity matrix. 
 For an index set $[R]$, we denote the $R$-product of $N$-dimensional Euclidean spaces by $\otimes_{r\in[R]} \mathbb{R}^{N}$. A vector $y\in \otimes_{r\in[R]} \mathbb{R}^{N}$ is written as $(y_1, y_2,...,y_R),$ where $y_r\in \mathbb{R}^N$ for all $r \in [R]$. The $W$-norm induced on $\otimes_{r\in[R]} \mathbb{R}^{N}$ equals $\|y\|_{I(W)}= \sqrt{\sum_{r=1}^R \|y_r\|_W^2}$. We reserve the symbol $\rho$ for $\max_{y_r\in B_r, \forall r} \sqrt{\sum_{r\in [R]}\|y_r\|_1^2}$. Furthermore, we use $(x)_+$ to denote the function $\max\{x, 0\}$. 

\subsection{The QDSFM Problem}

Consider a collection of submodular functions $\{F_r\}_{r\in[R]}$ defined over the ground set $[N]$, and denote their Lov{\'a}sz extensions and base polytopes by $\{f_r\}_{r\in[R]}$ and $\{B_r\}_{r\in[R]}$, respectively. Let $S_r \subseteq [N]$ denote the set of variables incident to $F_r$ and assume that the functions $F_r$ are normalized and nonnegative, i.e., that $F_r(\emptyset) = 0$ and $F_r \geq 0$. These two mild constraints are satisfied by almost all submodular functions that arise in practice. Note that we do not impose monotonicity constraints which are frequently assumed and used in submodular function optimization; nevertheless, our setting naturally applies to the special case of monotonously increasing submodular functions.

The QDSFM problem is formally stated as follows:

\begin{align}\label{QDSFM}
\text{QDSFM:} \quad\quad \min_{x\in \mathbb{R}^N}\|x - a\|_W^2 + \sum_{r\in[R]}\left[f_r(x)\right]^2,
\end{align}
where $a\in \mathbb{R}^N$ is a given vector and $W\in\mathbb{R}^{N\times N}$ is a positive diagonal matrix. The QDSFM problem is convex because the  Lov{\'a}sz extensions $f_r$ are convex, and nonnegative. In fact, it is also strongly convex, implying that there exists a unique optimal solution, denoted by $x^*$. The quadratic form $\left[f_r(x)\right]^2$ may be generalized to $q(f_r(x)),$ where the mapping $q(\cdot): \mathbb{R}_{\geq 0} \rightarrow \mathbb{R}_{\geq 0}$ satisfies the monotonicity, strong convexity and smoothness conditions. In this case, subsequent results such as those pertaining to the transformation~\eqref{conjugate} may be obtained by using the convex conjugate of $q(\cdot)$, leading to similar conclusions as those described in Section 3 for the quadratic form. Our focus is on quadratic forms as they correspond to important applications described in Sections 5 and 6.

As a remark, in contrast to the DSFM problem, we are not aware of any discrete optimization problems that directly correspond to the QDSFM problem. However, as we will show in Section 5, QDSFM can be used to compute the PageRank vectors for hypergraphs which itself is a relevant discrete optimization question. Moreover, as will be demonstrated in subsequent sections, the obtained PageRank vectors are helpful for approximating the hypergraph conductance, another important discrete optimization objective.

\subsection{Dual Formulation}
Despite being convex, the objective is in general nondifferentiable. This implies that only sublinear convergence can be obtained when directly applying the subgradient method. To address this issue, we revert to the dual formulation of the QDSFM problem. A natural idea is to mimic the approach used for DSFM by exploiting the characterization of the Lov{\'a}sz extension, $f_r(x) = \max_{y_r\in B_r} \langle y_r, x\rangle$, $\forall r$. However, this leads to a semidefinite programing problem for the dual variables $\{y_r\}_{r\in[R]}$, which is too expensive to solve for large problems. Instead, we establish a new dual formulation that overcomes this obstacle. 

The dual formulation hinges upon the following simple yet key observation: 
\begin{align}\label{conjugate} 
[f_r(x)]^2 =  \max_{\phi_r\geq 0} \phi_r f_r(x) -  \frac{ \phi_r^2}{4}  = \max_{\phi_r\geq 0}\max_{y_r\in \phi_r B_r}\; \langle y_r, x\rangle - \frac{\phi_r^2}{4}.
\end{align}
Let $y = (y_{1}, y_2,...,y_R)$ and $\phi = (\phi_1, \phi_2, ..., \phi_R)$. For each $r$, we define a convex cone 
$$C_r = \{(y_r, \phi_r) | \phi_r \geq 0, y_r\in \phi_r B_r \}$$ 
which represents the feasible set of the variables 
$(y_r, \phi_r)$. Furthermore, we denote the product of cones by
$$\mathcal{C} = \bigotimes_{r\in [R]} C_r:=\{(y,\phi): \phi_r\geq0,y_r\in\phi_rB_r,\forall r\in[R]\}.$$ 
Invoking equation~\eqref{conjugate}, we arrive at the following two dual formulations for the original QDSFM problem in Lemma~\ref{dualform}.

\begin{lemma} \label{dualform} The following optimization problem is dual to \eqref{QDSFM}:
\begin{align}\label{CDform}
 \min_{y, \phi}\;g(y, \phi): = \|\sum_{r\in[R]} y_r - 2Wa\|_{W^{-1}}^2 + \sum_{r\in[R]} \phi_r^2, \;\text{s.t.}\; (y, \phi)\in \mathcal{C}.
\end{align}
By introducing $\Lambda = (\lambda_r) \in \bigotimes_{r\in [R]} \mathbb{R}^{N}$, the previous optimization problem takes the form  
\begin{align} \label{APform}
 \min_{y, \phi, \Lambda}\sum_{r\in [R]}\left[\| y_r - \frac{\lambda_r}{\sqrt{R}}\|_{W^{-1}}^2 + \phi_r^2\right], \;\text{s.t.}\; (y, \phi)\in \mathcal{C} ,\sum_{r\in [R]} \lambda_r = 2Wa. 
\end{align}
The primal variables in both cases may be computed as $x = a - \frac{1}{2}W^{-1}\sum_{r\in[R]} y_r$.
\end{lemma}

The proof of Lemma~\ref{dualform} is presented in Section~\ref{proof:DF}. The dual formulations for the DSFM problem were described in~Lemma 2 of~\citep{jegelka2013reflection}. Similarly to the DSFM problem, the dual formulations~(\ref{CDform}) and~(\ref{APform}) are simple, nearly separable in the dual variables, and may be solved by algorithms such as those used for solving dual DSFM, including the Douglas-Rachford splitting method (DR)~\citep{jegelka2013reflection}, the alternative projection method (AP)~\citep{nishihara2014convergence}, and the random coordinate descent method (RCD)~\citep{ene2015random}. 

However, there are also some notable differences. Our dual formulations for the QDSFM problem are defined on a product of cones constructed from the base polytopes of the submodular functions.  The optimization problem~\eqref{APform} essentially asks for the best approximation of an affine space in terms of the product of cones, whereas in the DSFM problem one seeks an approximation in terms of a product of polytopes. In the next section, we propose to solve the dual problem~\eqref{CDform} using the random coordinate descent method (RCD), and to solve~\eqref{APform} using the alternative projection method (AP), both tailored to the conic structure. The conic structures make the convergence analysis and the computations of projections for these algorithms challenging.

Before proceeding with the description of our algorithms, we first provide a relevant geometric property of the product of cones $\bigotimes_{r\in [R]} C_r$, as stated in the following lemma. This property is essential in establishing the linear convergence of RCD and AP algorithms.
\begin{lemma}\label{submodularcone}
Consider a feasible solution of the problem $(y,\phi)\in \bigotimes_{r\in [R]} C_r$ and a nonnegative vector $\phi' = (\phi_{r}') \in \bigotimes_{r\in [R]}\mathbb{R}_{\geq 0}$. Let $s$ be an arbitrary point in the base polytope of $\sum_{r\in [R]} \phi_r' F_r$, and let 
$W^{(1)}, W^{(2)}$ be two positive diagonal matrices. Then, there exists a $y' \in \bigotimes_{r\in [R]} \phi_r'B_r$ satisfying
$\sum_{r\in [R]} y_r' = s$ and
$$\|y- y'\|_{I(W^{(1)})}^2 +\|\phi - \phi'\|^2 \leq \mu(W^{(1)}, W^{(2)})\left[\|\sum_{r\in [R]} y_r - s\|_{W^{(2)}}^2 + \|\phi -\phi'\|^2\right],$$ \vspace{-0.3cm} where 
\begin{align}\label{defmu}
\mu(W^{(1)}, W^{(2)}) = \max\left\{\sum_{i\in[N]}W_{ii}^{(1)}\sum_{j\in [N]}1/W_{jj}^{(2)},\;\frac{9}{4}\rho^2\sum_{i\in [N]}W_{ii}^{(1)}+1\right\},
\end{align}
and $\rho=\max_{y_r\in B_r, \forall r\in[R]} \sqrt{\sum_{r\in [R]}\|y_r\|_1^2}$.
\end{lemma}
The proof of Lemma~\ref{submodularcone} is relegated to Section~\ref{proof:GS}. 

\section{Linearly Convergent Algorithms for Solving the QDSFM Problem}
Next, we introduce and analyze the random coordinate descent method (RCD) for solving the dual problem~\eqref{CDform}, and the alternative projection method (AP) for solving~\eqref{APform}. Both methods exploit the separable structure of the feasible set. It is worth mentioning that our results may be easily extended to the Douglas-Rachford splitting method as well as accelerated and parallel variants of the RCD method used to solve DSFM problems~\citep{jegelka2013reflection,ene2015random, li2018revisiting}.

We define the projection $\Pi$ onto a convex cone $C_r$ as follows: for a point $b$ in $\mathbb{R}^{N}$ and a positive diagonal matrix $W'$ in $\mathbb{R}^{N\times N}$, we set 
$$\Pi_{C_r}(b; W') = \underset{(y_r, \phi_r)\in C_r}{\operatorname{argmin}}\|y_r - b\|_{W'}^2 + \phi_r^2.$$
Throughout the remainder of this section, we treat the projections as provided by an oracle. Later in Section~\ref{innerloop}, we provide some efficient methods for computing the projections. 

\subsection{The Random Coordinate Descent (RCD) Algorithm} 
Consider the dual formulation~\eqref{CDform}. For each coordinate $r$, optimizing over the dual variables $(y_r, \phi_r)$ is equivalent to computing a projection onto the cone $C_r$. This gives rise to the RCD method of Algorithm 1. 
\begin{table}[t]
\centering
\begin{tabular}{l}
\hline
\label{RCDMalg}
\textbf{Algorithm 1:} \textbf{The RCD method for Solving~\eqref{CDform}} \\
\hline
\ 0: For all $r$, initialize $y_r^{(0)}\leftarrow 0$, $\phi_r^{(0)}$ and $k\leftarrow 0$\\
\ 1: In iteration $k$:\\
\ 2: \; Uniformly at random pick an $r\in [R]$.\\
\ 3: \; $(y_{r}^{(k+1)}, \phi_{r}^{(k+1)})\leftarrow  \Pi_{C_{r}}(2Wa - \sum_{r'\neq r} y_{r'}^{(k)}; W^{-1}) $\\
\ 4: \; Set $y_{r'}^{(k+1)}\leftarrow y_{r'}^{(k)}$ for $r'\not = {r}$ \\
\hline
\end{tabular}
\end{table}

The objective $g(y,\phi)$ described in~\eqref{CDform} is not strongly convex in general. However, with some additional work, 
Lemma~\ref{submodularcone} can be used to establish the weak strong convexity of $g(y,\phi)$; this essentially guarantees a linear convergence rate of the RCD algorithm. To proceed, we need some additional notation. Denote the set of solutions of~\eqref{CDform} by 
\begin{align}\label{solutionset}
\Xi= \{(y,\,\phi)| \sum_{r\in [R]}y_r = 2W(a - x^*), \phi_r = \inf_{y_r\in \theta B_r} \theta, \forall r\}.
\end{align}
Note that the optimal $\phi_r$, as defined in~\eqref{solutionset}, is closely related to gauge functions of polytopes~\citep{bach2013learning}.
This representation is a consequence of the relationship between the optimal primal and dual solutions stated in Lemma~\ref{dualform}. For convenience, we denote the optimal value of the objective function of interest over $(y,\,\phi)\in \Xi$ by $g^* = g (y,\phi)$, and also define a distance function 
$$d((y,\phi), \Xi) = \sqrt{\min\limits_{(y',\phi')\in\Xi} \|y-y'\|_{I(W^{-1})}^2 + \|\phi -\phi'\|^2}.$$ 
\begin{lemma}[Weak Strong Convexity]\label{strongconv} Suppose that $(y, \phi)\in \bigotimes_{r\in [R]} C_r$ and that $(y^*, \phi^*)\in\Xi$ minimizes 
$\|y-y'\|_{I(W^{-1})}^2 + \|\phi-\phi'\|^2$ over $(y',\phi')\in\Xi$, i.e., $(y^*, \phi^*)$ is the projection of $(y,\phi)$ onto $\Xi$.
Then, 
\begin{align}\label{eq:WSC}
\|\sum_{r\in [R]} (y_r- y_r^*)\|_{W^{-1}}^2+  \|\phi-\phi^*\|^2\geq  \frac{d^2((y,\phi), \Xi)}{\mu(W^{-1}, W^{-1})}.
\end{align}
\end{lemma}

Lemma~\ref{strongconv} can be proved by applying Lemma~\ref{submodularcone} with $\phi'=\phi^*, s = 2W(a - x^*),\;W^{(1)}, W^{(2)} = W^{-1}$ and the definition of $y^*$. Note that the claim in~(\ref{eq:WSC}) is significantly weaker than the usual strong convexity condition. We show that such a condition is sufficient to ensure linear convergence of the RCD algorithm and provides the results in the following theorem. 

\begin{theorem}[Linear Convergence of RCD] \label{linearconv}
Running $k$ iterations of Algorithm 1 produces a pair $(y^{(k)}, \phi^{(k)})$ that satisfies 
\begin{align*}
&\mathbb{E}\left[g(y^{(k)},  \phi^{(k)})- g^* + d^2((y^{(k)},  \phi^{(k)}), \Xi) \right]  \\
&\leq \left[1- \frac{2}{R[1+\mu(W^{-1}, W^{-1})]}\right]^{k}\left[g(y^{(0)},  \phi^{(0)})- g^* + d^2((y^{(0)}, r^{(0)}), \Xi) \right].
\end{align*}
\end{theorem}
 Theorem~\ref{linearconv} asserts that at most $O(R\mu(W^{-1}, W^{-1})\log\frac{1}{\epsilon})$ iterations are required to obtain an $\epsilon$-optimal solution in expectation for the QDSFM problem. The proof of Lemma~\ref{strongconv} and Theorem~\ref{linearconv} are postponed to Section~\ref{proof:RCD}.

\subsection{The Alternative Projection (AP) Algorithm}\label{appsec:AP}
The AP method can be used to solve the dual problem~\eqref{APform}, which is of the form of a \emph{best-approximation problem}, by alternatively performing projections between the product of cones and a hyperplane. Furthermore, for some incidence relations, $S_r$ may be a proper subset of $[N]$, which consequently requires the $i$th component of $y_r$, i.e., $y_{r,i}$ to be zero if $i\not\in S_r$. Enforcing $\lambda_{r, i} = 0$ for $i\not\in S_r$ allows the AP method to avoid redundant computations and achieve better convergence rates. This phenomenon has also been observed for the DSFM problem in~\citep{djolonga2015scalable,li2018revisiting}. To avoid redundant computations, we use the AP approach to solve the following dual problem:
\begin{align} \label{APcompactform}
 &\min_{y,\phi, \Lambda}\quad \sum_{r\in [R]}\left[\| y_r - \lambda_r\|_{\Psi W^{-1}}^2 + \phi_r^2\right], \\
&\text{s.t. $(y, \phi)\in \mathcal{C}$, $\sum_{r\in [R]} \lambda_r = 2Wa$, and $\lambda_{r,i} = 0$ for all $i\not\in S_r$}.  \nonumber
\end{align}
Here, $\Psi\in\mathbb{R}^{N\times N}$ is a positive diagonal matrix in which $\Psi_{ii} = |\{r\in[R]| i\in S_r\}|$ equals the number of submodular functions that $i$ is incident to. 
\begin{lemma}\label{APequaldual}
Problem~\eqref{APcompactform} is equivalent to problem~\eqref{CDform}.
\end{lemma}
The proof of Lemma~\ref{APequaldual} is presented in Section~\ref{proof:APequaldual}. The AP method for solving~\eqref{APcompactform} is listed in Algorithm 2. Observe that Step 5 is a projection onto cones defined based on the positive diagonal matrix $\Psi W^{-1}$ which differs from the one used in the RCD method. Note that compared to the RCD algorithm, AP requires one to compute projections onto \emph{all} cones $C_r$ in each iteration. Thus, in most cases, AP requires larger computation cost than the RCD algorithm. On the other hand, AP naturally lends itself to parallelization since the projections can be decoupled and implemented very efficiently.  
\begin{table}[t]
\centering
\begin{tabular}{l}
\hline
\label{APM}
\textbf{Algorithm 2: } \textbf{The AP Method for Solving ~\eqref{APcompactform}} \\
\hline
\ 0: For all $r$, initialize $y_r^{(0)} \leftarrow 0, \phi_r^{(0)}\leftarrow 0$, and $k\leftarrow 0$\\
\ 1: In iteration $k$:\\
\ 2: \; $\alpha^{(k+1)} \leftarrow 2W^{-1}\sum_r y_r^{(k)} -4a$. \\
\ 3: \; For all $r\in [R]$:\\
\ 4: \quad \; $\lambda_{r,i}^{(k+1)} \leftarrow y_{r,i}^{(k)} - \frac{1}{2}(\Psi^{-1}W\alpha^{(k+1)})_i$ for $i\in S_r$\\
\ 5: \quad \; $(y_r^{(k+1)}, \phi_{r}^{(k+1)})\leftarrow  \Pi_{C_r}(\lambda_r^{(k+1)}; \Psi W^{-1})$\\ 
\hline
\end{tabular}
\end{table}

Before moving to the convergence analysis, we first prove the uniqueness of optimal $\phi$ in the following lemma.
\begin{lemma}\label{lem:phiunique}
The optimal value of $\phi$ to problem~\eqref{APcompactform} is unique.
\end{lemma}
\begin{proof}
We know~\eqref{APcompactform} is equivalent to~\eqref{CDform}. Suppose there are two optimal solutions $(\bar{y}, \bar{\phi})$ and $(\tilde{y}, \tilde{\phi})$. As the optimal $x^*$ is unique, $\sum_{r\in [R]}\bar{y}_r = \sum_{r\in [R]}\tilde{y}_r=  2W(a - x^*)=\sum_{r\in [R]}\frac{\bar{y}_r+\tilde{y}_r}{2} $. If $\bar{\phi}\neq \tilde{\phi}$, we have $\sum_{r\in [R]} \left(\frac{\bar{\phi}_r+\tilde{\phi}_r}{2}\right)^2 < \sum_{r\in [R]}\bar{\phi}_r^2 = \sum_{r\in [R]}\tilde{\phi}_r^2$, which makes $g(\frac{\bar{y}+\tilde{y}}{2}, \frac{\bar{\phi}+\tilde{\phi}}{2})$ smaller than $g(\bar{y}, \bar{\phi})=g(\tilde{y}, \tilde{\phi})$ and thus causes contradiction.
\end{proof}

To determine the convergence rate of the AP method, we adapt the result of~\citep{nishihara2014convergence} on the convergence rate of APs between two convex bodies. In our setting, the two convex bodies of interest are the the cone $\mathcal{C}$ and the hyperplane  
\begin{align*}
\mathcal{Z} = \{(y,\phi)| \sum_{r\in[R]} y_r = 2W(a - x^*),  \phi_r = \phi_r^*, y_{r,i} = 0, \forall \, i\not\in S_r\},
\end{align*}
where $\phi^* = (\phi_r^*)_{r\in [R]}$ is the unique optimal solution of~\eqref{APcompactform}.
\begin{lemma}[\cite{nishihara2014convergence}] 
Let $\Xi$ be as defined in~\eqref{solutionset}. In addition, define the distance function 
$$d_{\Psi W^{-1}}((y,\phi), \Xi) = \sqrt{\min\limits_{(y',\phi')\in\Xi} \|y-y'\|_{I(\Psi W^{-1})}^2 + \|\phi -\phi'\|^2}.$$ 
In the $k$-th iteration of Algorithm 5, the pair $(y^{(k)}, \phi^{(k)})$ satisfies
\begin{align*}
d_{\Psi W^{-1}}((y^{(k)}, \phi^{(k)}), \Xi) \leq 2d_{\Psi W^{-1}}((y^{(0)}, \phi^{(0)}), \Xi)(1 - \frac{1}{\kappa_*^2})^k,
\end{align*}
where 
$$\kappa_* = \sup_{(y, \phi) \in \mathcal{C} \cup \mathcal{Z} / \Xi} \frac{d_{\Psi W^{-1}}((y, \phi), \Xi)}{\max\{d_{\Psi W^{-1}}((y, \phi), \mathcal{C}), d_{\Psi W^{-1}}((y, \phi), \mathcal{Z})\}}.$$
\end{lemma}
The next Lemma establishes a finite upper bound on $\kappa_*$. This guaratees linear convergence rates for the AP algorithm. 
\begin{lemma}\label{upperboundkappa}
One has $\kappa_*^2 \leq 1+ \mu(\Psi W^{-1}, W^{-1})$.
\end{lemma}
The proof of Lemma~\ref{upperboundkappa} may be found in Section~\ref{proof:upperboundkappa}. 

Lemma~\ref{upperboundkappa} implies that running the AP algorithm with $O(\mu(\Psi W^{-1}, W^{-1}) \log \frac{1}{\epsilon})$ iterations guarantee that $d_{\Psi W^{-1}}((y^{(k)}, \phi^{(k)}), \Xi)\leq \epsilon$. Moreover, as $\Psi_{ii} \leq R$, the iteration complexity of the AP method is smaller than that of the RCD algorithm. However, each iteration of the AP solver requires performing projections on all cones $C_r, \, r\in [R]$, while each iteration of the RCD method requires computing only one projection onto a single cone. Other methods such as the DR and primal-dual hybrid gradient descent (PDHG)~\citep{chambolle2011first} used in semi-supervised learning on hypergraphs, which is a specific QDSFM problem~\citep{hein2013total}, also require computing a total of $R$  projections during each iteration. Thus, from the perspective of iteration cost, RCD is significantly more efficient, especially when $R$ is large and computing $\Pi(\cdot)$ is costly. This phenomenon is also observed in practice, as illustrated by the experiments in Section~\ref{sec:exp}. 

The following corollary summarizes the previous discussion and will be used in our subsequent derivations in Section~\ref{sec:PR}.
\begin{corollary}\label{specialcase}
Suppose that $W = \beta D$, where $\beta$ is a hyper-parameter, and $D$ is a diagonal matrix such that 
$$D_{ii} = \sum_{r\in[R]:\, i\in S_r} \max_{S\subseteq V} [F_r(S)]^2.$$ 
Recall that $\Psi_{ii} = |\{r\in[R]| i\in S_r\}|$. Then, Algorithm 1 (RCD)  requires an expected number of 
$$O\left(N^2R\max\{1,9\beta^{-1}\}\max_{i,j\in[N]}\frac{D_{ii}}{D_{jj}}\log\frac{1}{\epsilon}\right)$$ 
iterations to return a solution $(y, \phi)$ that satisfies $d((y, \phi), \Xi) \leq \epsilon$. Algorithm 2  (AP)  requires a number of 
$$O\left(N^2R\max\{1,9\beta^{-1}\}\max_{i,j\in[N]}\frac{\Psi_{jj}}{R}\frac{D_{ii}}{D_{jj}}\log\frac{1}{\epsilon}\right)$$ 
iterations to return a solution $(y, \phi)$ that satisfies $d_{\Psi W^{-1}}((y, \phi), \Xi) \leq \epsilon$. 
\end{corollary}
The proof of Corollary~\ref{specialcase} is provided in Section~\ref{proof:specialcase}. Note that the term $N^2R$ also appears in the expression for the complexity of the RCD and the AP methods for solving the DSFM problem~\citep{ene2017decomposable}. The term $\max\{1,9\beta^{-1}\}$ implies that whenever $\beta$ is small, the convergence rate is slow. In Section~\ref{sec:PR}, we provide a case-specific analysis showing that $\beta$ is related to the mixing time of the underlying Markov process in the PageRank application. Depending on the application domain, other interpretations of $\beta$ could exist, but will not be discussed here to avoid clutter. It is also worth pointing out that the variable $D$ in Corollary~\ref{specialcase} is closely related to the notion of  degree in hypergraphs~\citep{li2018submodular,yoshida2019cheeger}. 

\section{Computing the Conic Projections}\label{innerloop}
In this section, we describe efficient routines for computing a projection onto the conic set $\Pi_{C_r}(\cdot)$. As the procedure works for all values of $r\in[R]$, we drop the subscript $r$ for simplicity of notation. 
Let $C=\{(y,\phi)| y\in \phi B, \phi\geq 0\}$, where $B$ denotes the base polytope of the submodular function $F$. Recall that the conic projection onto $C$ is defined as  
\begin{align}\label{projection}
\Pi_{C}(a; \tilde{W})  = \arg\min_{(y,\phi)} h(y,\phi) \triangleq \|y- a\|_{\tilde{W}}^2 + \phi^2 \quad \text{s.t. $(y,\phi)\in C$}.
\end{align} 
 Let $h^*$ and $(y^*, \phi^*)$ be the optimal value of the objective function and the optimal solution, respectively. When performing projections, one only needs to consider the variables incident to $F$, and set all other variables to zero. For simplicity, we assume that all variables in $[N]$ are incident to $F$. The results can easily extend to the case that the incidences are of a general form. 

Unlike the QDSFM problem, solving the DSFM involves the computation of projections onto the base polytopes of submodular functions. Two algorithms, the Frank-Wolfe (FW) method~\citep{frank1956algorithm} and the Fujishige-Wolfe minimum norm algorithm (MNP)~\citep{fujishige2011submodular}, are used for this purpose. Both methods assume inexpensive linear minimization oracles on polytopes and guarantee a $1/k$-convergence rate. The MNP algorithm is more sophisticated yet empirically more efficient. Nonetheless, neither of these methods can be applied directly to conic projections. To this end, we modify these two methods by adjusting them to the conic structure in~\eqref{projection}  and show that they still achieve a $1/k$-convergence rate. We refer to the procedures as \emph{the conic FW method} and \emph{the conic MNP method}, respectively. 

\subsection{The Conic MNP Algorithm}
Detailed steps of the conic MNP method are presented in Algorithm 3. 
\begin{table}[t]
\label{generalalg-MNP}
\centering
\begin{tabular}{l}
\hline
\textbf{Algorithm 3: } \textbf{The Conic MNP Method for Solving~\eqref{projection}} \\
\hline
\textbf{Input}:  $\tilde{W}$, $a$, $B$ and a small positive constant $\delta$. \textbf{Maintain} $\phi^{(k)} = \sum_{q_i\in S^{(k)}}\lambda_i^{(k)}$\\
Choose an arbitrary $q_1\in B$. Set $S^{(0)} \leftarrow \{q_1\}$, $\lambda_1^{(0)}\leftarrow \frac{\langle a, q_1\rangle_{\tilde{W}}}{1+ \|q_1\|_{\tilde{W}}^2}$, $y^{(0)} \leftarrow  \lambda_1^{(0)}q_1$, $k\leftarrow 0$ \\
1. Iteratively execute (\textbf{MAJOR LOOP}): \\
2. \; $q^{(k)} \leftarrow \arg\min_{q\in B} \langle  \nabla_{y} h(y^{(k)}, \phi^{(k)}), q \rangle_{\tilde{W}}$\\
3. \; \textbf{If} $\langle y^{(k)} - a, q^{(k)} \rangle_{\tilde{W}} + \phi^{(k)}\geq -\delta$, \textbf{break}; \\
4. \; \textbf{Else} $S^{(k)} \leftarrow S^{(k)}\cup \{q^{(k)}\}$.  \\
5. \;\;\;\;\;\;\; Iteratively execute (\textbf{MINOR LOOP}):\\
6. \;\;\;\;\;\;\;\;\;\; $\alpha \leftarrow  \arg\min_{\alpha} \|\sum_{q_i^{(k)}\in S^{(k)}} \alpha_i q_i^{(k)} - a\|_{\tilde{W}}^2 + (\sum_{q_i^{(k)}\in S} \alpha_i)^2$, \\
7. \;\;\;\;\;\;\;\;\;\; $z^{(k)} \leftarrow \sum_{q_i^{(k)}\in S} \alpha_i q_i^{(k)}$ \\
8. \;\;\;\;\;\;\;\;\;\;  \textbf{If} $\alpha_i \geq 0$ for all $i$, \textbf{break}; \\
9. \;\;\;\;\;\;\;\;\;\;  \textbf{Else} $\theta = \min_{i: \alpha_i < 0} \lambda_i^{(k)}/(\lambda_i^{(k)} - \alpha_i)$, \; $\lambda_i^{(k+1)} \leftarrow \theta\alpha_i + (1-\theta)\lambda_i^{(k)}$,\\
10.\;\;\;\;\;\;\;\;\;\;\;\;\;\;\; $y^{(k+1)} \leftarrow \theta z^{(k)} + (1-\theta)y^{(k)}$, \; $S^{(k+1)}\leftarrow\{i: \lambda^{(k+1)}> 0\}$,  \;$k\leftarrow k+1$ \\
11.\; $y^{(k+1)} \leftarrow z^{(k)}$, $\lambda^{(k+1)} \leftarrow \alpha$, $S^{(k+1)}\leftarrow\{i: \lambda^{(k+1)} > 0\}$, $k\leftarrow k+1$\\
\hline
\end{tabular}
\end{table}
The conic MNP algorithm keeps track of an \emph{active set} $S =\{q_1, q_2,...\}$ and searches for the best solution in its conic hull. Denote the cone of an active set $S$ by $\text{cone}(S) = \{\sum_{q_i\in S}\alpha_i q_i| \alpha_i \geq 0\}$ and its linear set by $\text{lin}(S) = \{\sum_{q_i\in S}\alpha_i q_i|\alpha_i\in\mathbb{R}\}$. 
Similar to the original MNP algorithm,  
Algorithm 3 also includes two level-loops: the MAJOR and MINOR loop. In the MAJOR loop, one greedily add a new active point $q^{(k)}$ to the set $S$ obtained from the linear minimization oracle w.r.t. the base polytope (Step 2). By the end of the MAJOR loop, we obtain  $y^{(k+1)}$ that minimizes $h(y,\phi)$ over $\text{cone}(S)$ (Step 3-11).  The MINOR loop is activated when $\text{lin}(S)$ contains some point $z$ that guarantees a smaller value of the objective function than the optimal point in $cone(S)$, provided that some active points from $S$ may be removed. It is important to point out that compared to the original MNP method, Steps 2 and 6 as well as the termination Step 3 are specialized for the conic structure. 

Theorem~\ref{Wolfegaurantee} below shows that the conic MNP algorithm preserves the $1/k$-convergence rate of the original MNP method.

\begin{theorem}\label{Wolfegaurantee}
Let $B$ be an arbitrary polytope in $\mathbb{R}^{N}$ and let $C = \{(y,\phi)| y\in \phi B, \phi\geq 0\}$ be the cone induced by the polytope. For any positive diagonal matrix $\tilde{W}$, define $Q = \max_{q\in B} \|q\|_{\tilde{W}}$. Algorithm 3 produces a sequence of pairs $(y^{(k)}, \phi^{(k)})_{k=1,2,...}$ such that $h(y^{(k)}, \phi^{(k)})$ decreases monotonically. It terminates when $k= O(N\|a\|_{\tilde{W}}\max\{Q^2,1\}/\delta),$ with 
$$h(y^{(k)}, \phi^{(k)}) \leq h^* + \delta\|a\|_{\tilde{W}}.$$
\end{theorem}

Note that the result is essentially independent on the submodularity assumption and applies to general cones induced by arbitrary polytopes. Proving Theorem~\ref{Wolfegaurantee} requires a careful modification of the arguments in~\citep{chakrabarty2014provable} to handle the conic structure. We provide a detailed proof in Section~\ref{proof:Wolfegaurantee}.

\subsection{The Conic FW Algorithm}
\label{appsec:FW}
We introduce next the conic FW method, summarized in Algorithm 4. Note that Steps 2, 4 are specialized for cones. The main difference between the FW and MNP methods is the size of active set: FW only maintains two active points, while MNP may maintain as many as $N$ points. A similar idea was proposed in~\citep{harchaoui2015conditional} to deal with the conic constraints, with the following minor difference: our method directly finds the variable $y^{(k+1)}$ from the cone $(\{y^{(k)}, q^{(k)}\})$ (Steps 4, 5) while \citep{harchaoui2015conditional} first determines finite upper bounds $(\tilde{\gamma}_1, \tilde{\gamma}_2)$ on the coefficients $(\gamma_1, \gamma_2)$ (see the notation in Steps 4, 5) and then searches for the update $y^{(k+1)}$ in conv$(\{\tilde{\gamma}_1y^{(k)},  \tilde{\gamma}_2q^{(k)}, 0\})$. Our method simplifies this procedure by avoiding computations of the upper bounds $(\tilde{\gamma}_1, \tilde{\gamma}_2)$ and removing the constraints $\gamma_1 \leq \tilde{\gamma}_1, \gamma_2 \leq \tilde{\gamma}_2$ in Step 4.

For completeness, we establish a $1/k$-convergence rate for the conic FW algorithm in Theorem~\ref{FWalgCR}.
 
 \begin{table}[t]
\centering{}
\begin{tabular}{l}
\hline
\textbf{Algorithm 4:} \textbf{The Conic FW Algorithm for Solving~\eqref{projection}} \\
\hline
\ \textbf{Input}:  $\tilde{W}$, $a$, $B$ and a small positive $\delta$\\
Initialize $y^{(0)} \leftarrow 0 $, $ \phi^{(0)}\leftarrow 0$ and $k\leftarrow 0$\\ 
1. \textbf{Iteratively execute the following steps:} \\
2. \;\; $q^{(k)} \leftarrow \arg\min_{q\in B} \langle \nabla_{y}h(y^{(k)}, \phi^{(k)}) , q \rangle $\\
3. \;\;  \textbf{If} $\langle y^{(k)} - a, q^{(k)} \rangle_{\tilde{W}} + \phi^{(k)}\geq -\delta$, \textbf{break}.\\
4. \;\; \textbf{Else}: $(\gamma_1^{(k)}, \gamma_2^{(k)}) \leftarrow  \arg\min_{\gamma_1 \geq 0,\gamma_2\geq 0} h(\gamma_1y^{(k)} + \gamma_2 q^{(k)}, \gamma_1^{(k)} \phi^{(k)}  + \gamma_2^{(k)})$ \\
5. \;\;\;\;\;\;\;\; \;\;  $y^{(k+1)} \leftarrow \gamma_1^{(k)} y^{(k)} + \gamma_2^{(k)} q^{(k)}, \, \phi^{(k+1)}\leftarrow  \gamma_1^{(k)} \phi^{(k)}  + \gamma_2^{(k)}$, \, $k\leftarrow k+1$.\\
\hline
\end{tabular}
\end{table}

\begin{theorem}\label{FWalgCR}
Let $B$ be an arbitrary polytope in $\mathbb{R}^{N}$ and let $C = \{(y,\phi)| y\in \phi B, \phi\geq 0\}$ denote its corresponding cone. For some positive diagonal matrix $\tilde{W}$, define $Q = \max_{q\in B} \|q\|_{\tilde{W}}$. Then, the point $(y^{(k)}, \phi^{(k)})$ generated by Algorithm 4 satisfies $$h(y^{(k)}, \phi^{(k)})  \leq h^* + \frac{2\|a\|_{\tilde{W}}^2 Q^2}{k+2}.$$
\end{theorem}
The proof of Theorem~\ref{FWalgCR} is deferred to Section~\ref{proof:FWalgCR}.
%

\subsection{Analysis of Time Complexity and Further Discussion}

Let $\text{EO}$ stand for the time needed to query the function value of $F(\cdot)$. The linear program in Step 2 of both the MNP and FW algorithms requires $O(N\log N + N \times EO)$ operations if implemented as a greedy method. Step 6 of the MNP method requires solving a quadratic program with no constraints, and the solution may be obtained in closed form using $O(N|S|^2)$ operations. The remaining operations in the MNP method introduce a $O(N)$ complexity term. Hence, the execution of each MAJOR or MINOR loop requires $O(N\log N + N\times EO + N|S|^2)$ operations. In the FW method, the optimization problem in Step 4 is relatively simple, since it reduces to solving a nonnegative quadratic program with only two variables and consequently introduces a complexity term $O(N)$. Therefore, the complexity of the FW algorithm is dominated by Step 2, which requires $O(N\log N + N\times EO)$ operations. 

Below we discuss the distinctions of our method from several relevant methods for computing projections on cones. Compared to the linearly convergent away-steps Frank-Wolfe algorithm~\citep{wolfe1970convergence} and its generalized versions for the conic case, such as the away-steps/pairwise non-negative matching pursuit (AMP/PWMP) techniques recently proposed in~\citep{locatello2017greedy}, the size of the active set $S$ in the (conic) MNP algorithm is always bounded by $N$ instead of the (increasing) number of iterations. Particularly, when applied to the QDSFM problems, the size of the active set $S$ in the (conic) MNP algorithm can  be even smaller in practice as it  is bounded by the number of elements that affects the value of the submodular function, which is usually smaller than $N$~\citep{li2018revisiting}. Moreover, in terms of base polytopes of submodular functions, AMP/PWMP may not perform as well as MNP  because the steepest ascent over the active set may not be executed as efficiently as the descent over polytopes obtained via the greedy approach (Step 2, which generally appears in FW-based methods).

Another benefit of the (conic) MNP method is that it can compute exact projections for the problem~\eqref{projection} in finitely many steps. Because of the structure of the cones generated from polytopes, the exact projections are expressed as linear combinations of the extreme points of the polytopes. As the number of such combinations is finite, MNP can find the exact projections in finitely many iterations, thereby inheriting a similar property of the MNP method pertaining to polytopes~\citep{fujishige2011submodular}. While this property also holds for the fully corrective non-negative matching pursuit (FCMP)~\citep{locatello2017greedy}, FCMP requires solving an additional nonnegative quadratic program in the inner loop,  and is thus much less efficient; MNP does not require an explicit solution of the nonnegative quadratic program as it allows for an efficient update of the set of atoms, $S^{(k)},$ described in Algorithm 3, through proper leveraging of the problem structure (minor loops of Algorithm 3). In contrast, the AMP/PWMP methods in~\citep{locatello2017greedy} as well as our conic FW method introduced in Section~\ref{appsec:FW} may not ensure exact projections. As the proposed algorithms for QDSFM depend highly on precise projections for the problem~\eqref{projection}, using MNP as the subroutine for computing projections has the potential to yield better accuracy than other methods discussed above. This is further evidenced in our experiments, to be discussed below.


\subsection{Numerical Illustration I: Convergence of QRCD-MNP and QRCD-FW}\label{sec:mnp-fw}
We  illustrate the convergence speeds of the RCD algorithm when using the conic MNP and the conic FW methods as subroutines for computing the conic projections, referred to as QRCD-MNP and QRCD-FW, respectively. We also compare them with two other methods, QRCD-AMP and QRCD-PWMP, where the subroutine for computing the conic projection is replaced by AMP and PWMP~\citep{locatello2017greedy}, respectively. For this purpose, we construct a synthetic QDSFM problem~\eqref{QDSFM} by fixing $N = 100$, $R = 100$, and $W_{ii} = 1,$ for all $i\in[N]$. We generate each incidence set $S_r$ by choosing uniformly at random a subset of $[N]$ of cardinality $10$, and then set the entries of $a$ to be i.i.d. standard Gaussian. 

We consider the following submodular function: for $r\in[R]$, $S\subseteq S_r$,
\begin{align}\label{arbi-submodular}
F_r(S) = \frac{\min\{|S|, |S_r/S|\}^\theta}{ (|S_r|/2)^{\theta}}, \quad \quad \theta\in \{0.25,\,0.5,\,1\}.
\end{align}

\begin{figure*}[t]
\includegraphics[trim={0cm 0cm 0cm 0},clip,width=.32\textwidth]{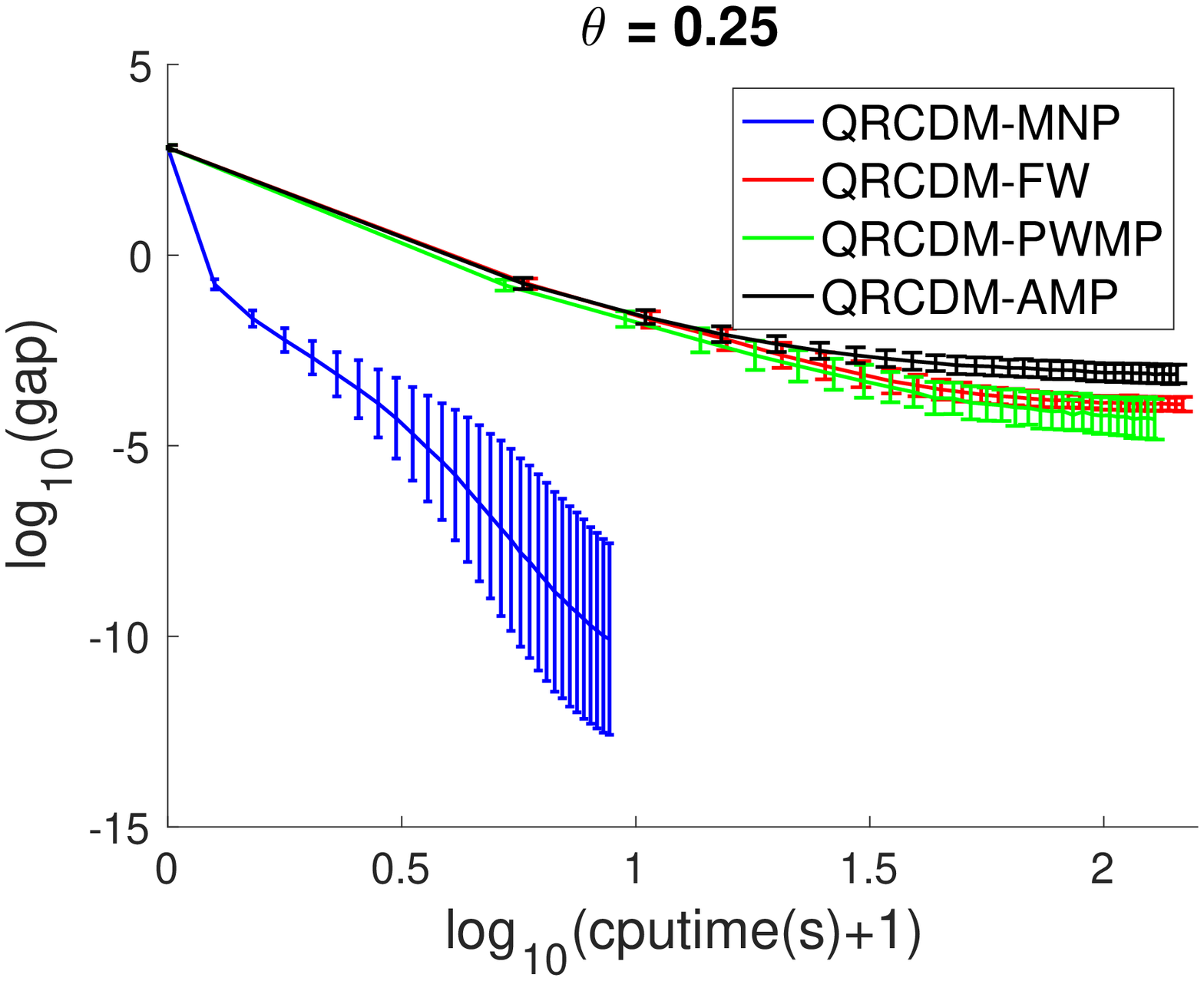}
\includegraphics[trim={0cm 0cm 0cm 0},clip,width=.32\textwidth]{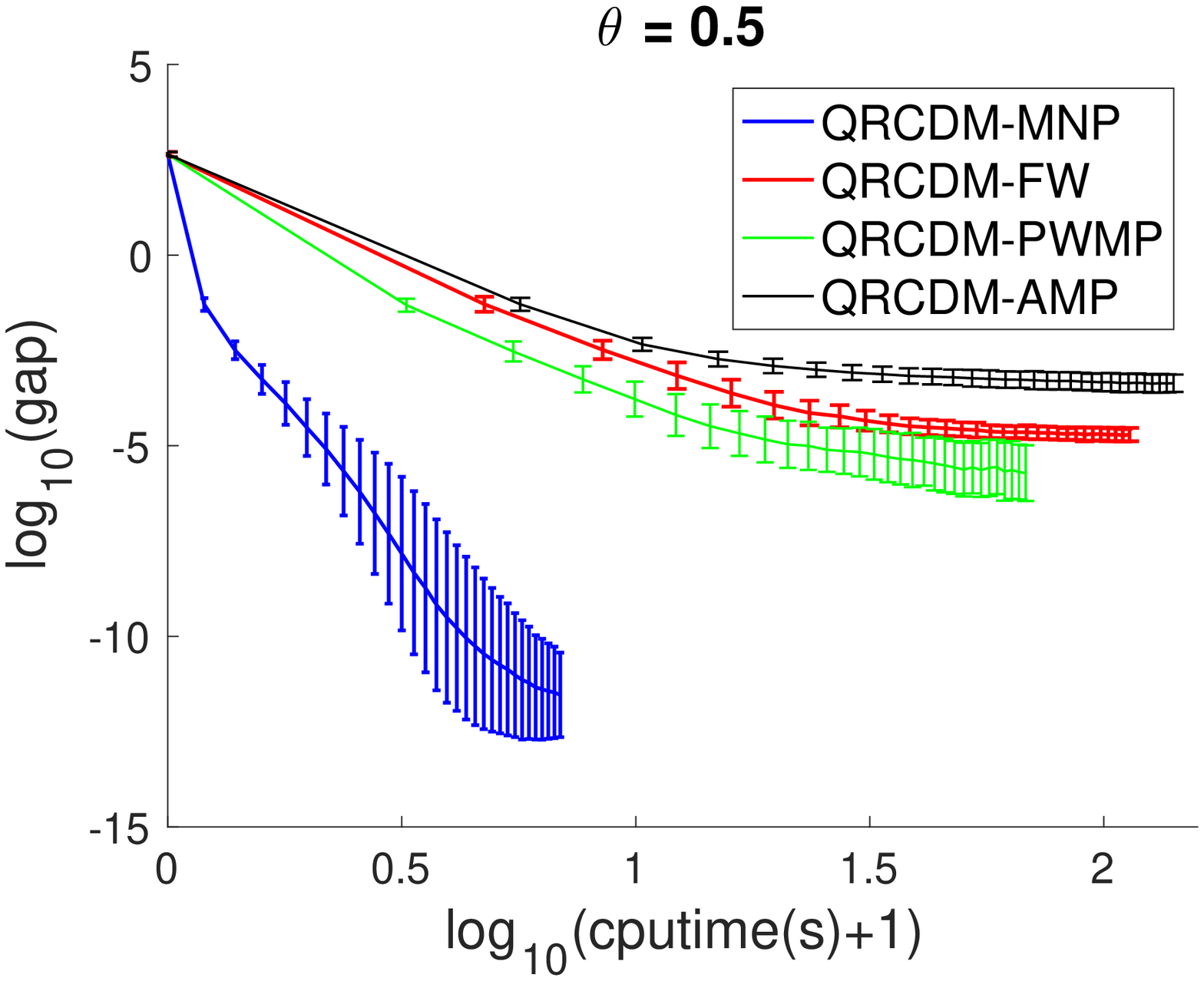}
\includegraphics[trim={0cm 0cm 0cm 0},clip,width=.32\textwidth]{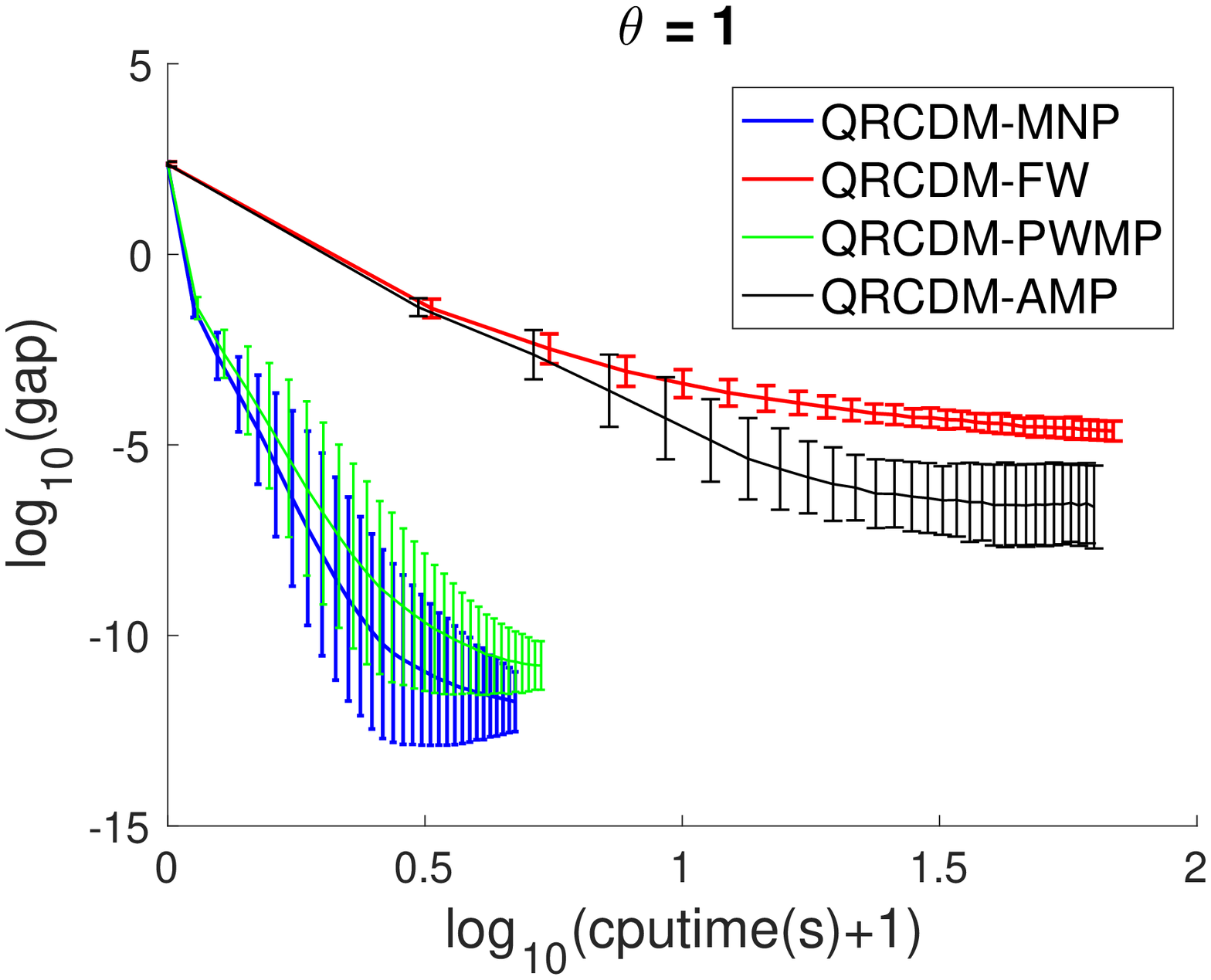}
\caption{Convergence results ($\log_{10}(\text{gap})$: mean $\pm$ std) for the QRCD-MNP, QRCD-FW, QRCD-AMP, QRCD-PWMP algorithms for general submodular functions.}
\label{fig:MNP-FW}
\end{figure*}

The number of RCD iterations is set to $300R = 3\times 10^4$. For the projections performed in the inner loop, as only the conic MNP method is guaranteed to find the exact solution in finite time, we constrain the number of iterations used by competing methods to ensure a fair comparison. To this end, note that for the AMP and PWMP methods~\citep{locatello2017greedy}, a parameter $\delta$ that controls the precision is computed as $-\gamma \|\mathbf{d}_t\|^2$ (here we use the notation from Algorithm 3 of~\citep{locatello2017greedy}). For all algorithms, we terminate the inner loop by setting $\delta = 10^{-12}$ or a maximum number of iterations $|S_r|^3 = 1000$.  All algorithms were executed on a 3.2GHz Intel Core i5 machine. For each algorithm, the average and variance of the primal gap are obtained via $100$ independent tests. The convergence results are shown in Figure~\ref{fig:MNP-FW}, which demonstrates that QRCD-MNP outperforms the other three methods.The performance of QRCD-PWMP is satisfactory when $\theta = 1$, but it deteriorates for smaller values of $\theta$, which may be a consequence of increasing the number of extreme points of cones/polytopes corresponding to the submodular functions~\eqref{arbi-submodular}. QRCD-FW and QRCD-AMP exhibit similar convergence behaviors but fail to provide high-precision solutions for reasonable CPU times.

\subsection{Numerical Illustration II: Dependency on the Scale of the Submodular Functions}

As shown in Theorems~\ref{Wolfegaurantee} and ~\ref{FWalgCR}, the complexity of the conic MNP/FW methods depends on the scale of the submodular functions $F$, specifically, $(h(y^{(k)}, \phi^{(k)}) - h^*)/\|a\|_{\tilde{W}}^2 \sim O(F^2)$. We would like to evaluate such dependency. Since we are only concerned with the computations of projection, we focus on a single submodular function $F(\cdot)$ that once again takes the form~\eqref{arbi-submodular} with parameter $\theta = 0.25$ and a ground set of cardinality $N = |S_r| = 100$. The submodular function is further scaled as $s*F(\cdot)$, where $s$ is a constant in the set $\{1, 2, 5, 10, 20\}$. We generate the entries of $a$ in~\eqref{projection} at the same scale as the submodular function, i.e., $a_v \sim \mathcal{N}(0, s^2)$. The proposed MNP, FW, and the AMP/PWMP methods from~\citep{locatello2017greedy} were tested, and their CPU execution time reported in Figure~\ref{fig:scaleprojection}. 
All methods terminate when $(h(y^{(k)}, \phi^{(k)}) - h^*)/\|a\|_{\tilde{W}}^2 \leq 10^{-5}$ over independent runs.

As can be seen from Figure~\ref{fig:scaleprojection}, 
the complexities of all methods increase as the submodular function scales up, while the conic MNP method tends to be the most robust with respect to scale changes. PWMP appears to be the most sensitive to scaling due to the growing size of the active set with increasing $s$. 


\begin{figure*}[t]
\centering
\includegraphics[trim={0cm 0cm 0cm 0},clip,width=.5\textwidth]{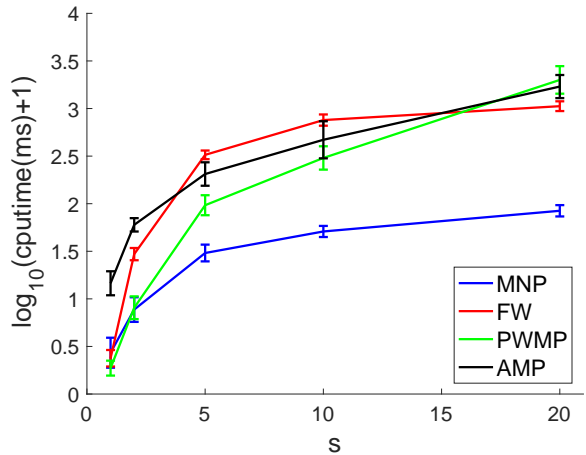}
\caption{The required CPU time for different methods used to compute the projection~\eqref{projection} versus the scale of the submodular function $s$.}
\label{fig:scaleprojection}
\end{figure*}

\section{Applications to PageRank }

In what follows, we introduce one important application of the QDSFM problem - PageRank (PR). Our treatment of the PR process relies on diffusion 
processes over hypergraphs~\citep{chan2017diffusion,chan2018spectral}. We demonstrate that what is known as a \emph{PR vector} can be efficiently computed using our proposed QDSFM solver. We also show that the newly introduced PR retains important properties of the standard PR over graphs, pertaining both to mixing and local partitioning properties~\citep{andersen2006local}.  

\subsection{Background: PageRank(PR)}\label{sec:PR}

PR is a well-known learning algorithm over graphs, first introduced in the context of webpage search~\citep{page1999pagerank}. Given a graph $G=(V, E)$, with $V = [N]$ and $E$ denoting the edge set over $V$, we let $A$ and $D$ denote the adjacency matrix and diagonal degree matrix of $G$, respectively. PR essentially reduces to finding a fixed point $p\in \mathbb{R}^N$ of the iterative process 
$$p^{(t+1)} = \alpha p^{(0)} + (1-\alpha) AD^{-1} p^{(t)},$$ 
where $p^{(0)} \in \mathbb{R}^N$ is a fixed vector and $\alpha \in (0,1)$. By using different exploration vectors $p^{(0)}$, PR can be adapted to different graph-based applications. For example, for webpages ranking~\citep{page1999pagerank}, $p^{(0)}$ may be a uniformly at random chosen vertex. For finding a local graph partitioning around a vertex $i$~\citep{andersen2006local}, $p^{(0)}$ is typically set to $e_i$, i.e., a zero-one vector with a single $1$ in its $i$-th component. 

It is easy to verify that the fixed point $p$ is a solution to the problem
\begin{align}\label{PReq}
\min_p\frac{\alpha}{1-\alpha} \|p - p^{(0)}\|_{D^{-1}}^2 + (D^{-1} p)^T(D -A)(D^{-1} p)  =\|x - a\|_{W}^2 + \langle x, L(x)\rangle,
\end{align}
where $x = D^{-1} p$, $a = D^{-1}p^{(0)}$, $W = \frac{\alpha}{1-\alpha} D$ and $L = D- A$.  

Recent work exploring the role of high-order functional relations among vertices in networks~\citep{benson2016higher, li2017inhomogeneous, li2017motif} brought forward the need for studying PR-type of problems for hypergraphs~\citep{zhou2017local,yin2017local}.\footnote{Hypergraphs are natural extensions of graphs in which edges are replaced by hyperedges, representing subsets of vertices in $V$ of size $\geq 2$.} There exist several definitions of the PR algorithm for hypergraphs. For example, the approach in~\citep{zhou2017local} relies on what is termed a \emph{multilinear PR}~\citep{gleich2015multilinear} and is limited to the setting of uniform hypergraphs, in which all hyperedges have the same cardinality. Computing a multilinear PR requires tensor multiplication and the latter method has complexity exponential in the size of the hyperedges. The PR formulation in~\citep{yin2017local} requires one to first \emph{project hypergraphs onto graphs} and then run standard graph-based PR. The projection step usually causes large distortions when the hyperedges are large and leads to low-quality hypergraph partitionings~\citep{hein2013total, li2018submodular}. 

Using equation~\eqref{PReq}, we introduce a new formulation for PR on hypergraphs that leverages nonlinear Laplacian operators $L(\cdot)$ for (un)directed hypergraphs~\citep{chan2017diffusion,chan2018spectral} and submodular hypergraphs~\citep{li2018submodular,yoshida2019cheeger}. The new PR may be applied to both uniform and non-uniform hypergraphs with arbitrarily large hyperedges. The stationary point is the solution to a QDSFM problem, where the term $\langle x, L(x)\rangle$ in~\eqref{PReq} is represented by a sum of Lov\'{a}sz extensions of submodular functions defined over hyperedges (See Table~\ref{tab:TLexamples}). 

\begin{table}[t] 
\begin{tabular}{|p{4.5cm}<{\centering}|p{4.8cm}<{\centering}|p{4.3cm}<{\centering}|}
\hline
A single component of $\langle x, L(x)\rangle$ & The combinatorial structure & The submodular function \\
\hline
$w_{r} (x_i - x_j)^2$, $S_r = \{i,j\}$ & A graph~\citep{zhou2004learning,zhu2003combining}  & $F_{r}(S) = \sqrt{w_{ij}}$ if $|S\cap \{i,j\}|=1$  \\
\hline
$w_{r} \max_{i,j \in S_r} (x_i - x_j)^2$ & A hypergraph~\citep{hein2013total} & {$F_r(S) = \sqrt{w_{r}}$ if $|S\cap S_r| \in [1, |S_r| -1]$}  \\
\hline
$w_{r}\max\limits_{(i,j)\in H_r \times T_r}(x_i -x_j)_+^2$ & A directed hypergraph~\citep{zhang2017re} & {$F_r(S) = \sqrt{w_{r}}$ if $|S\cap H_r| \geq 1$, $| ([N]/S)\cap T_r| \geq 1$}  \\
\hline
General $[f_r(x)]^2$ & A submodular hypergraph~\citep{li2018submodular,yoshida2019cheeger} & Any symmetric submodular function\\
\hline
\end{tabular}
\vspace{0.1in}
\centering 
\caption{The term $\langle x, L(x) \rangle$ for different combinatorial structures. In the third column, whenever the stated conditions are not satisfied, it is assumed that $F_r=0$. For directed hypergraphs, $H_r$ and $T_r$ are subsets of $S_r$ which we subsequently term the \emph{head} and the \emph{tail set}. For $H_r = T_r = S_r$, one recovers the setting for undirected hypergraphs.}\label{tab:TLexamples}
\end{table}

To demonstrate the utility of this particular form of PageRank, we show that the PR vector obtained as a solution of the QDSFM can be used to find a cut of a directed hypergraph (as well as undirected hypergraph) with small conductance, which is an important invariant for local partitioning and community detection~\citep{zhou2017local,yin2017local}. The results essentially generalize the mixing results for PR over graphs derived in~\citep{andersen2006local} and~\citep{andersen2008local}. 
Our main technical contribution is handling new technical challenges that arise from the non-linearity of the operators involved. In what follows, we mainly focus on directed hypergraphs. Tighter and more general results may be obtained for submodular hypergraphs~\citep{li2018submodular,yoshida2019cheeger}, which will be reported in a companion paper.

\subsection{Terminology: Boundary, Volume, Conductance} 

Let $G= (V, E)$ be a hypergraph with $V=[N]$ and hyperedge in $E$ equated with the incidence sets $\{S_r\}_{r\in [R]}$. More precisely, each hyperedge represents an incidence set $S_r$ associated with a triplet $(w_r, H_r, T_r)$, where $w_r$ is a scalar weight and $H_r$ and $T_r \subseteq S_r$ are the head and tail sets defined in Table~\ref{tab:TLexamples}. The pair $(H_r, T_r)$ defines a submodular function $F_r$ with incidence set $S_r$ according to $F_r(S) = 1$ if $|S\cap H_r|\geq 1$ and $|\bar{S} \cap T_r|\geq 1$ (where $\bar{S}$ stands for the complement of $S$), and $F_r(S) =0$ otherwise. The corresponding Lov\'{a}sz extension reads as $f_r(x) =  \max_{(i, j)\in H_r\times T_r} (x_i -x_j)_{+}$. If $S_r$ contains only two vertices for all values $r\in[R]$, the hypergraph clearly reduces to a graph; and similarly, if $H_r = T_r = S_r$, directed hypergraphs reduce to undirected hypergraphs. We also define the degree of a vertex $i$ as 
$$d_i = \sum_{r: i\in S_r} w_r,$$ 
and let $D$ denote the diagonal matrix of degree values. The volume of a set $S\subseteq V$ is defined as
\begin{align*}
\text{vol}(S) = \sum_{i\in S} d_i. 
\end{align*}
Furthermore, setting $m = \sum_{i\in V} d_i = \text{vol}(V)$ we define the \emph{boundary} of a set $S$ and its \emph{volume} as 
\begin{align*}
\partial S = \{r\in[R]| S\cap H_r \neq \emptyset, \bar{S}\cap T_r \neq \emptyset\}, \quad \text{vol}(\partial S) = \sum_{r\in \partial S} w_r = \sum_{r\in [R]}w_r F_r(S). 
\end{align*}
We also define a distribution $\pi_S$ over $V$ according to $(\pi_S)_i  = \frac{d_i}{\text{vol}(S)}$. Furthermore, using the boundary and volume, we also define the \emph{conductance} of the set $S$ as
 \begin{align}\label{conductanceset}
\Phi(S) = \frac{\text{vol}(\partial S)}{\min\{\text{vol}(S), \text{vol}(\bar{S})\}}.
\end{align}
The boundary, volume and conductance are all standard and well-studied graph-theoretic concepts~\citep{bollobas2013modern}, generalized above to accommodate hypergraphs. The \emph{hypergraph conductance} is defined as the minimal value of the conductance over all possible subsets of $V$, that is, 
\begin{align}\label{conductance}
\Phi^* = \min_{S\subset V, S\neq \emptyset, V} \Phi(S).
\end{align} 
The hypergraph conductance is an important objective function used for detecting communities of good quality in networks with higher-order functional units~\citep{benson2016higher}. Solving the optimization~\eqref{conductance} naturally yields a two-way partition of $V$ so that the corresponding procedure is also referred to as \emph{hypergraph partitioning}. \\

\subsection{The PageRank Process as QDSFM Optimization}
Next, we formally define the PR process for directed hypergraphs and show how it relates to the QDSFM problem. We borrow the definition of a diffusion process (DP) over directed hypergraphs proposed in~\citep{chan2017diffusion,chan2018spectral} to describe our PageRank (PR) process. To this end, we first need to define a Markov operator as follows. 

Suppose that $p\in \mathbb{R}^N$ is a potential over $[N]$ and let $x = D^{-1}p$. For all hyperedges $S_r$, we let $S_r^{\downarrow}(x) = \arg\min_{i\in T_r} x_i$ and $S_r^{\uparrow}(x) = \arg\max_{j\in H_r} x_j$. A hyperedge $S_r$ is said to be active if  $\arg\max_{j\in H_r} x_j > \arg\min_{i\in T_r} x_i$, and inactive otherwise.
\begin{enumerate}
\item For each active hyperedge $S_r$ and for each pair of vertices $(i,j) \in S_r^{\uparrow}(x) \times S_r^{\downarrow}(x)$, we introduce a non-negative scalar $a_r^{(ij)}$ such that $\sum\nolimits_{(i,j) \in S_r^{\uparrow}(x) \times S_r^{\downarrow}(x)} a_r^{(ij)} = w_r.$ For pairs of vertices $(i,j)\notin S_r^{\uparrow}(x) \times S_r^{\downarrow}(x)$  or belonging to inactive hyperedges, we set $a_r^{(i j)}=0$.
\item Using the above described pairwise weights, we create a symmetric matrix $A$ such that $A_{ij} = \sum_{r\in [R]} a_r^{(ij)}$ and $A_{ii} = d_i - \sum_{i,j\in V, j\neq i} A_{ij}$.
\end{enumerate}    
For undirected graphs for which $|S_r| = |T_r| = |H_r| = 2$ for all $r\in [R]$, the matrix $A$ is an adjacency matrix of a graph. Using the terminology of~\citep{chan2018spectral}, we define the Markov operator $M(p)$ based on $A$ according to 
$$M(p) = Ax = AD^{-1}p.$$ 
The Markov operator $M(p)$ is closely related to the Lov\'{a}sz extension $f_r$, which may be seen as follows. 
Let $\nabla f_r(x)$  be the subdifferential set of $f_r$ at $x$. We once again recall that as $F_r$ is submodular, $ f_r(x) = \arg\max_{y\in B_r} \langle y, x\rangle$~\citep{bach2013learning}. This gives rise to the following result.
\begin{lemma}\label{lem:PR1}
For any $p\in \mathbb{R}^N$ and $x= D^{-1}p$, it holds that $p - M(p) \in \sum_{r\in [R]} w_r f_r(x)\nabla f_r(x)$.
\end{lemma}
\begin{proof}
Consider a vertex $i\in V$. Then,
\begin{align*}
&(p - M(p))_i = p_i - A_{ii}\frac{p_i}{d_i} - \sum_{j:\in V, j \neq i} A_{ij}\frac{p_j}{d_j}  = \sum_{j:\in V, j \neq i} A_{ij}(x_i - x_j)  =  \sum_{r\in[R]}  \sum_{j: \in V, j \neq i} a_r^{(ij)}(x_i - x_j) \\
& \stackrel{1)}{=}   \sum_{r: \in[R],\,S_r \,\text{active},\, i\in S_r^{\uparrow}} w_rf_r(x)    \sum_{j:\in  S_r^{\downarrow}} a_r^{(ij)}/w_r +  \sum_{r: \in[R],\, S_r \,\text{active},\, i\in S_r^{\downarrow}} w_rf_r(x)   \sum_{j:\in S_r^{\uparrow}} (-a_r^{(ij)})/w_r 
\end{align*}
where 1) holds due to the definitions of $a_r^{(ij)}$ and $f_r(x)$. It only remains to show that for a $r\in [R]$, active $S_r$, the vector $v$  
\begin{equation*}
v_i =\left\{
\begin{array}{cc}
\sum_{j:\in  S_r^{\downarrow}} a_r^{(ij)}/w_r & i\in S_r^{\uparrow},\\
 -\sum_{j:\in  S_r^{\uparrow}} a_r^{(ij)}/w_r & i\in S_r^{\downarrow}, \\
 0& \text{otherwise}
 \end{array}\right.
\end{equation*}
satisfies $v\in \nabla f_r(x)$. Since $\sum_{(i,j)\in S_r^{\uparrow}\times S_r^{\downarrow}} a_{r}^{(ij)} = w_r$ and $a_{r}^{(ij)}\geq 0$, we know that $v\in B_r$. Moreover, because $v\in \nabla f_r(x)$, we have $\langle v, x\rangle = f_r(x)$ and thus the claim clearly holds.
\end{proof}
The statement of Lemma~\ref{lem:PR1} describes a definition of a Laplacian operator for submodular hypergraphs consistent with previous definitions~\citep{li2018submodular}. We write the underlying Laplacian operator as $L(x) = (D - A)x = p - M(p)$. Note that when $p=\pi_V$, the components of $x=D^{-1}p$ are equal and hence $M(p) = AD^{-1}p =  \frac{1}{\text{vol}(S)}A\mathbf{1}  =  \frac{1}{\text{vol}(S)}D\mathbf{1} = p$. Therefore, $\pi_V$ may be viewed as an analogue of a stationary distribution.

Given an initial potential $p_0\in \mathbb{R}^N$ and an $\alpha\in (0,1]$, the PR process is described by the following ordinary differential equation:
\begin{align}\label{eqn:pr-process}
\textbf{The PR process} \quad \frac{dp_t}{dt} = \alpha (p_0- p_t) + (1-\alpha)(M(p_t)-p_t).  
\end{align}
The choice of $\alpha = 0$ reduces the PR process to the DP defined in~\citep{chan2017diffusion,chan2018spectral}. Tracking the PR process at every time point $t$, i.e., specifying $a_{r}^{(ij)}$,  is as difficult as tracking DP, which in turn requires solving a densest subset problem\footnote{To track the diffusion of $p(t)$~\eqref{eqn:pr-process}, one has to be careful in terms of deciding which components of $p(t)$ are allowed to change within an infinitesimal time. The selection procedure should guarantee that after an infinitesimal time $p(t)\rightarrow p(t) + dp(t)$, $M(p(t))$ and $M(p(t)+ dp(t))$ satisfy $M(p(t))=M'p(t)$ and $M(p(t) + dp(t)) = M'p(t) + M'dp(t)$ for some shared choice of the matrix $M'$. Here, $M'$ captures the active diffusion network that remains unchanged in $[t, t+dt)$. In general, $M'$ will transition through discrete states in a continuous-time domain. Determining which components of $p(t)$ are allowed to change is equivalent to solving a densest set problem.}~\citep{chan2017diffusion,chan2018spectral}. However, we only need to examine and evaluate the stationary point which corresponds to the PR vector. The PR vector $pr(\alpha, p_0)$ satisfies the equation
\begin{align} \label{eq:PRpoint}
 \alpha (p_0- pr(\alpha, p_0)) + (1-\alpha)(M(pr(\alpha, p_0))-pr(\alpha, p_0))   = 0.
\end{align}
From Lemma~\ref{lem:PR1}, we know that $pr(\alpha, p_0)$ can be obtained by solving a QDSFM. Let $x_{pr}= D^{-1} pr(\alpha, p_0)$. Then  
\begin{align}\label{eq:PRopt}
 &\alpha (p_0- Dx_{pr}) -(1-\alpha)\sum_{r\in[R]} w_r f_r(x_{pr})\nabla f_r(x_{pr})  \ni 0 \nonumber \\
 & \Leftrightarrow x_{pr} = \arg\min_{x} \|x - x_0\|_{W}^2 + \sum_{r\in[R]} [f_r'(x)]^2,  
\end{align}
where $x_0 = D^{-1} p_0, W=\frac{\alpha}{1-\alpha}D$ and $f_r' = \sqrt{w_r} f_r$. 

\subsection{Complexity of Computing PageRank over Directed Hypergraphs}\label{appsec:exactprojection}

We now analyze the computation complexity of computing PageRank over directed hypergraphs using QSDFM solvers. First, due to the specific structure of the Lov\'{a}sz extension $f_r'$ in~\eqref{eq:PRopt}, we devise an algorithm of complexity $O(|S_r|\log |S_r|)$ that \emph{exactly} computes the required projections onto the induced cones. This projection algorithm is much more efficient than the conic FW and MNP methods designed for general polytopes. We focus on projections with respect to one particular value $r$; since the scaling $w_r$ does not affect the analysis, we henceforth omit the subscript $r$ and the superscript $'$ and simply use $f$ instead of $f_r'$.

The key idea is to revert the projection $\Pi_{C}(a; \tilde{W})$ back to its primal form. In the primal setting, optimization becomes straightforward and only requires careful evaluation of the gradient values. First, following a similar strategy as described in Lemma~\ref{dualform}, it is easy to show that 
\begin{align*}
\min_{z} \frac{1}{2}\|z - b\|_{W}^2 + \frac{1}{2}[f(z)]^2
\end{align*}
is the dual of the problem~\eqref{projection}, where $W = \tilde{W}^{-1}$, $b = \frac{1}{2}W^{-1}a$, and  $f$ is the Lov{\'a}sz extension corresponding to the base polytope $B$. Then, one has $y = a - 2Wz$, $\phi= 2\langle y, z \rangle_{W^{-1}}$. 

Next, recall that for a directed hyperedge, we introduced head and tail sets $H, T$, respectively, and $f(z) = (\max_{i\in H}z_j -  \min_{j\in T}z_j)_+$. Clearly, when $H,\,T$ both equal to the corresponding incidence set, the directed hypergraph reduces to an undirected hypergraph. To solve this optimization problem, define two intermediate variables $\gamma = \max_{i\in H}z_i $ and $\delta = \min_{j\in T}z_j$. Denote the derivatives with respect to $\gamma$ and $\delta$ as $\triangle_\gamma$ and $\triangle_\delta$, respectively. 
Hence, $\triangle_\gamma$ and $\triangle_\delta$ saitsfy
\begin{align*}
\triangle_\gamma = (\gamma - \delta)_+ + \sum_{i\in S_H(\gamma)} W_{ii}(\gamma - b_i), \quad \triangle_\delta = -(\gamma - \delta)_+ + \sum_{j\in S_T(\delta) } W_{jj}(\delta - b_j),
\end{align*}
where $S_H(\gamma) = \{i| i\in H, b_i \geq \gamma\}$ and $S_T(\delta) = \{j| j\in T, b_j \leq \delta\}$. The optimal values of $\gamma$ and $\delta$ are required to simultaneously satisfy $\triangle_\gamma=0$ and $\triangle_\delta=0$. 
Algorithm 5 is designed to find such values of $\gamma$ and $\delta$. The ``search'' for $(\gamma, \delta)$ starts from $(\max_{i\in H} b_i, \min_{j\in T} b_j)$, and one gradually decreases $\gamma$ and increases $\delta$ while ensuring $\triangle_\gamma = -\triangle_\delta$ (see Steps 5-10 in Algorithm 5). The reason for imposing the constraint $\triangle_\gamma = -\triangle_\delta$ during the procedure is to decrease the number of tuning parameters. As a result, the constraint implicitly transforms the procedure into simple line search.
The complexity of Algorithm 5 is dominated by the sorting step (Step 1), which requires $O(|S_r|\log |S_r|)$ operations, as $w_H$, $w_T$, $\triangle_\gamma$, and $\triangle_\delta$ can all be efficiently tracked within the inner loops.  

\begin{table}[t]
\centering
\begin{tabular}{l}
\hline
\label{generalalg}
\textbf{Algorithm 5: } \textbf{An Exact Projection Algorithm for a Directed Hyperedge} \\
\hline
\textbf{Input}:  $W$, $b$, $H$, $T$\\
1.\; Sort $\{b_i\}_{i\in H}$ and $\{b_j\}_{j\in T}$. \\
2.\; Initialize $\gamma \leftarrow \max_{i\in H} b_i$ and $\delta \leftarrow \min_{j\in T} b_j$. \\
3.\; \textbf{If} $\gamma \leq \delta$, \textbf{return} $z = b$. \\
4.\; Iteratively execute: \\
5.  \;\;\;\;  $w_H\leftarrow \sum_{i\in S_H(\gamma)}W_{ii}$, $w_T \leftarrow \sum_{j\in S_T(\delta)}W_{jj}$ \\
6.  \;\;\;\;  $\gamma_1 \leftarrow \max_{i\in H/S_H(\gamma)} b_v$, $\delta_1 \leftarrow \delta + (\gamma - \gamma_1)w_H/w_T $\\
7.  \;\;\;\;  $\delta_2 \leftarrow \min_{j\in T/S_T(\delta)} b_v$, $\gamma_2 \leftarrow \gamma - (\delta_2 - \delta)w_T/w_H  $ \\
8.  \;\;\;\;  $k^* \leftarrow \arg\min_{k\in\{1,2\}} \delta_k $ \\
9.  \;\;\;\;  \textbf{If} $\gamma_{k^*} \leq \delta_{k^*}$ or $\triangle_{\gamma_{k^*}} \leq 0$, \textbf{break} \\
10.\;\;\;\;  $(\gamma, \delta) \leftarrow (\gamma_{k^*}, \delta_{k^*})$ \\
11.\;$(\gamma, \delta) \leftarrow  ( w_T,  w_H)\frac{\triangle_{\gamma}}{w_Tw_H + w_T + w_H}$ \\
12.\;Set $z_i$  to $\gamma$, if $i\in S_H(\gamma)$, to $\delta$, if $i\in S_T(\delta)$, and to $b_i$ otherwise. \\
\hline
\end{tabular}
\end{table}

Combining the projection method in Algorithm 5 and Corollary~\ref{specialcase} with parameter choice $\beta=\frac{\alpha}{1-\alpha}$, we arrive at the following result summarizing the overall complexity of running the PageRank process on (un)directed hypergraphs. 
\begin{corollary}
The PageRank problem on (un)directed hypergraphs can be obtained by solving~\eqref{eq:PRopt}. A combination of Algorithm 1 (RCD) and Algorithm 5 returns an $\epsilon$-optimal solution in expectation with total computation complexity: 
$$O\left(N^2\max\left\{1,9\frac{1-\alpha}{\alpha}\right\}\max_{i,j\in[N]}\frac{D_{ii}}{D_{jj}}\sum_{r\in[R]} |S_r|\log|S_r|\log\frac{1}{\epsilon}\right).$$ 
A combination of Algorithm 2 (AP) and Algorithm 5 returns an $\epsilon$-optimal solution with total computation complexity:  
$$O\left(N^2\max\left\{1,9\frac{1-\alpha}{\alpha}\right\}\max_{i,j\in[N]}\frac{\Psi_{jj}D_{ii}}{D_{jj}}\sum_{r\in[R]} |S_r|\log|S_r|\log\frac{1}{\epsilon}\right).$$
\end{corollary}

\subsection{Analysis of Partitions Based on PageRank}
We are now ready to analyze hypergraph partitioning procedures~\eqref{conductance} based on our definition of PageRank. The optimization problem~\eqref{conductance} is NP-hard even for graphs so we only seek approximate solutions. Consequently, an important application of our version of PageRank is to provide an approximate solution for~\eqref{conductance}. We start by proving that our version of PageRank, similarly to the standard PageRank over graphs~\citep{andersen2006local}, satisfies a mixing result derived based on so-called Lov\'{a}sz-Simonovits curves~\citep{lovasz1990mixing}. This result may be further used to prove that \emph{Personalized} PageRank leads to partitions with small conductance\footnote{Recall that personalized PageRank is a PR process with initial distribution $p_0 = 1_i$, for some vertex $i\in V$.}. The main components of our proof are inspired by the proofs for PageRank over graphs~\citep{andersen2006local}. The novelty of our approach is that we need several specialized steps to handle the nonlinearity of the Markov operator $M(\cdot)$, which in the standard graph setting is linear. We postpone the proofs of all our results to Section~\ref{proof:PR}. 

We first introduce relevant concepts and terminology necessary for presenting our results. 

\emph{Sweep cuts} are used to partition hypergraphs based on the order of components of a distribution $p$ and are defined as follows. Set $x= D^{-1} p$ and sort the components of $x$ so that
\begin{align}\label{eq:PRorder}
x_{i_1} \geq x_{i_2} \geq \cdots \geq x_{i_N}.
\end{align}
Let $\mathcal{S}_j^p = \{x_{i_1}, x_{i_2}, ..., x_{i_j}\}$. Recall the definition of the conductance of a set $\Phi(\cdot)$~\eqref{conductanceset} and evaluate 
$$j^* = \arg\min_{j\in [N]} \Phi(\mathcal{S}_j^p).$$ 
The pair $(\mathcal{S}_{j^*}^p, \bar{\mathcal{S}}_{j^*}^p)$ is referred to as a \emph{sweep cut}. Furthermore, define $\Phi_p = \Phi(\mathcal{S}_{j^*}^p)$; in words, $\Phi_p$ is the smallest conductance that may be achieved by sweep cuts of a potential $p$. Our target is to show that the conductance based on a Personalized PR vector, i.e., $\Phi_{pr(\alpha, 1_v)}$, may approximate $\Phi^*$. 

Intuitively, the conductance of a hypergraph for a given $p_0$ and $\alpha$ captures how ``close'' a PR vector $pr(\alpha, p_0)$ is to the stationary distribution $\pi_V$: a large conductance results in mixing the potentials of different vertices and thus makes $pr(\alpha, p_0)$ closer to $\pi_V$. To characterize the degree of ``closeness'', we introduce the Lov\'{a}sz-Simonovits curve. Given a potential vector $p\in \mathbb{R}^N$ and $x = D^{-1}p$, suppose that the order of the components in $x$ follows 
equation~\eqref{eq:PRorder}. Let $\text{vol}(\mathcal{S}_0^p) = 0$. Define a piecewise linear function $I_p(\cdot): [0, m]\rightarrow [0,1]$ according to
\begin{align*}
I_p (z) &= p(\mathcal{S}_{j-1}^p) + \frac{z - \text{vol}(\mathcal{S}_{j-1}^p)}{d_{v_{j}}} p_{v_{j}}, \;  \text{for}\; \text{vol}(\mathcal{S}_{j-1}^p) \leq  k \leq \text{vol}(\mathcal{S}_{j}^p), \; j\in[N]. 
\end{align*}
Note that the Lov\'{a}sz-Simonovits curve of $\pi_V$ contains only one segment by its very definition. It is easy to check that $I_p(z)$ is continuous and concave in $z$. Moreover, for any set $S\subseteq [N]$, we have $p(S) \leq I_p(\text{vol}(S))$. We further write $V_p = \{j: x_{i_j}>x_{i_{j+1}}\}$ and refer to $\{\text{vol}(\mathcal{S}_j^p)\}_{j\in V_p}$ as the (exact) break points of $I_p$. The relevance of the Lov\'{a}sz-Simonovits curve in the aforementioned context may be understood as follows. Although the  Lov\'{a}sz-Simonovits curve is piecewise linear, the changes of the slopes at the break points may be used to describe its ``curvatures''. A large (small) curvature corresponds to the case that the distribution of the potential is far from (close to) the stationary distribution, which further implies a small (large) conductance of the graphs/hypergraphs (see Theorem~\ref{thm:PRmixing}). Therefore, analyzing the Lov\'{a}sz-Simonovits curvature is of crucial importance in the process of bounding the conductance. The first step of this analysis is presented in Lemma~\ref{lem:RPmixing1}, which establishes the ``curvatures'' at the break points.
\begin{lemma}\label{lem:RPmixing1}
Let $p = pr(\alpha, p_0)$, $x= D^{-1}p$ and $j\in V_p$. Then, 
\begin{align*}
I_p(\text{vol}(\mathcal{S}_j^p)) \leq \frac{\alpha}{2-\alpha} p_0(\mathcal{S}_j^p) + \frac{1- \alpha}{2-\alpha}[I_p(\text{vol}(\mathcal{S}_j^p) - \text{vol}(\partial \mathcal{S}_j^p) )+I_p(\text{vol}(\mathcal{S}_j^p) + \text{vol}(\partial \mathcal{S}_j^p) ].
\end{align*}
Furthermore, for $k\in[0, m]$,
\begin{align*}
I_p(k) \leq  p_0(k) \leq  I_{p_0}(k).
\end{align*}
\end{lemma}
\begin{remark}
In comparison with the undirected graph case (Lemma 5~\citep{andersen2006local}), where the 
result holds for arbitrary $S\subseteq V$, our claim is true only for $\mathcal{S}_j^p$ for which $j\in V_p$. 
However, this result suffices to establish all relevant PR results.  
\end{remark}
Using the upper bound on the break points in $I_p$, we can now construct a curve that uniformly bounds $I_p$ using $\Phi_p$.
\begin{theorem}\label{thm:PRmixing}
Let $p = pr(\alpha, p_0)$ be the previously defined PR vector, and let $\Phi_p$ be the minimal conductance 
achieved by a sweep cut. For any integer $t\geq 0$ and any $k\in [0, m]$, the following bound holds:
\begin{align*}
I_p(k) \leq \frac{k}{m} + \frac{\alpha}{2-\alpha}t + \sqrt{\frac{\min\{k, m-k\}}{\min_{i: (p_0)_i> 0} d_i}} \left(1- \frac{\Phi_p^2}{8}\right)^t.
\end{align*} 
\end{theorem}

For graphs~\citep{andersen2006local}, Theorem~\ref{thm:PRmixing} may be used to characterize the graph partitioning property of sweep cuts induced by a personalized PageRank vector. Similarly, for general directed hypergraphs, we establish a similar result in Theorem~\ref{thm:PRpartition}.
\begin{theorem}\label{thm:PRpartition}
Let $S$ be a set of vertices such that $\text{vol}(S) \leq \frac{1}{2} M$ and $\Phi(S) \leq \frac{\alpha}{c},$ for some constants $\alpha, c$. If there exists a distribution $P$ for sampling vertices $v\in S$ such that $\mathbb{E}_{i\sim P}[pr(\alpha, 1_i)(\bar{S})] \leq \frac{c}{8} \, pr(\alpha, \pi_S)(\bar{S})$, then with probability at least $\frac{1}{2}$, 
$$\Phi_{pr(\alpha, 1_i)} = O(\sqrt{\alpha\log \frac{\text{vol}(S)}{d_i}}),$$ 
where $i$ is sampled according to $P$.
\end{theorem}

\begin{remark}
Note that in Theorem~\ref{thm:PRpartition}, we require that the sampling distribution $P$ satisfy $\mathbb{E}_{i\sim P}[pr(\alpha, 1_i)(\bar{S})] \leq \frac{c}{8} \, pr(\alpha, \pi_S)(\bar{S})$. For graphs, when we sample a vertex $i$ with probability proportional to its degree $d_i$, this condition is naturally satisfied, with $c=8$. However, for general (un)directed hypergraphs, the sampling procedure is non-trivial and harder to handle due to the non-linearity of the random-walk operator $M(\cdot)$. We relegate a more in-depth study of this topic to a companion paper.  
\end{remark}

\section{Applications to Semi-supervised Learning and Numerical Experiments}\label{sec:ssl}

Another important application of QDSFM is semi-supervised learning (SSL). SSL is a learning paradigm that allows one to utilize the underlying structure or distribution of unlabeled samples whenever the information provided by labeled samples is insufficient for learning an inductive predictor~\citep{gammerman1998learning, joachims2003transductive}. A standard setup for a $K$-class transductive learner is as follows: given $N$ data points $\{z_i\}_{i\in[N]},$ along with labels for the first $l$ ($\ll N$) samples $\{y_i| y_i \in [K] \ \}_{i\in[l]}$, the learner is asked to infer the labels of all the remaining data points $i\in[N]/[l]$. The widely-used SSL problem with least square loss requires one to solve $K$ regularized problems. For each class $k\in [K]$, one sets the scores of data points within the class to
$$\hat{x}^{(k)} = \arg\min_{x^{(k)}} \beta \|x^{(k)} - a^{(k)}\|^2 + \Omega(x^{(k)}),$$
where $a^{(k)}$ describes the information provided by the known labels, i.e.,  $a_i^{(k)} = 1$ if $y_i = k$, and $0$ otherwise, $\beta$ denotes a hyperparameter and $\Omega$ is a smoothness regularizer. The labels of the data points are estimated as 
$$\hat{y}_{i} = \arg\max_{k} \{\hat{x}_i^{(k)}\}.$$ 
In typical graph and hypergraph learning problems, $\Omega$ is chosen to be a Laplacian regularizer constructed using $\{z_i\}_{i\in[N]}$; this regularization term ties the above learning problems to PageRank (again, refer to Table~\ref{tab:TLexamples}). With the Laplacian regularization, each edge/hyperedge corresponds to one decomposed part $f_r$ in the QDSFM problem. The variables themselves may be normalized with respect to their degrees, in which case the normalized Laplacian is used instead. For example, in graph learning, one of the terms in $\Omega$ assumes the form $w_{ij} (x_i/\sqrt{d_i} - x_j/\sqrt{d_j})^2,$ where $d_i$ and $d_j$ correspond to the degrees of the vertex variables $i$ and $j$, respectively. Using some simple algebra, it can be shown that the normalization term is embedded into the matrix $W$ used in the definition of the QDSFM problem~\eqref{QDSFM}.

\subsection{Experimental Setup}\label{sec:exp}

We numerically evaluate our SSL learning framework for hypergraphs on both real and synthetic datasets. For the particular problem at hand, the QDSFM problem can be formulated as follows: 
\begin{align}\label{expobj}
\min_{x\in \mathbb{R}^N} \beta \|x - a\|^2 + \sum_{r\in [R]}\max_{i,j\in S_r}(\frac{x_i}{\sqrt{W_{ii}}} - \frac{x_j}{\sqrt{W_{jj}}})^2,
\end{align}
where $a_i \in\{-1, 0, 1\}$ indicates if the corresponding variable $i$ has a negative, missing, or positive label, respectively, $\beta$ is a parameter that balances out the influence of observations and the regularization term,  $\{W_{ii}\}_{i\in[N]}$ defines a positive measure over the variables and may be chosen to be the degree matrix $D$, with $D_{ii} = |\{r\in [R]: i \in S_r\}|$. Each part in the decomposition corresponds to one hyperedge. We compare eight different solvers falling into three categories: (a) our proposed general QDSFM solvers, \emph{QRCD-SPE, QRCD-MNP, QRCD-FW and QAP-SPE};  (b) alternative specialized solvers for the given problem~\eqref{expobj}, including \emph{PDHG}~\citep{hein2013total}  and \emph{SGD}~\citep{zhang2017re}; (c) SSL solvers that do not use QDSFM as the objective, including \emph{DRCD}~\citep{ene2015random} and  \emph{InvLap}~\citep{zhou2007learning}. The first three methods all have outer-loops that execute RCD, but with different inner-loop projections computed via the \emph{exact projection algorithm} for undirected hyperedges (see Algorithm 5 in Section~\ref{appsec:exactprojection}), or the generic MNP and FW. As it may be time consuming to find the precise projections via MNP and FW, we always fix the number of MAJOR loops of the MNP and the number of iterations of the FW method to $100|S_r|$ and $100|S_r|^2,$ respectively. Empirically, these choices provide an acceptable trade-off between accuracy and complexity. The QAP-SPE method uses AP in the outer-loop and exact inner-loop projections (see Algorithm 5 in Section~\ref{appsec:exactprojection}). PDHG and SGD are the only known solvers for the specific objective~\eqref{expobj}. PDHG and SGD depend on certain parameters that we choose in standard fashion: for PDHG, we set $\sigma = \tau = \frac{1}{\sqrt{1+ \max_{i} D_{ii}}}$ and for SGD, we set $\eta_k = \frac{1}{k\beta\max_{i} W_{ii}}$. 

DRCD is a state-of-the-art solver for DSFM and also uses a combination of outer-loop RCD and inner-loop projections. InvLap first transforms hyperedges into cliques and then solves a Laplacian-based linear problem. All the aforementioned methods, except InvLap, are implemented in C++ in a nonparallel fashion. InvLap is executed via matrix inversion operations in Matlab, and may be parallelized. The CPU times of all methods are recorded on a 3.2GHz Intel Core i5. The results are summarized for $100$ independent tests. When reporting the results, we use the primal gap (``gap'')~\footnote{We place a high accuracy requirement for QDSFM-MNP by ignoring the CPU time needed to guarantee that it achieves a duality gap as low as $1e-14$. Note that the duality gap is computed according to $\|x^{(k)} - a\|_W^2 + \sum_{r\in[R]}\left[f_r(x^{(k)})\right]^2 - g(y^{(k)}, \phi^{(k)})/4$, where $x^{(k)} = a - \frac{1}{2}W^{-1}\sum_{r\in[R]} y_r^{(k)}$, which is an upper bound on the primal gap. Therefore, the achieved value of the objective has high accuracy and thus can be used as the optimal primal value.} to characterize the convergence properties of different solvers. 

\subsection{Synthetic Data} 
We generated a hypergraph with $N = 1000$ vertices that belong to two equal-sized clusters. We uniformly at random generated $500$ hyperedges within each cluster and $1000$ hyperedges across the two clusters. Note that in higher-order clustering, we do not need to have many hyperedges within each cluster to obtain good clustering results. Each hyperedge includes $20$ vertices. We also uniformly at random picked $l=1,2,3,4$ vertices from each cluster to represent labeled datapoints. With the vector $x$ obtained by solving~\eqref{expobj}, we classified the variables based on the Cheeger cut rule~\citep{hein2013total}: suppose that 
$\frac{x_{i_1}}{\sqrt{W_{i_1i_1}}}\geq \frac{x_{i_2}}{\sqrt{W_{i_2i_2}}}\geq \cdots \geq \frac{x_{i_N}}{\sqrt{W_{i_Ni_N}}},$
and let $\mathcal{S}_j = \{i_1, i_2, ..., i_j\}$. We partitioned $[N]$ into two sets $(\mathcal{S}_{j^*}, \bar{\mathcal{S}}_{j^*}),$ where 
$$j^* = \arg\min_{j\in[N]} \Phi(\mathcal{S}_j)\triangleq \frac{| S_r \cap \mathcal{S}_j\neq \emptyset, S_r \cap \bar{\mathcal{S}}_j \neq \emptyset\}|}{\max\{\sum_{r\in[R]} |S_r \cap \mathcal{S}_j|, \sum_{r\in[R]} |S_r \cap \bar{\mathcal{S}}_j|\}}.$$ 
The classification error is defined as (\# of incorrectly classified vertices)$/N$. In the experiment, we used $W_{ii} = D_{ii}$, $\forall\,i,$ and tuned 
$\beta$ to be nearly optimal for different objectives with respect to the classification error rates: for QDSFM as the objective, using QRCD-SPE, QAP-SPE, PDHG, and SGD as the methods of choice, we set $\beta = 0.02$; for DSFM as the objective, including the DRCD method, we set $\beta = 1$; for InvLap, we set $\beta = 0.001$. 

The left subfigure of Figure~\ref{fig:synthetic1} shows that QRCD-SPE  converges much faster than all other methods when solving the problem~\eqref{expobj} according to the gap metric (we only plotted the curve for $l=3$ as all other values of $l$ produce similar patterns). To avoid clutter, we moved the results for QRCD-MNP and QRCD-FW to the right two subfigures in Figure~\ref{fig:synthetic1}, as these methods are typically $100$ to $1000$ times slower than QRCD-SPE. In Table~\ref{tab:synthetic2}, we described the performance of different methods with comparable CPUtimes. Note that when QRCD-SPE converges (with primal gap $10^{-9}$ achieved after $0.83$s), the obtained $x$ leads to a much smaller classification error than other methods. QAP-SPE, PDHG and SGD all have large classification errors as they do not converge within short CPU time-frames. QAP-SPE and PDHG perform only a small number of iterations, but each iteration computes the projections for all the hyperedges, which is more time-consuming. The poor performance of DRCD implies that the DFSM is not a good objective for SSL. 
InvLap achieves moderate classification errors, but still does not match the performance of QRCD-SPE. 
Furthermore, note that InvLap uses Matlab, which is optimized for matrix operations, thus fairly efficient.
However, for experiments on real datasets with fewer but significantly larger hyperedges, InvLap does not offer as good  performance as the one for synthetic data. The column ``Average $100c(\mathcal{S}_{j^*})$'' also illustrates that the QDSFM objective is a good choice for finding approximate balanced cuts of hypergraphs.  
\begin{figure*}[t]
\centering
\includegraphics[trim={0cm 0cm 0cm 0},clip,width=.32\textwidth]{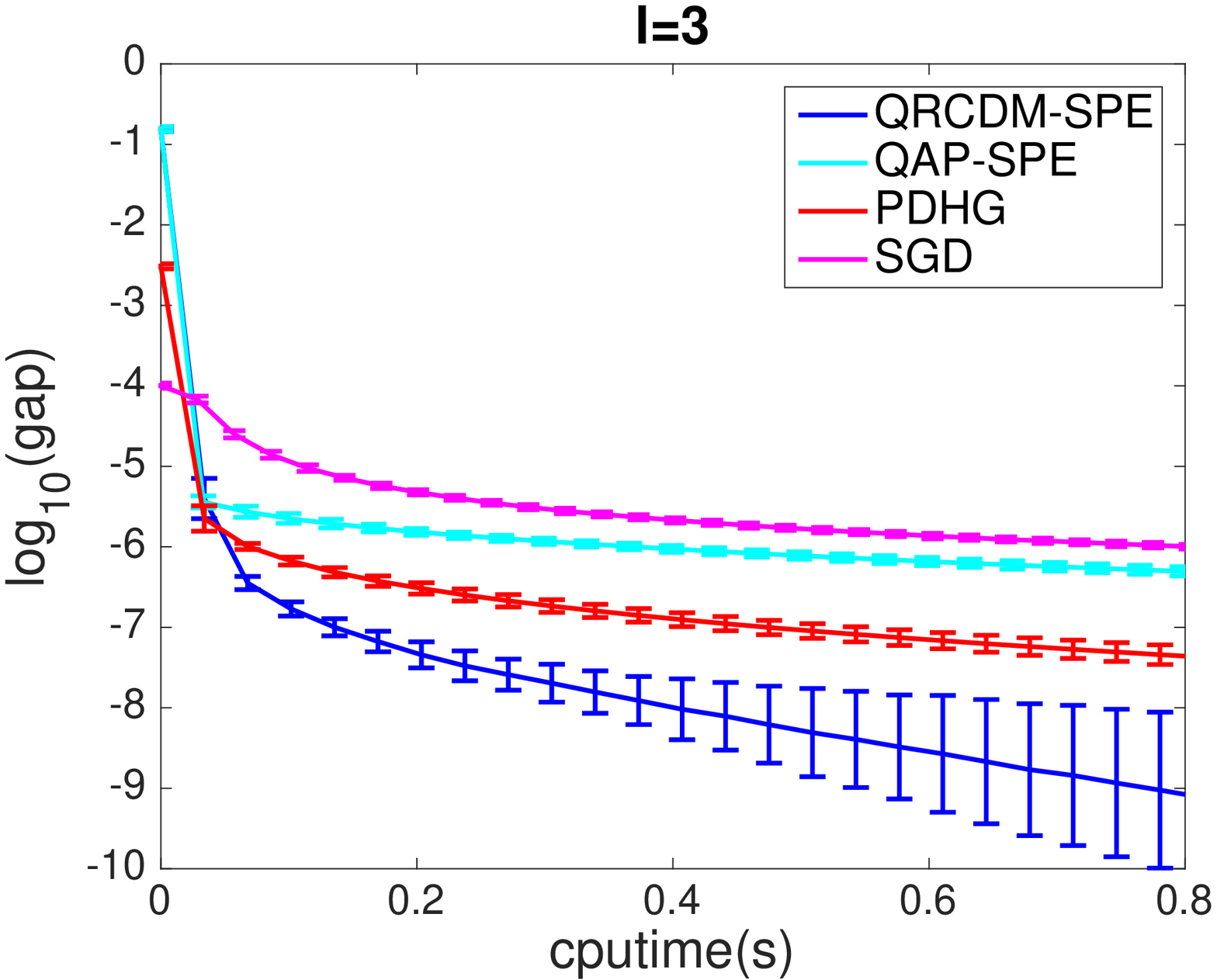}
\includegraphics[trim={0cm 0cm 0cm 0},clip,width=.32\textwidth]{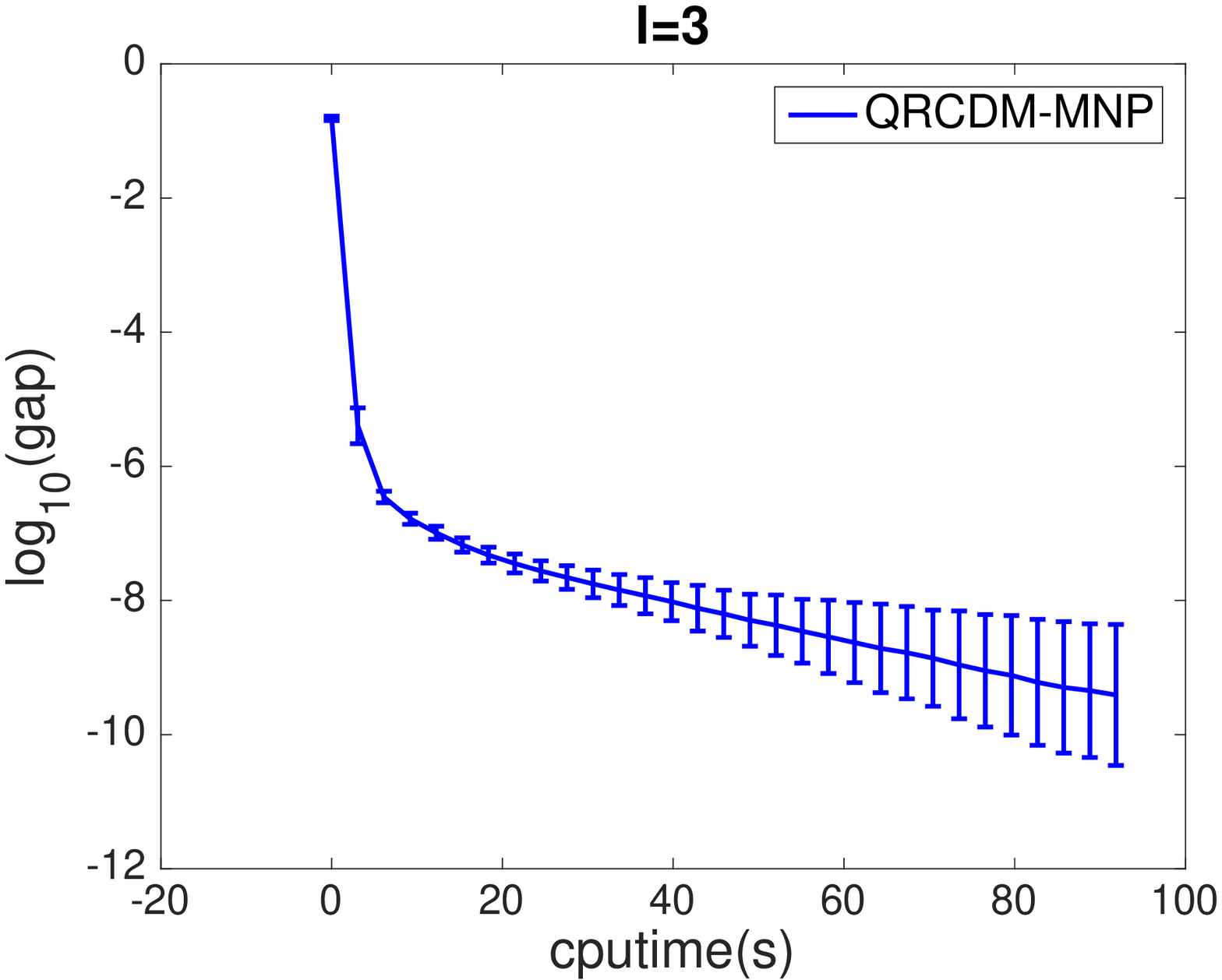}
\includegraphics[trim={0cm 0cm 0cm 0},clip,width=.32\textwidth]{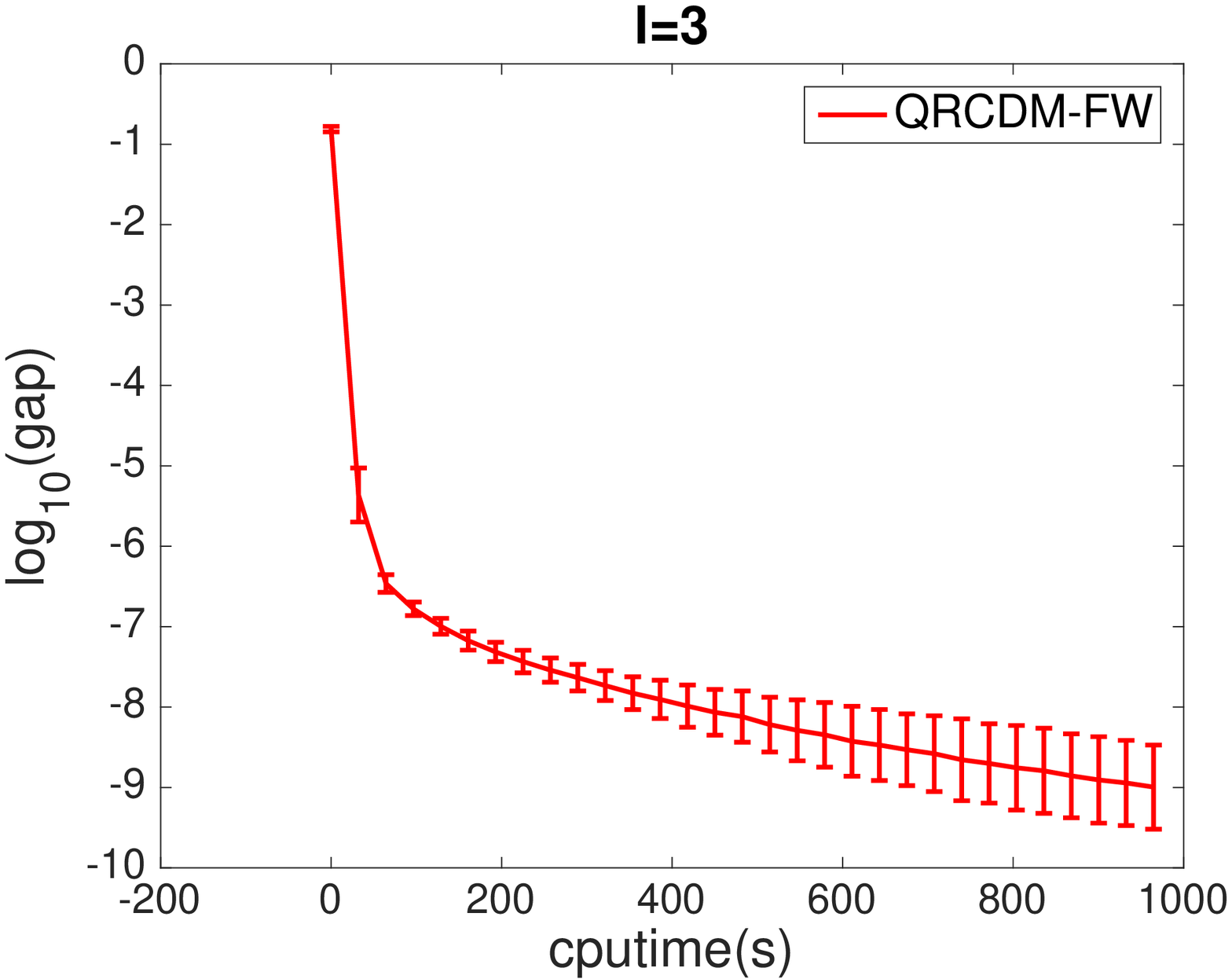}
\caption{Convergence results on synthetic datasets: gap vs CPU time of different QDSFM solvers (with an average $\pm$ standard deviation).}
\label{fig:synthetic1}
\end{figure*}
\begin{table*}[t]
\centering
\tiny{
\begin{tabular}{p{0.3cm}<{\centering}|p{1.35cm}<{\centering}| p{0.35cm}<{\centering}| p{0.35cm}<{\centering}| p{0.35cm}<{\centering}| p{0.35cm}<{\centering}| p{0.35cm}<{\centering}| p{0.35cm}<{\centering}| p{0.35cm}<{\centering}| p{0.35cm}<{\centering}|  p{0.35cm}<{\centering}| p{0.35cm}<{\centering}| p{0.35cm}<{\centering}| p{0.35cm}<{\centering}| p{1.1cm}<{\centering}| p{0.9cm}<{\centering}}
\hline
\multirow{3}{*}{Obj.} &\multirow{3}{*}{Solvers}& \multicolumn{8}{c|}{Classification error rate (\%)}  & \multicolumn{4}{c|}{Average 100$c(\mathcal{S}_{j^*})$} &  \multirow{3}{*}{\#iterations} &  \multirow{3}{*}{cputime(s)} \\
\cline{3-14}
&&  \multicolumn{2}{c|}{l=1} & \multicolumn{2}{c|}{l=2} &\multicolumn{2}{c|}{l=3} & \multicolumn{2}{c|}{l=4} & \multirow{2}{*}{l=1} & \multirow{2}{*}{l=2} & \multirow{2}{*}{l=3} & \multirow{2}{*}{l=4}  & & \\
\cline{3-10}
&& MN & MD & MN & MD & MN & MD & MN & MD & & & & & &    \\
\hline
\parbox[t]{2mm}{\multirow{4}{*}{\rotatebox[origin=c]{90}{QDSFM}}} &QRCD-SPE & \textbf{2.93} & \textbf{2.55}  & \textbf{2.23} & \textbf{0.00} & \textbf{1.47} & \textbf{0.00} & \textbf{0.78} & \textbf{0.00} & \textbf{6.81} & \textbf{6.04} & \textbf{5.71} & \textbf{5.41} & $4.8\times 10^5$ & 0.83  \\
&QAP-SPE & 14.9 & 15.0  & 12.6 & 13.2 & 7.33 & 8.10 & 4.07 & 3.80 & 9.51 & 9.21 & 8.14 & 7.09 & $2.7\times 10^2$ & 0.85  \\
&PDHG & 9.05 & 9.65 & 4.56 & 4.05 & 3.02 & 2.55 & 1.74 & 0.95  & 8.64 & 7.32 & 6.81 & 6.11 & $3.0\times 10^2$ & 0.83  \\
&SGD & 5.79 & 4.15 & 4.30 & 3.30 & 3.94 & 2.90 & 3.41 & 2.10 & 8.22 & 7.11 & 7.01 & 6.53 & $1.5 \times 10^4$ &0.86  \\
\cline{1-16}
\multirow{2}{*}{\rotatebox[origin=c]{90}{\tiny{Oth.}}}  &DRCD & 44.7 &44.2 & 46.1 & 45.3 & 43.4 & 44.2 & 45.3 & 44.6 & 9.97 & 9.97 & 9.96 & 9.97 & $3.8 \times 10^6$  &0.85   \\
&InvLap & 8.17 & 7.30 & 3.27  & 3.00 & 1.91 & 1.60 & 0.89 &  0.70 & 8.89 & 7.11 & 6.18 & 5.60 & --- &\textbf{0.07} \\
\hline
\end{tabular}
}
\caption{Prediction results on synthetic datasets: classification error rates $\&$ Average 100 $c (\mathcal{S}_{j^*})$ for different solvers (MN: mean, MD: median).}
\label{tab:synthetic2}
\end{table*}

\begin{figure*}[t]
\centering
\includegraphics[trim={0cm 0cm 0.0cm 0.0cm},clip,width=.4\textwidth]{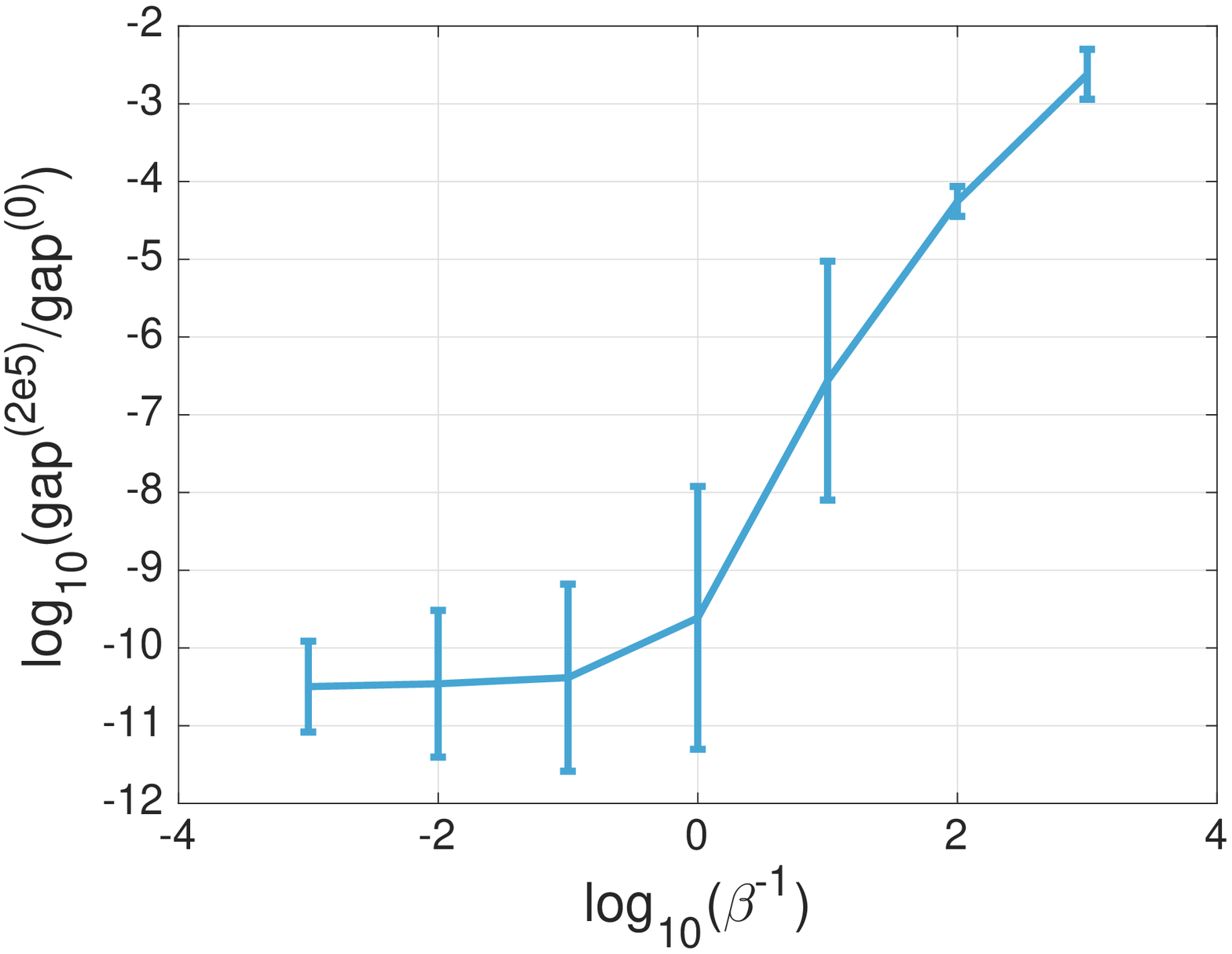}
\includegraphics[trim={0cm 0cm 0.0cm 0.0cm},clip, width=.4\textwidth]{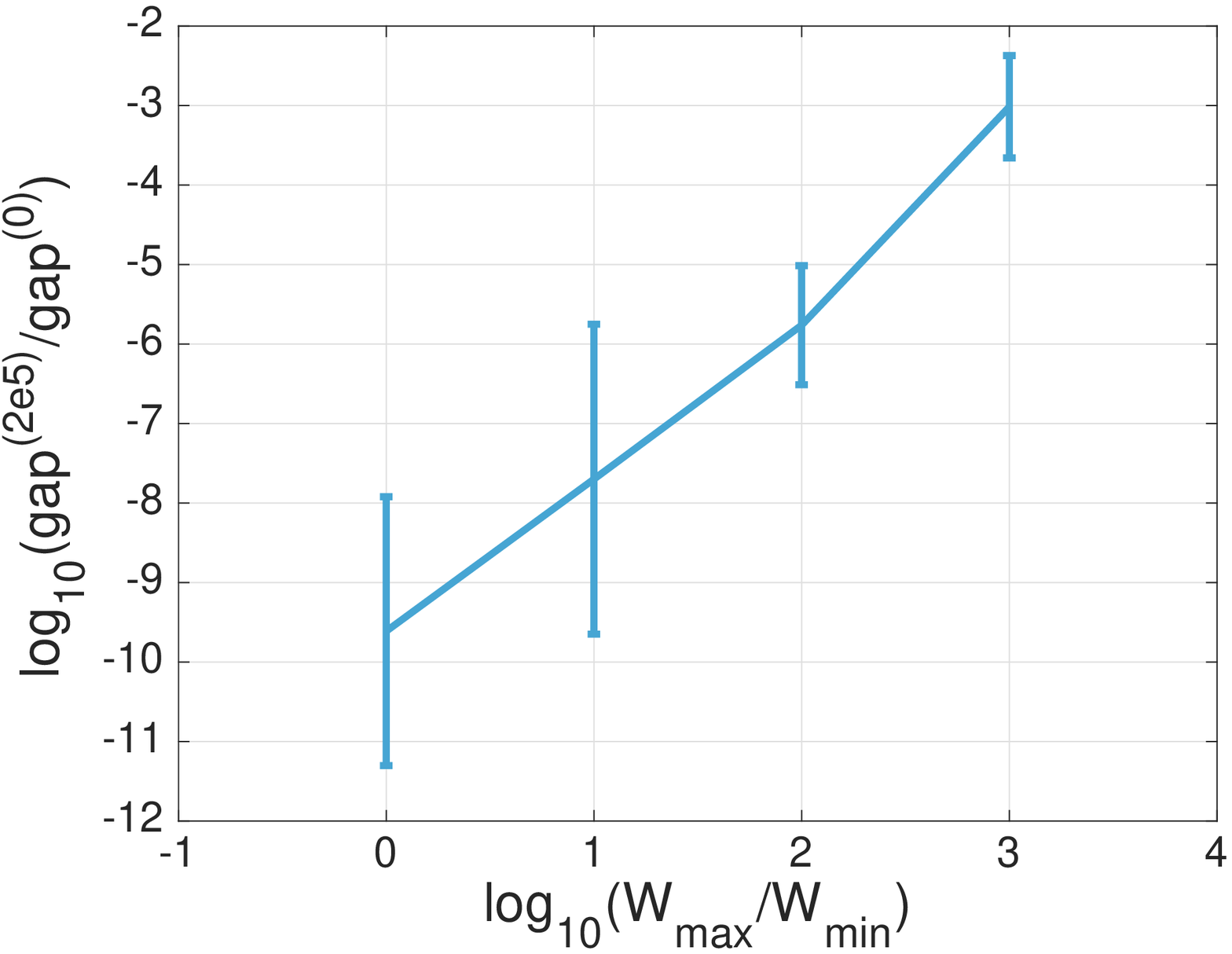}
\caption{Parameter sensitivity:  the rate of a primal gap of QRCD after $2\times 10^{5}$ iterations with respect to different choices of the parameters $\beta$ $\&$ $\max_{i,j\in [N]}W_{ii}/W_{jj}$.}
\label{fig:synthetic3}
\end{figure*}

We also evaluated the influence of parameter choices on the convergence of QRCD methods. According to Theorem~\ref{linearconv}, the required number of RCD iterations for achieving an $\epsilon$-optimal solution for~\eqref{expobj} is roughly $O(N^2R \max(1, 9/(2\beta))\max_{i, j\in[N]}W_{ii}/W_{jj}\log 1/\epsilon)$ (see Section~\ref{sec:para-dep}). We mainly focus on testing the dependence on the parameters $\beta$ and $\max_{i, j\in[N]}W_{ii}/W_{jj}$, as the term $N^2R$ is also included in the iteration complexity of DSFM and was shown to be necessary given certain submodular structures~\citep{li2018revisiting}. To test the effect of $\beta$, we fix $W_{ii} = 1$ for all $i$, and vary $\beta \in [10^{-3}, 10^{3}]$. To test the effect of $W$, we fix $\beta = 1$ and randomly choose half of the vertices and set their $W_{ii}$ values to lie in $\{1, 0.1, 0.01, 0.001\}$, and set the remaining ones to $1$. Figure~\ref{fig:synthetic3} shows our results. The logarithm of gap ratios is proportional to $\log\beta^{-1}$ for small $\beta,$ and $\log \max_{i, j\in[N]}W_{ii}/W_{jj}$, which is not as sensitive as predicted by Theorem~\ref{linearconv}. Moreover, when $\beta$ is relatively large ($>1$), the dependence on $\beta$ levels out. 
\begin{figure*}[t]
\centering
\includegraphics[trim={0cm 0cm 0cm 0},clip,width=.32\textwidth]{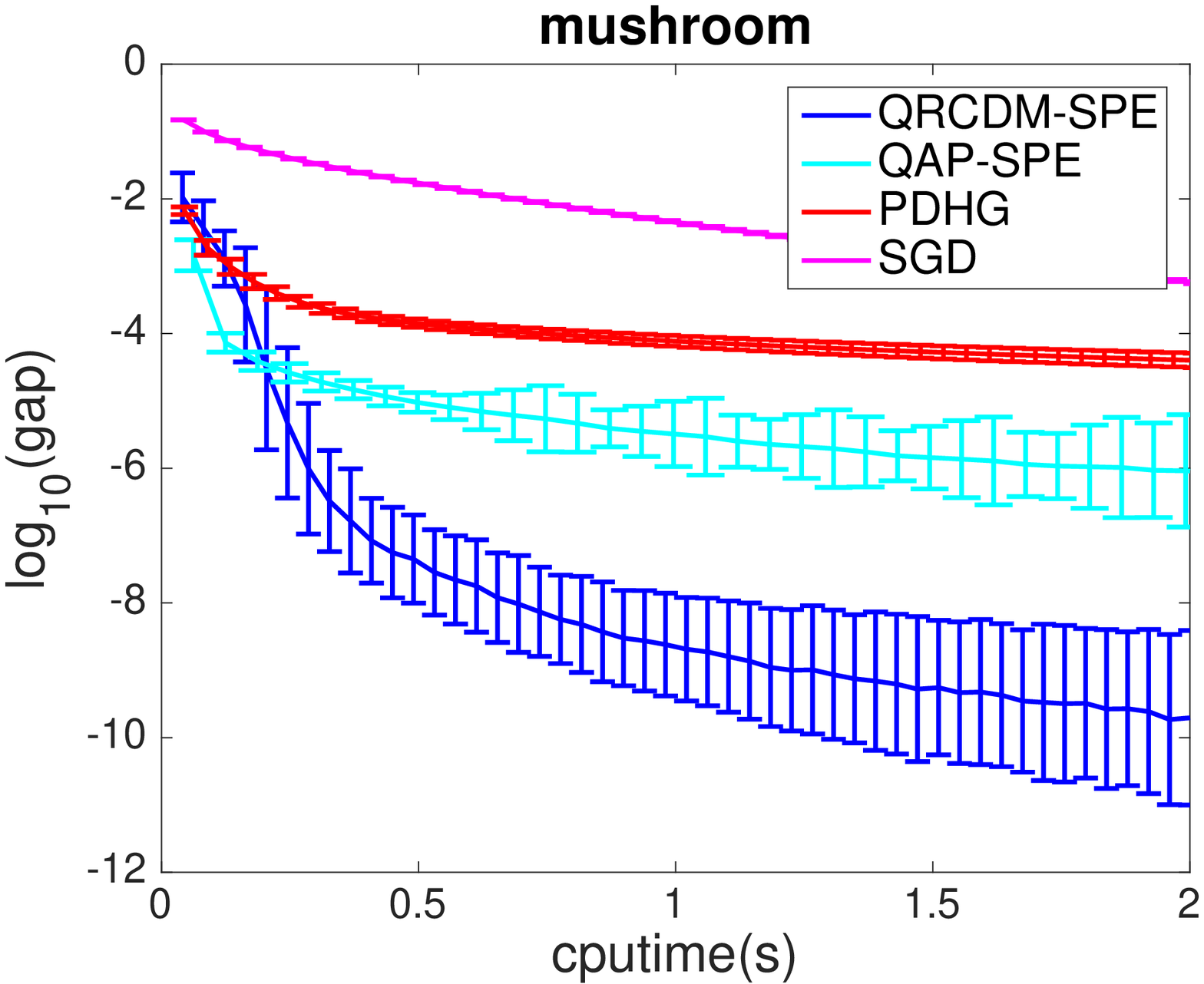}
\includegraphics[trim={0cm 0cm 0cm 0},clip,width=.32\textwidth]{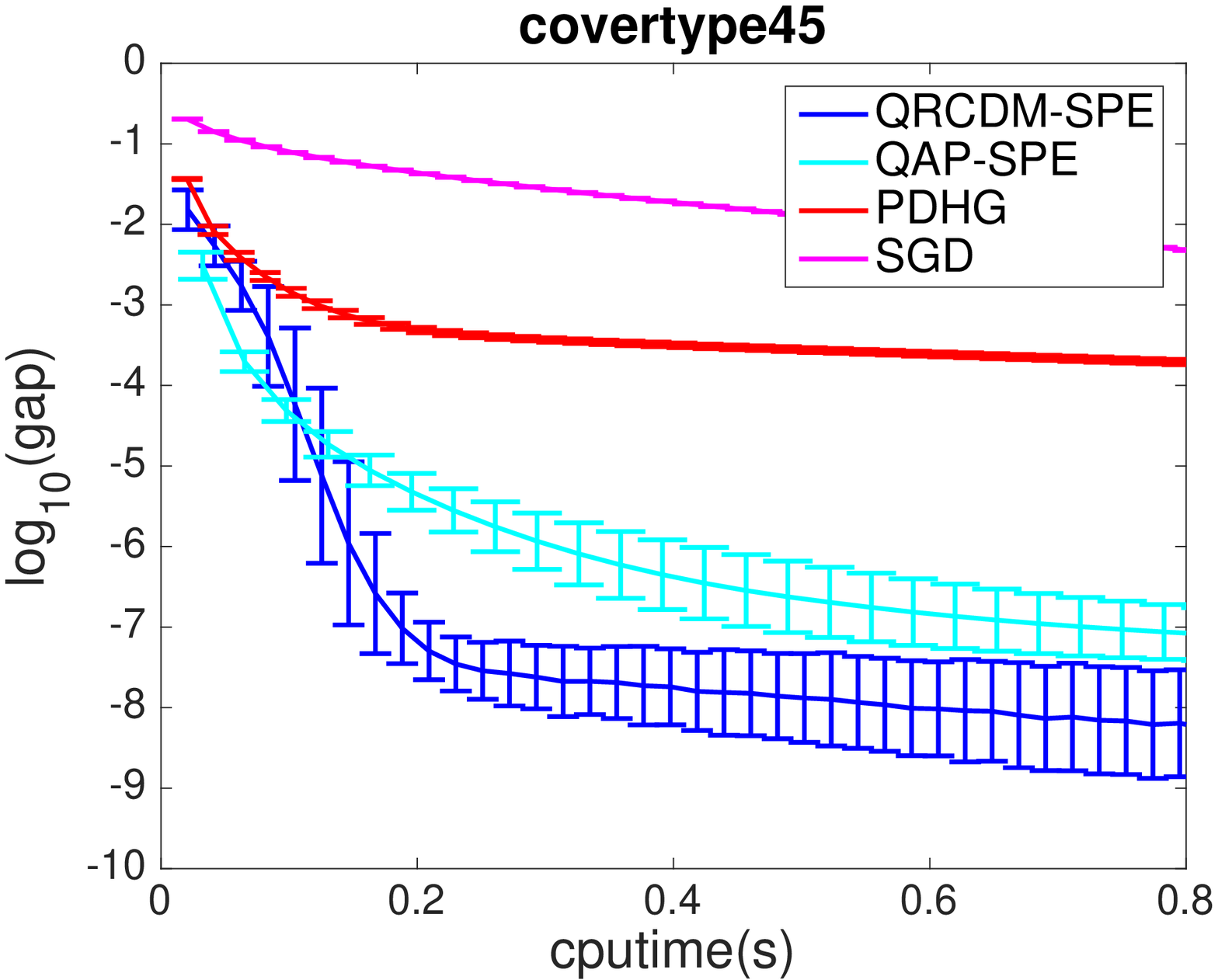}
\includegraphics[trim={0cm 0cm 0cm 0},clip,width=.32\textwidth]{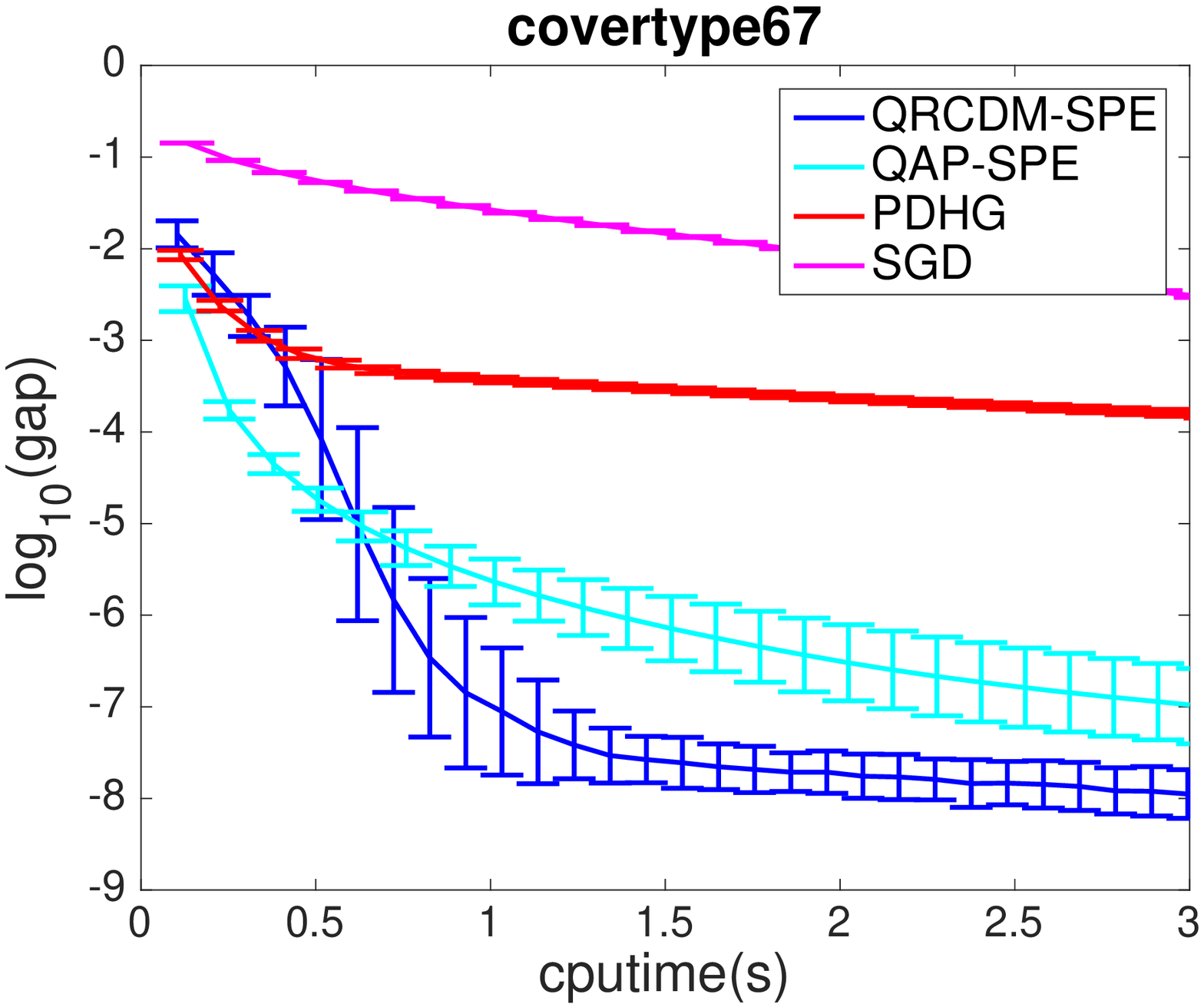}
\caption{Convergence of different solvers for QDFSM over three different real datasets.}
\label{fig:realdata}
\end{figure*}

\subsection{Real Data} 
We also evaluated the proposed algorithms on three UCI datasets: \emph{Mushroom}, \emph{Covertype45}, \emph{Covertype67}, used as standard test datasets for SSL on hypergraphs~\citep{zhou2007learning, hein2013total, zhang2017re}. 
Each dataset corresponds to a hypergraph model described in~\citep{hein2013total}: Entries correspond to vertices, while each categorical feature is modeled as one hyperedge; numerical features are first quantized into $10$ bins of equal size, and then mapped to hyperedges. Compared to synthetic data, in the real datasets, the size of most hyperedges is significantly larger ($\geq$ 1000) while the number of hyperedges is small ($\approx 100$). See Table~\ref{tab:dataset} for a detailed description of the parameters of the generated hypergraphs. Previous work has shown that smaller classification errors can be achieved by using QDSFM as the objective instead of DSFM or InvLap~\citep{hein2013total}. In our experiment, we focused on comparing the convergence of different solvers for QDSFM. We set $\beta = 100$ and $W_{ii} = 1,$ for all $i,$ and set the number of observed labels to $100$, which is a proper setting according to~\citep{hein2013total}. Figure~\ref{fig:realdata} shows the results. Again, the proposed QRCD-SPE and QAP-SPE methods both converge faster than PDHG and SGD, while QRCD-SPE performs the best. Note that we did not plot the results for QRCD-MNP and QRCD-FW as these methods converge extremely slowly due to the large sizes of the hyperedges. InvLap requires $22$, $114$ and $1802$ seconds to run on the Mushroom, Covertype45, and Covertype67 datasets, respectively. Hence, the method does not scale well with the problem size. 
\begin{table*}[t] 
\begin{tabular}{|p{1.6cm}<{\centering}|p{1.6cm}<{\centering}|p{2.6cm}<{\centering}|p{2.6cm}<{\centering}|}
\hline
Dataset & Mushroom & Covertype$45$ & Covertype$67$ \\
\hline
$N$ &  8124 & 12240 & 37877 \\
$R$ &  112 &  127 & 136 \\
$\sum_{r\in [R]} |S_r|$ & 170604 & 145999 & 451529 \\
\hline
\end{tabular}
\centering 
\caption{The UCI datasets used for experimental testing.}
\label{tab:dataset}
\end{table*}

\section{Detailed Proofs}
\subsection{Proofs of the Results Pertaining to the Dual Formulation}
%

\subsubsection{Proof of Lemma~\ref{dualform}}\label{proof:DF}
In all derivations that follow, we may exchange the order of minimization and maximization (i.e., $\min \max = \max \min$) due to Proposition 2.2~\citep{ekeland1999convex}. Plugging equation~\eqref{conjugate} into \eqref{QDSFM}, we obtain 
\begin{align*}
 &\min_x    \sum_{r\in [R]} [f_r(x)]^2 + \| x- a\|_{W}^2 \\
 =&  \min_x  \max_{\phi_r \geq 0, y_r\in \phi_r B_r} \sum_{r\in [R]}  \left[\langle y_r, x\rangle - \frac{\phi_r^2}{4} \right] + \| x- a\|_{W}^2 \\
  = &   \max_{\phi_r \geq 0, y_r\in \phi_r B_r} \min_x  \sum_{r\in [R]}  \left[\langle y_r, x\rangle - \frac{\phi_r^2}{4} \right] + \| x- a\|_{W}^2 \\
 = &\max_{\phi_r \geq 0, y_r\in \phi_r B_r}   -\frac{1}{4}\|\sum_{r\in [R]}  y_r - 2Wa\|_{W^{-1}}^2- \frac{1}{4} \sum_{r}\phi_r^2 + \|a\|_W^2.  
 \end{align*}
 By eliminating some constants, one obtains the dual~\eqref{CDform}. 
 
Next, we prove that the problem~\eqref{APform} is equivalent to~\eqref{CDform} when removing $\lambda_r$. 

First,~\eqref{APform} is equivalent to
 \begin{align*}
 & \min_{\phi_r \geq 0, y_r\in \phi_r B_r, \lambda_r} \max_{\lambda}  \sum_{r\in [R]} \left[\| y_r - \frac{\lambda_r}{\sqrt{R}}\|_{W^{-1}}^2 + \phi_r^2\right] + \left\langle \lambda, \sum_{r\in [R]} \lambda_r - 2Wa \right\rangle \\
= &   \min_{\phi_r \geq 0, y_r\in \phi_r B_r} \max_{\lambda} \min_{\lambda_r} \sum_{r\in [R]}\left[\| y_r -\frac{\lambda_r}{\sqrt{R}}\|_{W^{-1}}^2 +  \phi_r^2\right] + \left\langle \lambda, \sum_{r\in [R]} \lambda_r - 2Wa \right\rangle\\
= &  \min_{\phi_r \geq 0, y_r\in \phi_r B_r}  \max_{\lambda} \sum_{r\in [R]} \left[\frac{1}{4}\|\lambda\|_{W}^2+ \phi_r^2\right] + \left\langle \lambda, \sqrt{R}\sum_{r\in [R]} (y_r - \frac{1}{2} W\lambda_r)- 2Wa \right\rangle  \\
=& \min_{\phi_r \geq 0, y_r\in \phi_r B_r} \max_{\lambda} -\frac{R}{4} \|\lambda\|_{W}^2 + \sqrt{R}\left\langle \lambda, \sum_{r\in [R]} y_r - 2Wa\right\rangle + \sum_{r\in [R]}\phi_r^2\\
=& \min_{\phi_r \geq 0, y_r\in \phi_r B_r}   \|\sum_{r\in [R]} y_r - 2Wa\|_{W^{-1}}^2 + \sum_{r\in [R]}\phi_r^2,
   \end{align*}
which is equivalent to~\eqref{CDform}. 

\subsubsection{Proof of Lemma~\ref{submodularcone}}\label{proof:GS}

We start by recalling the following lemma from~\citep{li2018revisiting} that characterizes the geometric structure of the product base polytope.  
\begin{lemma}[\cite{li2018revisiting}]\label{submodularprop}
Assume that $W \in \mathbb{R}^{N\times N}$ is a positive diagonal matrix. Let $y \in \bigotimes_{r\in [R]} \phi'_r B_r$ and let 
$s$ be in the base polytope of the submodular function $\sum_{r} \phi'_r F_r $. Then, there exists a point 
$y'\in \bigotimes_{r\in [R]} \phi'_r B_r$ such that $\sum_{r\in [R]} y_r' = s$ and $\|y'-y\|_{I(W)}\leq \sqrt{\frac{\sum_{i=1}^{N}W_{ii}}{2}} \|\sum_{r\in[R]} y_r - s\|_1$. 
\end{lemma}
Lemma~\ref{submodularprop} cannot be used directly to prove Lemma~\ref{submodularcone}, since $y,\,y'$ 
are in different product base polytopes, $ \bigotimes_{r\in [R]} \phi_r B_r$ and $ \bigotimes_{r\in [R]} \phi_r' B_r$, respectively. 
However, the following lemma shows that one can transform $y$ to lie in the same base polytopes that contains $y'$.
\vspace{-0.2in}
\begin{lemma}\label{rescale}
For a given feasible point $(y,\phi)\in  \bigotimes_{r\in [R]} C_r$, and a nonnegative vector  $\phi' = (\phi_r) \in \bigotimes_{r\in [R]}\mathbb{R}_{\geq 0}$, one has 
\begin{align*}
\|\sum_{r\in [R]} y_r - s\|_{1} +  \frac{\rho}{2}\|\phi'-\phi\| \geq \|\sum_{r\in [R]} \frac{\phi_r'}{\phi_r}y_r - s\|_{1}.
\end{align*}
\end{lemma}
\begin{proof}
For all $r$, let $\tilde{y}_r = y_r/ \phi_r \in B_r$, and define a function that depends on $\phi$,
\begin{align*}
h(\phi) = \|\sum_{r\in [R]} y_r - s\|_{1} = \|\sum_{r\in [R]} \phi_r\tilde{y}_r  - s\|_{1}. 
\end{align*}
For all $\phi$ and $r$, $|\nabla_{\phi_r} h(\phi)|  \leq \|\tilde{y}_r\|_{1}.$
Therefore,
\begin{align*}
h(\phi') &= h(\phi) + \int_{t=0}^1 \langle \nabla h|_{\phi+t(\phi'-\phi)}, t(\phi'-\phi)\rangle dt \\
& \geq h(\phi) - \frac{\max_{t\in[0,1]}\|\nabla h|_{\phi+t(\phi'-\phi)}\|}{2} \|\phi'-\phi\| \\
& \geq h(\phi) - \frac{\max_{\tilde{y}_r\in B_r, \forall r}\sqrt{\sum_{r\in[R]} \|\tilde{y}_r\|_{1}^2} }{2}\|\phi'-\phi\| = h(\phi) - \frac{\rho}{2}\|\phi'-\phi\| .
\end{align*}
\end{proof}

Combining Lemma~\ref{rescale} with Lemma~\ref{submodularprop}, we can establish the claim of Lemma~\ref{submodularcone}. 

First, let $\rho(W^{(1)}) = \max_{y\in\bigotimes_{r\in[R]}B_r} \sqrt{\sum_{r\in[R]}\|y_r\|_{W_{r}}^2}$. Suppose that $y'\in \bigotimes_{r\in [R]} \phi'_rB_r$ is such that $\sum_{r\in [R]} y_r'= s$ and it minimizes $\sum_{r\in [R]}\|\frac{\phi_r'}{\phi_r}y_r - y_r'\|_{W^{(1)}}^2$. As $s$ lies in the base polytope of $\sum_{r\in[R]}\phi_r'F_r$, we know that such an $y'$ exists. Moreover, we have  
\begin{eqnarray*}
&&\|y - y'\|_{I(W^{(1)})} \leq \sum_{r\in [R]}\|y_r'-\frac{\phi_r'}{\phi_r} y_r\|_{W^{(1)}} +  \sum_{r\in [R]}\|y_r - \frac{\phi_r'}{\phi_r} y_r\|_{W^{(1)}} \\
&\stackrel{1)}{\leq}&\sqrt{\frac{\sum_{i\in[N]}W_{ii}^{(1)}}{2}}\|\sum_{r\in [R]} \frac{\phi_r'}{\phi_r} y_r- s\|_{1} + \rho(W^{(1)})\|\phi'-\phi\| \\
&\stackrel{2)}{\leq}& \sqrt{\frac{\sum_{i\in[N]}W_{ii}^{(1)}}{2}}\left[\|\sum_{r\in [R]} y_r- s\|_{1} + \frac{\rho}{2} \|\phi'-\phi\| \right] +  \rho(W^{(1)})\|\phi'-\phi\| \\
& =& \sqrt{\frac{\sum_{i\in[N]}W_{ii}^{(1)}\sum_{j\in[N]}1/W_{jj}^{(2)}}{2}} \|\sum_{r\in [R]} y_r- s\|_{W^{(2)}}  \\
&& \quad \quad + \left[\sqrt{\frac{\sum_{i\in [N]}W_{ii}^{(1)}}{2}}\frac{\rho}{2} + \rho(W^{(1)}) \right]\|\phi'-\phi\|  \\
& \stackrel{3)}{\leq} &\sqrt{\frac{\sum_{i\in [N]}W_{ii}^{(1)}\sum_{j\in[N]}1/W_{jj}^{(2)}}{2}} \|\sum_{r\in[R]} y_r- a\|_{W^{(2)}} + \frac{3}{2}\sqrt{\frac{\sum_{i\in[N]}W_{ii}^{(1)}}{2}}\rho\|\phi'-\phi\|, 
\end{eqnarray*}
where 1) follows from Lemma~\ref{submodularprop} and the definition of $\rho(W^{(1)})$, 2) is a consequence of Lemma~\ref{rescale} and 3) holds because 
\begin{align*}
\sum_{i\in [N]}W_{ii}^{(1)}\sum_{r\in[R]}\|y_r\|_1^2 \geq \sum_{r\in[R]}\|y_r\|_{W_{r}}^2.
\end{align*}

\subsection{Proof for the Linear Convergence Rate of the RCD Algorithm}\label{proof:RCD}

\subsubsection{Proof of Lemma~\ref{strongconv}}
If $(y^*, \phi^*)$ is the optimal solution, then it must hold that $\sum_{r\in [R]} y^*_r  = 2W(a - x^*)$ because of the duality between the primal and dual variables. Moreover, we also know that there must exist a nonempty collection $\mathcal{Y}$ of points $y'\in\bigotimes_{r\in [R]} \phi_r^* B_r$ such that  $\sum_{r\in [R]} y'_r =\sum_{r\in [R]} y^*_r$. Using Lemma~\ref{submodularcone}, and setting $\phi'=\phi^*, s = 2W(a - x^*),\;W^{(1)}, W^{(2)} = W^{-1}$, we can show that there exists some $y' \in \mathcal{Y}$ such that  
\begin{align*}
\|y-y'\|_{I(W^{-1})}^2 + \|\phi-\phi^*\|^2 \leq\mu(W^{-1}, W^{-1})\left[\|\sum_{r\in [R]}(y_r- y_r')\|_{W^{-1}}^2+  \|\phi-\phi^*\|^2\right].
\end{align*}
 According to the definition of $y^*$, one has $\|y-y^*\|_{I(W^{-1})}^2\leq\|y-y'\|_{I(W^{-1})}^2$ for $y'\in\mathcal{Y}$. This concludes the proof. 

\subsubsection{Proof of Theorem~\ref{linearconv}}

First, suppose that $(y^*, \phi^*) = \arg\min_{(y', \phi')\in \Xi} \|y^{(k)}-y'\|_{I(W^{-1})}^2 + \|\phi^{(k)}-\phi'\|^2$. Throughout the proof, for simplicity, we use $\mu$ to denote $\mu(W^{-1}, W^{-1})$. We start by establishing the following results.
\begin{lemma} It can be shown that the following inequalities hold:
\begin{align}
&\quad \langle \nabla g(y^{(k)}, \phi^{(k)}), (y^*- y^{(k)}, \phi^* - \phi^{(k)})\rangle\nonumber  \\
&\stackrel{1)}{\leq} g(y^*, \phi^*) - g(y^{(k)}, \phi^{(k)}) -\frac{1}{\mu}\left(\|y^{(k)}-y^*\|_{I(W^{-1})}^2 + \|\phi^{(k)}-\phi^*\|^2\right)  \nonumber \\
& \stackrel{2)}{\leq} \frac{2}{\mu + 1}\left[ g(y^*, \phi^*) - g(y^{(k)}, \phi^{(k)})- \|y^{(k)}-y^*\|_{I(W^{-1})}^2 - \|\phi^{(k)}-\phi^*\|^2\right].\label{sc4}
\end{align} 
\end{lemma}
\begin{proof}
From Lemma~\ref{strongconv}, we can infer that
\begin{align}
&\|\sum_{r\in [R]} (y_r- y_r^*)\|_{W^{-1}}^2+  \|\phi-\phi^*\|^2\geq  \frac{1}{\mu}\left[\|y-y^*\|_{I(W^{-1})}^2 + \|\phi-\phi^*\|^2\right] \Rightarrow \nonumber  \\
& g(y^*, \phi^*) \geq g(y^{(k)}, \phi^{(k)})+ \langle \nabla g(y^{(k)},  \phi^{(k)}), (y^*- y^{(k)}, \phi^* - \phi^{(k)})\rangle \nonumber \\
& \quad\quad\quad\quad\quad\quad\quad\quad\quad\quad+\frac{1}{\mu}\left[\|y-y^*\|_{I(W^{-1})}^2 + \|\phi-\phi^*\|^2\right],\label{sc1} \\
& g(y^{(k)}, \phi^{(k)})  \geq g(y^*, \phi^*) +\langle \nabla g(y^{*}, \phi^*), (y^{(k)}- y^{*}, \phi^{(k)} - \phi^*)\rangle\nonumber \\
& \quad\quad\quad\quad\quad\quad\quad\quad\quad\quad+\frac{1}{\mu}\left[\|y-y^*\|_{I(W^{-1})}^2 + \|\phi-\phi^*\|^2\right]. \label{sc2} 
\end{align} 
As $\langle \nabla g(y^{*}, \phi^*), (y^{(k)}- y^{*}, \phi^{(k)} - \phi^*)\rangle \geq 0$, \eqref{sc2} gives 
\begin{align}
g(y^*, \phi^*) -  g(y^{(k)}, \phi^{(k)})\leq  -\frac{1}{\mu}\left[\|y-y^*\|_{I(W^{-1})}^2 + \|\phi^{(k)}-\phi^*\|^2\right]. \label{sc3}
\end{align} 
Inequality~\eqref{sc1} establishes Claim 1) in~\eqref{sc4}. Claim 2) in~\eqref{sc4} follows from~\eqref{sc3}.
\end{proof}

The following lemma is a direct consequence of the optimality of $y_r^{(k+1)}$ as the projection $\Pi_{C_r}$. 
\begin{lemma}\label{optcond} One has
\begin{align*}
&\langle \nabla_r g((y^{(k)}, \phi^{(k)})), (y_r^{(k+1)} - y_r^*, \phi_r^{(k+1)} - \phi_r^*)\rangle \\
& \quad\quad \leq 2\langle y_r^{(k)} - y_r^{(k+1)}, y_r^{(k+1)} - y_r^*\rangle_{W^{-1}} + 2 \langle \phi_r^{(k)} - \phi_r^{(k+1)}, \phi_r^{(k+1)} - \phi_r^*\rangle. 
\end{align*}
\end{lemma}
The following lemma follows from a simple manipulation of the Euclidean norm.
\begin{lemma}\label{threepoints}
It holds that
\begin{align*} 
&\|y_r^{(k+1)} -  y_r^{(k)}\|_{W^{-1}}^2 + (\phi_r^{(k+1)} - \phi_r^{(k)})^2 \\
& \quad\quad  =\|y_r^{(k+1)} -  y_r^{*}\|_{W^{-1}}^2 + (\phi_r^{(k+1)} - \phi_r^{*})^2 + \|y_r^{(k)} -  y_r^{*}\|_{W^{-1}}^2 + (\phi_r^{(k)} - \phi_r^{*})^2\\
&\quad \quad\quad +  2\langle y_r^{(k+1)} - y_r^{*},  y_r^{*} - y_r^{(k)}\rangle_{W^{-1}} + 2\langle \phi_r^{(k+1)} - \phi_r^{*}, \phi_r^{*} - \phi_r^{(k)}\rangle\nonumber \\
 & \quad\quad =- \|y_r^{(k+1)} -  y_r^{*}\|_{W^{-1}}^2 - (\phi_r^{(k+1)} - \phi_r^{*})^2  +\|y_r^{(k)} -  y_r^{*}\|_{W^{-1}}^2 + (\phi_r^{(k)} - \phi_r^{*})^2\\
  &  \quad\quad\quad +  2\langle y_r^{(k+1)} - y_r^{*},  y_r^{(k+1)} - y_r^{(k)}\rangle_{W^{-1}} +2 \langle \phi_r^{(k+1)} - \phi_r^{*}, \phi_r^{(k+1)} - \phi_r^{(k)}\rangle.
  \end{align*}
  \end{lemma}

Our next task is to determine how the objective function decreases in each iteration. The following expectation is with respect to uniformly sampled values of $r\in [R]$ in the $k$-th iteration: 
\begin{align}
&\quad \mathbb{E}\left[g(y^{(k+1)}, \phi^{(k+1)})\right] - g(y^{(k)}, \phi^{(k)}) \nonumber\\
& = \mathbb{E}\left[\langle \nabla_{r} g(y^{(k)}, \phi^{(k)}), (y_r^{(k+1)}-y_r^{(k)}, \phi_r^{(k+1)} - \phi_r^{(k)})\rangle + \|y_r^{(k+1)}-y_r^{(k)}\|_{W^{-1}}^2 + (\phi_r^{(k+1)} - \phi_r^{(k)})^2\right]\nonumber \\
& = \mathbb{E}\left[ \langle \nabla_{r} g(y^{(k)}, \phi^{(k)}), (y_r^*- y_r^{(k)}, \phi_r^*- \phi_r^{(k)}  )\rangle+ \langle \nabla_{r} g(y^{(k)}, \phi^{(k)}), (y_r^{(k+1)}-y_r^*, \phi_r^{(k+1)} - \phi_r^*)\rangle  \right.\nonumber \\
&\left. \quad\quad\quad+ \|y_r^{(k+1)}-y_r^{(k)}\|_{W^{-1}}^2 + (\phi_r^{(k+1)} - \phi_r^{(k)})^2\right] \nonumber \\
& \stackrel{1)}{\leq} \mathbb{E} \left[ \langle \nabla_{r} g(y^{(k)}, \phi^{(k)}), (y_r^*- y_r^{(k)}, \phi_r^*- \phi_r^{(k)}  )\rangle - \|y_r^{(k+1)} - y_r^{*}\|_{W^{-1}}^2   +  \|y_r^{*} - y_r^{(k)}\|_{W^{-1}}^2 \right. \nonumber\\
&\left. \quad\quad\quad- (\phi_r^{(k+1)} - \phi_r^{*})^2 +(\phi_r^{*} - \phi_r^{(k)})^2 \right] \nonumber \\
&\stackrel{2)}{\leq} \frac{1}{R}\langle \nabla g(y^{(k)}, \phi^{(k)}), (y^{*}-y^{(k)}, \phi^{*} - \phi^{(k)})\rangle - \mathbb{E}\left[ \|y^{(k+1)} - y^{*}\|_{I(W^{-1})}^2 +\|\phi^{(k+1)} - \phi^{*}\|^2 \right] \nonumber\\ 
&\quad\quad\quad +  \|y^{(k)} - y^{*}\|_{I(W^{-1})}^2 + \|\phi^{(k)} - \phi^{*}\|^2 \\
& \stackrel{3)}{\leq}  \frac{2}{( \mu+ 1)R}\left[ g(y^*, \phi^*) - g(y^{(k)}, \phi^{(k)})\right] + \left(1-\frac{2}{(\mu + 1)R}\right)\left[\|y^{(k)}-y^*\|_{I(W^{-1})}^2 + \|\phi^{(k)}-\phi^*\|^2\right] \nonumber \\
&\quad\quad\quad  - \mathbb{E}\left[ \|y^{(k+1)} - y^{*}\|_{I(W^{-1})}^2 + \|\phi^{(k+1)} - \phi^{*}\|^2\right]. \label{onestep}
\end{align}
Here, $1)$ is a consequence of Lemma~\ref{optcond} and Lemma~\ref{threepoints}, $2)$ is due to $y^{(k+1)}_{r'} = y^{(k)}_{r'},\, \phi^{(k+1)}_{r'} = \phi^{(k)}_{r'}$ for $r'\neq r$, and $3)$ may be established from~\eqref{sc4}.

Equation \eqref{onestep} further leads to  
\begin{align*}
&\mathbb{E}\left[g(y^{(k+1)}, \phi^{k+1})- g(y^*, \phi^*) + d^2((y^{(k+1)}, \phi^{(k+1)}), \Xi) \right] \\
\leq &\quad \mathbb{E}\left[g(y^{(k+1)}, \phi^{k+1})- g(y^*, \phi^*) + \|y^{(k+1)} - y^{*}\|_{I(W^{-1})}^2 +  \|\phi^{(k+1)} - \phi^{*}\|_{I(W^{-1})}^2\right] \\
\leq & \quad \left[1-\frac{2}{(\mu + 1)R}\right] \mathbb{E}\left[g(y^{(k)}, \phi^{k})- g(y^*,\phi^*) + d^2((y^{(k)}, r^{(k)}), \Xi) \right] .
\end{align*}
The claimed proof follows by iterating the above derivations for all values of $k$. 
%

\subsection{Convergence Analysis of the AP Algorithm}
\subsubsection{Proof of Lemma~\ref{APequaldual}}\label{proof:APequaldual}

First, for $r\in [R]$, we define a diagonal matrix $A_r\in \mathbb{R}^{N\times N}$: $(A_r)_{ii} = 1,$ if $i\in S_r,$ and $0$ otherwise. Start with a Lagrangian dual of~\eqref{APcompactform}, and transform it according to
\begin{align*}
&\min_{(y,\phi)\in\bigotimes_{r\in [R]} C_r}\; \min_{ \Lambda: \lambda_{r,i} = 0,\, \forall(i, r)\not\in \mathbb{S}}\;\max_{\alpha\in \mathbb{R}^N}\quad \sum_{r\in[R]}\left[\| y_r - \lambda_r\|_{\Psi W^{-1}}^2 + \phi_r^2\right] + \langle \alpha, \sum_{r\in[R]} \lambda_r - 2Wa \rangle \\
=&\min_{(y,\phi)\in\bigotimes_{r\in [R]} C_r}\;\max_{\alpha\in \mathbb{R}^N}\; \min_{ \Lambda: \lambda_{r,i} = 0,\, \forall(i, r)\not\in \mathbb{S}}\quad \sum_{r\in[R]}\left[\| y_r - \lambda_r\|_{\Psi W^{-1}}^2 + \phi_r^2\right] + \langle \alpha, \sum_{r\in[R]} \lambda_r - 2Wa \rangle \\
\stackrel{1)}{=}& \min_{(y,\phi)\in\bigotimes_{r\in [R]} C_r}\; \max_{\alpha\in \mathbb{R}^N} \quad \sum_{r\in[R]}\left[ \frac{1}{4}\|A_r\Psi^{-1}W\alpha\|_{\Psi W^{-1}}^2 + \phi_r^2 \right] \\
& \quad\quad\quad\quad\quad\quad\quad\quad\quad\quad\quad\quad\quad\quad\quad\quad\quad\quad+ \left\langle \alpha, \sum_{r\in[R]} \left(y_r - \frac{1}{2}A_r\Psi^{-1}W\alpha\right) - 2Wa \right\rangle \\
\stackrel{2)}{=}&  \min_{(y,\phi)\in\bigotimes_{r\in [R]} C_r}\; \max_{\alpha\in \mathbb{R}^N} \quad - \frac{1}{4}\|\alpha\|_{W}^2 +\left\langle \alpha, \sum_{r\in[R]}y_r - 2Wa \right\rangle  + \sum_{r\in[R]} \phi_r^2  \\
= & \min_{(y,\phi)\in\bigotimes_{r\in [R]} C_r} \quad \|\sum_{r\in[R]}y_r - 2Wa \|_{W^{-1}}^2  + \sum_{r\in[R]} \phi_r^2.
\end{align*}
Here, 1) is due to $\lambda_r = y_r - \frac{1}{2} A_r\Psi^{-1}W\alpha$ while 2) is based on the fact that $\Psi = \sum_{r\in[R]} A_r$.
This establishes the claimed result.

\subsubsection{Proof of Lemma~\ref{upperboundkappa}}\label{proof:upperboundkappa}
Suppose that $(y, \phi) \in \mathcal{C}/\Xi$. Then,
\begin{align*}
[d_{\Psi W^{-1}}((y, \phi), \mathcal{Z})]^2 = &\min_{\lambda_r, \forall r\in[R]} \sum_{r}\left[\| y_r - \lambda_r\|_{\Psi W^{-1}}^2 + (\phi_r - \phi_r^*)^2\right]  \\
& \text{s.t.  $\sum_{r\in[R]} \lambda_r = 2W(a - x^*)$, $\lambda_{r,i} = 0, \forall r\in[R], i\not\in S_r$}.
\end{align*}
By eliminating $\lambda_r$, we arrive at 
$$[d_{\Psi W^{-1}}((y, \phi), \mathcal{Z})]^2 = \|\sum_{r}y_r - 2W(a- x^*) \|_{W^{-1}}^2  + \sum_{r}  (\phi_r - \phi_r^*)^2.$$ 
Based on Lemma~\ref{submodularcone}, we know that there exists a $(y', \phi') \in \Xi$ such that 
\begin{align*}
\mu(\Psi W^{-1},W^{-1}) \left[\|\sum_{r}(y_r - y_r')  \|_{W^{-1}}^2  + \sum_{r} (\phi_r - \phi_r')^2 \right]\geq \|y - y'\|_{\Psi W^{-1}}^2 + \sum_{r} (\phi_r - \phi_r')^2.
\end{align*}
As $\phi_r^*$ is the unique optima, it follows that $\phi_r^* = \phi_r'$. Also, $\sum_r y_r' = 2W(a-x^*)$. Moreover, as 
$$\|y - y'\|_{\Psi W^{-1}}^2 + \sum_{r} (\phi_r - \phi_r')^2 \geq [d_{\Psi W^{-1}}((y, \phi), \Xi)]^2,$$ 
according to the above definition, we have 
\begin{align*}
\frac{ [d_{\Psi W^{-1}}((y, \phi), \Xi)]^2}{[d_{\Psi W^{-1}}((y, \phi), \mathcal{Z})]^2} \leq \mu(\Psi W^{-1},W^{-1}).
\end{align*}
Next, suppose that $(y, \phi) \in \mathcal{Z}/\Xi$ and that 
$$(y', \phi')  = \arg\min_{(z,\psi) \in  \mathcal{C} } \|y-z\|_{I(\Psi W^{-1})}^2 + \|\phi -\psi\|^2,$$ 
$$(y'', \phi'') = \arg\min_{(z,\psi) \in \Xi} \|y'-z\|_{I(\Psi W^{-1})}^2 + \|\phi'- \psi\|^2.$$ 
Again, due to the definition of the distance $d_{\Psi W^{-1}}((y, \phi), \Xi)$, we have 
\begin{align}\label{ap1}
[d_{\Psi W^{-1}}((y, \phi), \Xi)]^2 \leq \|y-y''\|_{I(\Psi W^{-1})}^2 + \|\phi -\phi''\|^2.
\end{align}
Moreover, because of the way we chose $(y',\phi')$ and due to the fact that $\mathcal{C}$ is convex, we have
\begin{align}\label{ap2}
\|y-y''\|_{I(\Psi W^{-1})}^2 + \|\phi -\phi''\|^2 &\leq \|y-y'\|_{I(\Psi W^{-1})}^2 + \|\phi -\phi'\|^2 \nonumber \\ 
 &+ \|y'-y''\|_{I(\Psi W^{-1})}^2 + \|\phi' -\phi''\|^2.
\end{align}
Using Lemma~\ref{submodularcone}, we obtain 
\begin{align}\label{ap3}
\|y'-y''\|_{I(\Psi W^{-1})}^2 + \|\phi' -\phi''\|^2 \leq \mu(\Psi W^{-1}, W^{-1})( \|\sum_{r}(y_r'-y_r'')\|_{W^{-1}}^2 + \|\phi' -\phi''\|^2 ),
\end{align}
and we also have
\begin{align}\label{ap4}
 &\|\sum_{r}(y_r'-y_r'')\|_{W^{-1}}^2 =  \|\sum_{r}y_r' - 2W(a - x^*)\|_{W^{-1}}^2  \nonumber \\
 = & \|\sum_{r}(y_r'- y_r)\|_{W^{-1}}^2 \stackrel{1)}{\leq}  \|(y'- y)\|_{I(\Psi W^{-1})}^2,
\end{align}
where 1) follows from the Cauchy-Schwarz inequality over the entries $y_{r,i}, \, i\in S_r$. 

Finally, we set $\phi = \phi'' = \phi^*$ and combine \eqref{ap1}-\eqref{ap4} to obtain
\begin{align*}
[d_{\Psi W^{-1}}((y, \phi), \Xi)]^2 &\leq (1+ \mu(\Psi W^{-1}, W^{-1}))(\|y-y'\|_{I(\Psi W^{-1})}^2 + \|\phi -\phi'\|^2) \\
&= (1+ \mu(\Psi W^{-1}, W^{-1}))[d_{\Psi W^{-1}}((y, \phi), \mathcal{C})]^2,
\end{align*}
which concludes the proof. 

\subsection{Proof of Corollary~\ref{specialcase}}\label{proof:specialcase}
First, we establish an upper bound on $\rho$.
\begin{lemma}\label{boundrho}
Suppose that $D_{ii} = \sum_{r: r\in[R], i \in C_r} \max_{S\subseteq V} [F_r(S)]^2$. Then 
\begin{align*}
\rho^2 \leq 4\sum_{i\in[N]} D_{ii}.
\end{align*}
\end{lemma}
\begin{proof}
For each $r$, consider $y_r \in B_r$. Sort the entries of $y_r$ in descending order. Without loss of generality, assume that the ordering reads as $y_{r, i_1} \geq y_{r, i_2} \geq \cdots \geq y_{r, i_N}$. As $F_r([N]) = \sum_{j=1}^N y_{r, i_j} \geq 0$, we have $y_{r,i_{1}} \geq 0$. If $y_{r, i_N} \geq 0$, then $\|y_r\|_1 = \sum_{k=1}^ N y_{r, i_k} = F_r([N]) \leq \max_{S\subseteq [N]} F_r(S)$. If $y_{r, i_N} < 0$, there exists a $k'$ such that $y_{r, i_{k'}} \geq 0$ and $y_{r, i_{k'+1}} < 0$. Given the definition of $B_r$, we have 
$$\sum_{k=k'+1}^N |y_{r, i_{k}}| \leq \sum_{k=1}^{k'} |y_{r, i_{k}}| \leq  F_r(\{i_1, i_2,...,i_{k'}\}) \leq  \max_{S\subseteq [N]} F_r(S),$$ and thus $\|y_r\|_1 \leq 2 \max_{S\subseteq [N]} F_r(S)$. Moreover, as each variable in $[N]$ is incident to at least one submodular function, we have
\begin{align*}
\sum_{r\in[R]} \max_{S\subseteq [N]} [F_r(S)]^2 \leq  \sum_{i\in[N]} \sum_{r: i\in S_r} \max_{S\subseteq [N]} [F_r(S)]^2 \leq  \sum_{i\in[N]} D_{ii}.
\end{align*} 
Combining all of the above results, we obtain
\begin{align*}
\rho^2 = \sum_{r\in[R]} \max_{y_r\in B_r} \|y_r\|_1^2 \leq  4\sum_{r\in[R]}\max_{S\subseteq [N]} [F_r(S)]^2 \leq 4 \sum_{i\in[N]} D_{ii}.
\end{align*} 
\end{proof}

When $W=\beta D$, we have 
\begin{align*}
\sum_{i\in[N]} W_{ii} \sum_{j\in[N]}1/W_{jj} \leq N^2 \max_{i,j} \frac{W_{ii}}{W_{jj}} = N^2 \max_{i,j} \frac{D_{ii}}{D_{jj}},
\end{align*}
and 
\begin{align*}
\rho^2 \sum_{j\in[N]}1/W_{jj} \stackrel{1)}{\leq} 4 \sum_{i\in[N]} D_{ii} \sum_{j\in[N]}1/W_{jj} \leq  \frac{4}{\beta}N^2 \max_{i,j} \frac{D_{ii}}{D_{jj}},
\end{align*}
where $1)$ follows from Lemma~\ref{boundrho}. According to the definition of $\mu(W^{-1}, W^{-1})$ (see ~\eqref{defmu}),
\begin{align*}
\mu(W^{-1}, W^{-1}) \leq N^2 \max\{1, 9\beta^{-1}\}\max_{i,j} \frac{D_{ii}}{D_{jj}}.
\end{align*}
Similarly, we have 
\begin{align*}
\mu(\Psi W^{-1}, W^{-1}) \leq N^2 \max\{1, 9\beta^{-1}\}\max_{i,j} \frac{\Psi_{jj} D_{ii}}{D_{jj}}.
\end{align*}
This concludes the proof. 

\subsection{Convergence Analysis of the Conic MNP Algorithm}\label{proof:Wolfegaurantee}
\subsubsection{Preliminary Notation and Lemmas}
Given an active set $S = \{q_1, q_2,...\},$ and a collection of coefficients $\lambda= \{\lambda_1, \lambda_2, ...\},$ if $y = \sum_{q_i \in S} \lambda_i q_i$, we simply refer to $(y, S, \lambda)$ as a \emph{triple}. 

Define the following functions that depend on $S$ 
\begin{align*}
\tilde{h}(S, \lambda) &\triangleq h(\sum_{q_i \in S} \lambda_i q_i, \sum_{q_i \in S} \lambda_i), \\
\tilde{h}(S) &\triangleq \min_{\lambda:  \lambda_i \in \mathbb{R}, \forall i} \tilde{h}(S, \lambda), \\
\tilde{h}_+(S) &\triangleq \min_{\lambda: \lambda_i \geq 0, \forall i} \tilde{h}(S, \lambda).
\end{align*}
If the coefficients $\lambda$ ($ \lambda_i \in \mathbb{R}, \forall i$) minimize $\tilde{h}(S, \lambda)$, we call the corresponding triple $(y, S, \lambda)$ a \emph{good} triple. Given a triple $(y, S, \lambda)$, we also define 
$$\triangle(y, q) = - \langle y - a, q \rangle - \sum_{q_i\in S} \lambda_i,$$ 
$$\triangle (y) = \max_{q\in B} \triangle (y) = - \min_{q\in B}\langle y - a, q \rangle - \sum_{q_i\in S} \lambda_i,$$ 
and
$$\text{err}(y) = h(y, \sum_{q_i\in S} \lambda_i) - h^*.$$ 
The following lemma establishes the optimality of a good triple. 
\begin{lemma}\label{perpendicular}
 Given an active set $S$, consider the good triple $(y', S, \lambda')$ and an arbitrary triple $(y, S, \lambda)$. Then, 
 $$\langle y'-a, y\rangle_{\tilde{W}} + \langle\sum_{q_i\in S} \lambda_i', \sum_{q_i\in S} \lambda_i \rangle = \langle y'-a, y' - y\rangle_{\tilde{W}}+ \langle\sum_{q_i\in S} \lambda_i', \sum_{q_i\in S}(\lambda_i'- \lambda_i) \rangle = 0.$$
\end{lemma}
\begin{proof} 
Without loss of generality, assume that $\langle y'-a, y\rangle_{\tilde{W}} + \langle\sum_{q_i\in S} \lambda_i', \sum_{q_i\in S} \lambda_i \rangle < 0$. Then, for any $\epsilon>0$, $(y' + \epsilon y, S, \lambda'+ \epsilon \lambda)$ is also a triple. For $\epsilon$ sufficiently small, we have $\tilde{h}(S, \lambda' + \epsilon \lambda) < \tilde{h}(S, \lambda'),$ which contradicts the optimality of $(y', S, \lambda')$. Hence, 
$$\langle y'-a, y\rangle_{\tilde{W}} + \langle\sum_{q_i\in S} \lambda_i', \sum_{q_i\in S} \lambda_i \rangle = 0.$$ 
As $(y' - y , S, \lambda' - \lambda)$ is also a triple, repeating the above procedure we obtain the claimed equality. 
\end{proof}
\begin{lemma}\label{boundonr}
For any $\hat{y}\in B$, 
$$\arg\min_{\phi\geq 0} h(\phi\hat{y}, \phi)\leq \frac{\|a\|_{\tilde{W}}}{2}.$$ 
Moreover, $\phi^* \leq \frac{\|a\|_{\tilde{W}}}{2}$.
\end{lemma}
\begin{proof}
Given $\hat{y}$, the optimal value of $\phi$ satisfies 
$$\phi= \frac{\langle a, \hat{y} \rangle_{\tilde{W}}}{1+ \|\hat{y}\|_{\tilde{W}}^2} \leq \frac{\|a\|_{\tilde{W}}}{\|\hat{y}\|_{\tilde{W}} + \frac{1}{\|\hat{y}\|_{\tilde{W}}}}\leq \frac{\|a\|_{\tilde{W}}}{2}.$$ This establishes the claimed result.
\end{proof}

\begin{lemma}\label{errorconnect}
If $(y, S, \lambda)$ is a good triple, then $\triangle (y) \geq \frac{\text{err}(y)}{\|a\|_{\tilde{W}}}$,
where we recall that $\text{err}(y) = h(y, \sum_{q_i\in S} \lambda_i) - h^*.$
\end{lemma}
\begin{proof}
Recall that $(y^*, \phi^*)$ denotes the optimal solution. As $y^*/\phi^* \in B$, we have
\begin{align*}
\phi^*\triangle (y) &\geq - \langle y - a, y^* \rangle_{\tilde{W}} - \langle \phi^*, \sum_{q_i\in S} \lambda_i \rangle \\
&\stackrel{1)}{=}  - \langle y - a, y^* \rangle_{\tilde{W}} - \langle \phi^*, \sum_{q_i\in S} \lambda_i \rangle + \langle y-a, y \rangle_{\tilde{W}} + (\sum_{q_i\in S} \lambda_i)^2 \\
& = - \langle y - a, y^*-a \rangle_{\tilde{W}} - \langle \phi^*, \sum_{q_i\in S} \lambda_i \rangle + \langle y-a, y-a \rangle_{\tilde{W}} + (\sum_{q_i\in S} \lambda_i)^2 \\
&\stackrel{2)}{\geq} \frac{1}{2}\left[  \|y-a\|_{\tilde{W}}^2 + (\sum_{q_i\in S} \lambda_i)^2 -  \|y^*-a\|_{\tilde{W}}^2 + (\phi^*)^2\right] \\
& =  \frac{1}{2} \text{err}(y),
\end{align*} 
where $1)$ follows from Lemma~\ref{perpendicular}, while $2)$ is a consequence of the Cauchy-Schwarz inequality. By using the bound for $\phi^*$ described in Lemma~\ref{boundonr}, we arrive at the desired conclusion. 
\end{proof}

\subsubsection{Proof of Theorem~\ref{Wolfegaurantee}}

We only need to prove the following three lemmas which immediately give rise to Theorem~\ref{Wolfegaurantee}. Lemma~\ref{strictly} corresponds to the first statement of Theorem~\ref{Wolfegaurantee}. Combining Lemma~\ref{terminatingcond} and Lemma~\ref{decreaserate}, we can establish the second statement of Theorem~\ref{Wolfegaurantee}. This follows as we may choose $\epsilon = \delta \|a\|_{\tilde{W}}$. 

If Algorithm 2 terminates after less than $O(N\|a\|_{\tilde{W}}^2\max\{Q^2,1\}/\epsilon)$ iterations, then the condition of Lemma~\ref{terminatingcond} is satisfied and thus $\text{err}(y^{(k)})\leq \epsilon = \delta \|a\|_{\tilde{W}}$. If Algorithm 2 does not terminate after $O(N\|a\|_{\tilde{W}}^2\max\{Q^2,1\}/\epsilon)$ iterations, Lemma~\ref{decreaserate} guarantees $\text{err}(y^{(k)})\leq \epsilon = \delta \|a\|_{\tilde{W}}$.

\begin{lemma}\label{strictly}
At any point before Algorithm 2  terminates, one has 
$$h(y^{(k)}, \phi^{(k)}) \geq h(y^{(k+1)}, \phi^{(k+1)});$$ 
moreover, if $(y^{(k)}, \phi^{(k)})$ triggers a MAJOR loop, the claimed inequality is strict. 
\end{lemma}

The following lemma characterizes the pair $(y, \phi)$ at the point when the MNP method terminates.
\begin{lemma}\label{terminatingcond}
In the MAJOR loop at iteration $k$, if $\langle y^{(k)} - a, q^{(k)} \rangle_{\tilde{W}} + \phi^{(k)}\geq -\delta$, then $h(y^{(k)}, \phi^{(k)}) \leq h^* + \|a\|_{\tilde{W}}\delta$.
\end{lemma}

\begin{lemma}\label{decreaserate}
If Algorithm 2 does not terminate, then for any $\epsilon > 0$, one can guarantee that after $O(N\|a\|_{\tilde{W}}^2\max\{Q^2,1\}/\epsilon)$ iterations, Algorithm 2 generates a pair $(y, \phi)$ that satisfies $\text{err(y)} \leq \epsilon$.
\end{lemma}
The proofs of Lemma~\ref{strictly} and Lemma~\ref{terminatingcond} are fairly straightforward, while the proof of Lemma~\ref{decreaserate} is significantly more involved and postponed to the next section.  

\textbf{Proof of Lemma~\ref{strictly}}.
Suppose that $(y^{(k)}, \phi^{(k)})$ starts a MAJOR loop. As 
$$\langle y^{(k)} - a, q^{(k)} \rangle_{\tilde{W}} + \phi^{(k)}< -\delta,$$ 
we know that there exists some small $\varepsilon$ such that 
$$h(y^{(k)} + \varepsilon q^{(k)}, \phi^{(k)}+ \varepsilon) < h(y^{(k)}, \phi^{(k)}).$$ 
Consider next the relationship between $(z^{(k)}, \sum_{q_i\in S^{(k)}\cup\{q^{(k)}\}}\alpha_i)$ and  $(y^{(k)}, \phi^{(k)})$. 
Because of Step 6, we know that 
$$h(z^{(k+1)}, \sum_{q_i\in S^{(k)}\cup\{q^{(k)}\}}\alpha_i) = \tilde{h}(S^{(k)}\cup\{q^{(k)}\}) \leq h(y^{(k)} + \varepsilon q^{(k)}, \phi^{(k)}+ \varepsilon) < h(y^{(k)}, \phi^{(k)}).$$ 
If $(y^{(k+1)}, \phi^{(k+1)})$ is generated in some MAJOR loop, then 
$$(y^{(k+1)}, \phi^{(k+1)}) = (z^{(k+1)}, \sum_{q_i\in S^{(k)}\cup\{q^{(k)}\}}\alpha_i),$$ 
which naturally satisfies the claimed condition. If $(y^{(k+1)}, \phi^{(k+1)})$ is generated in some MINOR loop, then $(y^{(k+1)}, \phi^{(k+1)})$ lies strictly within the segment between $(z^{(k+1)}, \sum_{q_i\in S^{(k)}\cup\{q^{(k)}\}}\alpha_i)$ and $(y^{(k)}, \phi^{(k)})$ (because $\theta > 0$). Therefore, we also have $h(y^{(k+1)}, \phi^{(k+1)}) < h(y^{(k)}, \phi^{(k)})$. If $(y^{(k)}, \phi^{(k)})$ starts a MINOR loop, then we have 
$$h(z^{(k+1)}, \sum_{q_i\in S^{(k)}\cup\{q^{(k)}\}}\alpha_i) =\tilde{h}(S^{(k)}) \leq h(y^{(k)}, \phi^{(k)}),$$ 
once again due to Step 6. As $(y^{(k+1)}, \phi^{(k+1)})$ still lies within the segment between \\$(z^{(k+1)}, \sum_{q_i\in S^{(k)}\cup\{q^{(k)}\}}\alpha_i)$ and $(y^{(k)}, \phi^{(k)}),$ we have $h(y^{(k+1)}, \phi^{(k+1)}) \leq h(y^{(k)}, \phi^{(k)})$.

\textbf{Proof of Lemma~\ref{terminatingcond}}. Lemma~\ref{terminatingcond} is a corollary of Lemma~\ref{errorconnect}. To see why this is the case, observe that in a MAJOR loop, $(y^{(k)}, S^{(k)}, \lambda^{(k)})$ is always a good triple. Since 
$$\triangle(y^{(k)}) = -( \langle y^{(k)} - a, q^{(k)} \rangle_{\tilde{W}} + \phi^{(k)}) \leq \delta,$$ 
we have $\text{err}(y)\leq \delta \|a\|_{\tilde{W}}$. 

\subsubsection{Proof of Lemma~\ref{decreaserate}} The outline of the proof is similar to that of the standard case described in~\citep{chakrabarty2014provable}, and some results therein can be directly reused. The key step is to show that in every MAJOR loop $k$ with no more than one MINOR loop, the objective achieved by $y^{(k)}$ decreases sufficiently, as precisely described in Theorem~\ref{descent}. 
\begin{theorem}\label{descent}
For a MAJOR loop with no more than one MINOR loop, if the starting point is $y$, the starting point $y'$ of the next MAJOR loop satisfies 
\begin{align*}
\text{err}(y') \leq \text{err}(y) \left(1 - \frac{ \text{err}(y) }{\|a\|_{\tilde{W}}(Q^2+1)}\right).
\end{align*}
\end{theorem}
Based on this theorem, it is easy to establish the result in Lemma~\ref{decreaserate} by using the next lemma and the same approach as described in~\citep{chakrabarty2014provable}.   
\begin{lemma}[Lemma 1~\citep{chakrabarty2014provable}]
In any consecutive $3N+1$ iteratons, there exists at least one MAJOR loop with not more than one MINOR loop.
\end{lemma}
We now focus on the proof of Theorem~\ref{descent}. The next geometric lemma is the counterpart of Lemma 2~\citep{chakrabarty2014provable} for the conic case. 
\begin{lemma}\label{Wolfeconv2}
Given an active set $S$, consider a good triple $(y', S, \lambda')$ and an arbitrary triple $(y, S, \lambda)$. Then, for any $q\in\text{lin}(S)$ such that $\triangle(y,q) >0$, we have
   \begin{align*}
   \|y-a\|_{\tilde{W}}^2 + (\sum_{q_i\in S}\lambda_i)^2 -  \|y'-a\|_{\tilde{W}}^2 -  (\sum_{q_i\in S}\lambda_i')^2 \geq \frac{\triangle^2(y,q)}{\|q\|_{\tilde{W}}^2 + 1}.
   \end{align*}
\end{lemma}
\begin{proof}
First, we have
\begin{align*}
&\|y-a\|_{\tilde{W}}^2 + (\sum_{q_i\in S}\lambda_i)^2 -  \|y'-a\|_{\tilde{W}}^2 -  (\sum_{q_i\in S}\lambda_i')^2  \\
=& \|y- y'\|_{\tilde{W}}^2 + [\sum_{q_i\in S}(\lambda_i- \lambda_i')]^2 + 2\langle y - y', y' -a \rangle_{\tilde{W}} + 2 \langle \sum_{q_i\in S}(\lambda_i-\lambda_i'), \sum_{q_i\in S}\lambda_i'\rangle \\
\stackrel{1)}{=} & \|y- y'\|_{\tilde{W}}^2 + [\sum_{q_i\in S}(\lambda_i- \lambda_i')]^2,
\end{align*}
where 1) follows from Lemma~\ref{perpendicular}. Next, for any $\phi\geq 0$, 
\begin{align}
 &\|y- y'\|_{\tilde{W}}^2 + [\sum_{q_i\in S}(\lambda_i- \lambda_i')]^2 \nonumber \\
\stackrel{1)}{\geq}  & \frac{\left[\langle y- y',  y- \phi q\rangle_{\tilde{W}}+ \langle \sum_{q_i\in S}(\lambda_i- \lambda_i'),  \sum_{q_i\in S}\lambda_i - \phi\rangle\right]^2 }{\|y- \phi q\|_{\tilde{W}}^2 + (\sum_{q_i\in S}\lambda_i - \phi)^2} \nonumber\\
= & \frac{\left[\langle y -a,  y- \phi q\rangle_{\tilde{W}}+ \langle \sum_{q_i\in S}\lambda_i,  \sum_{q_i\in S}\lambda_i - \phi\rangle- \langle y' -a,  y- \phi q\rangle_{\tilde{W}}- \langle \sum_{q_i\in S}\lambda_i',  \sum_{q_i\in S}\lambda_i - \phi\rangle\right]^2 }{\|y- rq\|_{\tilde{W} }^2 + (\sum_{q_i\in S}\lambda_i - \phi)^2}  \nonumber\\ 
\stackrel{2)}{=} &\frac{\left[\langle y -a,  y- \phi q\rangle_{\tilde{W}}+ \langle \sum_{q_i\in S}\lambda_i,  \sum_{q_i\in S}\lambda_i - \phi\rangle\right]^2 }{\|y- \phi q\|_{\tilde{W} }^2 + (\sum_{q_i\in S}\lambda_i - \phi)^2}, \label{geo1}
 \end{align}
where $1)$ follows from the Cauchy-Schwarz inequality and $2)$ is due to Lemma~\ref{perpendicular}. Since $\triangle(y,q) > 0$, letting $\phi \rightarrow \infty$ reduces equation~\eqref{geo1} to $\frac{\triangle^2(y,q)}{\|q\|_{\tilde{W}}^2 + 1}$. 
\end{proof}

Next, using Lemma~\ref{Wolfeconv2}, we may characterize the decrease of the objective function for one MAJOR loop with no MINOR loop. As $(y^{(k+1)}, S^{(k+1)}, \lambda^{(k+1)})$ is a good triple and $y^{(k)}$ also lies in $\text{lin}(S)$, we have the following result.
\begin{lemma}\label{Wolfeconv3}
Consider some MAJOR loop $k$ without MINOR loops. Then 
\begin{align*}
\text{err}(y^{(k)})  - \text{err}(y^{(k+1)}) \geq\frac{\triangle^2(y^{(k)}, q^{(k)})}{Q^2 + 1} = \frac{\triangle^2(y^{(k)})}{Q^2 + 1}.
\end{align*}
\end{lemma}

Next, we characterize the decrease of the objective function for one MAJOR loop with one single MINOR loop.
\begin{lemma}\label{Wolfeconv4}
Consider some MAJOR loop $k$ with only one MINOR loop. Then 
\begin{align*}
\text{err}(y^{(k)})  - \text{err}(y^{(k+2)}) \geq\frac{\triangle^2(y^{(k)})}{Q^2 + 1}.
\end{align*}
\end{lemma}
\begin{proof}
Suppose that the active sets associated with $y^{(k)}, y^{(k+1)}, y^{(k+2)}$ are $S^{(k)}, S^{(k+1)}, S^{(k+2)},$ respectively. We know that within the MINOR loop, $(z^{(k)}, S^{(k)}\cup\{q^{(k)}\}, \alpha)$ is a good triple and $y^{(k+1)} = \theta y^{(k)} + (1- \theta) z^{(k)},$ for some $\theta\in [0,1]$. Let
\begin{align}\label{gap1}
A = \|y^{(k)}-a\|_{\tilde{W}}^2 + (\sum_{q_i\in S^{(k)}}\lambda_i^{(k)})^2 -  \|z^{(k)}-a\|_{\tilde{W}}^2 -  (\sum_{q_i\in S^{(k)}\cup\{q^{(k)}\}}\alpha_i \;)^2.
\end{align}
From Lemma~\ref{Wolfeconv2}, we have
\begin{align}\label{gap2}
A \geq \frac{\triangle^2(y^{(k)},q^{(k)})}{\|q^{(k)}\|_{\tilde{W}}^2 + 1} \geq \frac{\triangle^2(y^{(k)})}{Q^2 + 1}.
\end{align}
Note that both $S^{(k)}$ and $S^{(k+1)}$ are subsets of $S^{(k)}\cup\{q^{(k)}\}$. As $(z^{(k)}, S^{(k)}\cup\{q^{(k)}\}, \alpha)$ is a good triple, using Lemma~\ref{perpendicular}, we obtain 
$$\langle z^{(k)}- a,  z^{(k)} - y^{(k)} \rangle_{\tilde{W}}+\langle \sum_{q_i\in S^{(k)}\cup\{q^{(k)}\}} \alpha_i, \sum_{q_i\in S^{(k)}\cup\{q^{(k)}\}}  (\alpha_i- \lambda_i^{(k)}) \rangle = 0.$$ 
Furthermore, as $y^{(k+1)} = \theta y^{(k)} + (1- \theta) z^{(k)}= z^{(k)} - \theta(z^{(k)} - y^{(k)}) $ and $\lambda^{(k+1)} = \alpha  - \theta ( \alpha -\lambda^{(k)})$, we have
\begin{align*}
\|y^{(k)}-a\|_{\tilde{W}}^2 + (\sum_{i\in S^{(k)}}\lambda_i^{(k)})^2 -  \|y^{(k+1)}-a\|_{\tilde{W}}^2- (\sum_{i\in S^{(k+1)}}\lambda_i^{(k+1)})^2 = (1 - \theta^2)A.
\end{align*}
Moreover, we have 
\begin{align}\label{onestep}
\triangle(y^{(k+1)}, q^{(k)}) = \theta\triangle(y^{(k)}, q^{(k)}) + (1- \theta)\triangle(z^{(k)}, q^{(k)}) = \theta\triangle(y^{(k)}, q^{(k)}) = \theta\triangle(y^{(k)}),
\end{align}
which holds because $\triangle(y, q)$ is linear in $y$ and Lemma~\ref{perpendicular} implies $\triangle(z^{(k)}, q^{(k)}) = 0$. 

Since according to Lemma~\ref{strictly}, $h(y^{(k)}, \phi^{(k)})> h(y^{(k+1)}, \phi^{(k+1)})$, Lemma 4 in~\citep{chakrabarty2014provable} also holds in our case, and thus $q^{(k)} \in S^{(k+1)}$. To obtain $y^{(k+2)}$ and $S^{(k+2)}$, one needs to remove active points with a zero coefficients from $S^{(k+1)}$, so that $y^{(k+2)}$ once again belongs to a good triple with corresponding $S^{(k+1)}$. 
Based on~\ref{Wolfeconv2} and equation~\eqref{onestep}, we have the following result.
\begin{align}\label{gap3}
\|y^{(k+1)}-a\|_{\tilde{W}}^2 + (\sum_{q_i\in S^{(k+1)}}\lambda_i^{(k+1)})^2 -  &\|y^{(k+2)}-a\|_{\tilde{W}}^2 - (\sum_{q_i\in S^{(k+2)}}\lambda_i^{(k+2)})^2\\
 \geq& \frac{\triangle^2(y^{(k+1)}, q^{(k)})}{Q^2 + 1} = \frac{\theta^2\triangle^2(y^{(k)}) }{Q^2 + 1}.
\end{align}
Consequently, combining equations~\eqref{gap1},~\eqref{gap2} and~\eqref{gap3}, we arrive at
\begin{align*}
\text{err}(y^{(k)})  - \text{err}(y^{(k+2)})  \geq \frac{\triangle^2(y^{(k)})}{Q^2 + 1}.
\end{align*}
\end{proof}
And, combining Lemma~\ref{Wolfeconv3}, Lemma~\ref{Wolfeconv4} and Lemma~\ref{errorconnect} establishes Theorem~\ref{descent}.

\subsection{Convergence Analysis of the Conic Frank-Wolfe Algorithm} \label{proof:FWalgCR}
\subsubsection{Proof of Theorem~\ref{FWalgCR}} 

Using the same strategy as in the proof of Lemma~\ref{boundonr}, we may prove the following lemma. It hinges on the optimality assumption for $\gamma_1^{(k)} \phi^{(k)}  + \gamma_2^{(k)}$ in Step 3 of Algorithm 3.  
\begin{lemma}
In Algorithm 3, for all $k$, $\phi^{(k+1)}\leq \frac{\|a\|_{\tilde{W}}}{2}$.
\end{lemma}

Now, we prove Theorem~\ref{FWalgCR}. We write $y = \phi \hat{y},$ where $\hat{y}\in B$, so that
\begin{align*}
h(y,\phi)= h( \phi \hat{y}, \phi )  \geq h( y^{(k)}, \phi^{(k)})  + 2\langle y^{(k)} -a , \phi\hat{y} - y^{(k)}\rangle_{\tilde{W}} + \phi^2 - (\phi^{(k)})^2. 
\end{align*}
For both sides, minimize $(\hat{y},\phi)$ over $B \times [0, \frac{\|a\|_{\tilde{W}}}{2}]$, which contains $(y^*, \phi^*)$. Since $q^{(k)} = \arg\min_{q\in B} ( y^{(k)} -a , q \rangle_{\tilde{W}}$, we know that the optimal solutions satisfy $\hat{y} = q^{(k)} $ and $\phi =\tilde{\phi}^{(k)} = \min\{\max\{0, -\langle  y^{(k)}-a,  q^{(k)} \rangle \}, \frac{\|a\|_{\tilde{W}}}{2}\}$ of the RHS
\begin{align}\label{FWalgproof1}
h^*  =h(y^*, \phi^*)   \geq h( y^{(k)}, \phi^{(k)})   + 2\langle  y^{(k)} -a , \tilde{\phi}^{(k)}q^{(k)} - y^{(k)}\rangle_{\tilde{W}} + (\tilde{\phi}^{(k)})^2 - (\phi^{(k)})^2. 
\end{align}
Moreover, because of the optimality of $(\gamma_1^{(k)}\times\gamma_2^{(k)})\in\mathbb{R}_{\geq0}^2$, for arbitrary $\gamma \in [0,1]$ we have 
\begin{align*}
 &\quad h(\gamma_1^{(k)} y^{(k)} + \gamma_2^{(k)} q^{(k)},  \gamma_1^{(k)}\phi^{(k)}+ \gamma_2^{(k)})  \\
& \leq h((1-\gamma) y^{(k)}+ \gamma\tilde{\phi}^{(k)} q^{(k)}, (1-\gamma)\phi^{(k)} + \gamma\tilde{\phi}^{(k)}) \\
&= h(y^{(k)}, \phi^{(k)}) + 2\gamma \langle y^{(k)} - a,  \tilde{\phi}^{(k)}q^{(k)} - y^{(k)} \rangle_{\tilde{W}} \\
& +\gamma [(\tilde{\phi}^{(k)})^2 - (\phi^{(k)})^2] + \gamma^2 \|\tilde{\phi}^{(k)}q^{(k)} - y^{(k)}\|_{\tilde{W}}^2 + (\gamma^2 - \gamma) (\tilde{\phi}^{(k)} - \phi^{(k)})^2 \\
&\stackrel{1)}{\leq}  h(y^{(k)},  \phi^{(k)}) + \gamma ( h^* - h(y^{(k)},  \phi^{(k)})) + \gamma^2 \|\tilde{\phi}^{(k)}q^{(k)} - y^{(k)}\|_{\tilde{W}}^2 \\
&\stackrel{2)}{\leq} h(y^{(k)},  \phi^{(k)}) + \gamma ( h^* -  h(y^{(k)},  \phi^{(k)})) + \gamma^2 \|a\|_{\tilde{W}}^2 Q^2,
\end{align*}
where $1)$ follows from \eqref{FWalgproof1} and $\gamma^2 -\gamma \leq 0$, and $2)$ follows from 
$$\|\tilde{\phi}^{(k)}q^{(k)}- y^{(k)} \|_{\tilde{W}}^2  \leq 4\frac{\|a\|_{\tilde{W}}^2}{4} \max_{q\in B}\|q\|_{\tilde{W}}^2 =\|a\|_{\tilde{W}}^2 Q^2.$$ 

The claimed result now follows by induction. First, let $\hat{y}^* = y^*/\phi^{*},$ where $\phi^{*} = \frac{\langle \hat{y}^*, a \rangle_{\tilde{W}}}{1+\|\hat{y}^*\|_{\tilde{W}}^2}$. Then, 
\begin{align*}
h(y^{(0)},  \phi^{(0)}) - h^* \leq 2\langle y^*, a\rangle - ( y^*)^2 - (\phi^{*})^2 = \frac{\langle \hat{y}^*, a \rangle_{\tilde{W}}^2}{1+\|\hat{y}^*\|_{\tilde{W}}^2} \leq \|a\|_{\tilde{W}}^2 Q^2 .
\end{align*}
Suppose that $h(y^{(k)},  r^{(k)}) - h^* \leq \frac{2\|a\|_{\tilde{W}}^2 Q^2 }{k+2}.$ In this case, for all $\gamma\in[0,1]$, we have 
\begin{align*}
 h(y^{(t+1)},  \phi^{(t+1)}) - h^* \leq  (1-\gamma) [h(y^{(k)},  \phi^{(k)}) - h^*]  +  \gamma^2 \|a\|_{\tilde{W}}^2 Q^2 .
 \end{align*}
 By choosing $\gamma = \frac{1}{k+2}$, we obtain $h(y^{(k+1)},  \phi^{(k+1)}) - h^* \leq  \frac{2\|a\|_{\tilde{W}}^2 Q^2}{k+3},$ which concludes the proof. 

\subsection{Proofs for the Partitioning Properties of PageRank}\label{proof:PR}
\subsubsection{Proof of Lemma~\ref{lem:RPmixing1}}
Based on the definitions of $\mathcal{S}_j^p$ and $j\in V_p$, we have $\langle \nabla f_r(x),  1_{\mathcal{S}_j^p}\rangle = F_r(\mathcal{S}_j^p)$. 
Consequently, 
\begin{align}
& 2p(\mathcal{S}_j^p) - \sum_{r\in [R]} w_r f_r(x)\langle \nabla f_r(x),  1_{\mathcal{S}_j^p}\rangle  \nonumber \\
= & 2p(\mathcal{S}_j^p) - \sum_{r\in [R]} w_r f_r(x)F_r(\mathcal{S}_j^p) \nonumber \\
= & 2p(\mathcal{S}_j^p) - \sum_{r\in [R]} w_r  F_r(S_j^p)\max_{(i,j)\in S_r^{\uparrow}\times  S_r^{\downarrow}}(x_i -x_j) \nonumber\\
= & \left(I_p(\text{vol}(\mathcal{S}_j^p))  - \sum_{r\in [R]}  w_r  F_r(\mathcal{S}_j^p) \max_{i\in S_r^{\uparrow} } x_i  \right)  + \left( I_p(\text{vol}(\mathcal{S}_j^p)) +  \sum_{r\in [R]}  w_r F_r(\mathcal{S}_j^p)\min_{i\in S_r^{\downarrow}} x_i  \right) \nonumber\\
\leq & I_p\left(\text{vol}(\mathcal{S}_j^p) - \sum_{r\in [R]}  w_r F_r(\mathcal{S}_j^p)  \right) +  I_p\left(\text{vol}(\mathcal{S}_j^p) + \sum_{r\in [R]}  w_r F_r(\mathcal{S}_j^p)\right) \nonumber\\ 
= & I_p\left(\text{vol}(\mathcal{S}_j^p) - \text{vol}(\partial \mathcal{S}_j^p)  \right) +  I_p\left(\text{vol}(\mathcal{S}_j^p) + \text{vol}(\partial \mathcal{S}_j^p). \right) \label{eq:PRedge}
\end{align}
Using equation~\eqref{eq:PRpoint}, we have 
\begin{align*}
p(\mathcal{S}_j^p) &= \frac{\alpha}{2-\alpha}  p_0(\mathcal{S}_j^p) + \frac{2- 2\alpha}{2-\alpha}\left\{\frac{1}{2}\left[M(p) - p\right](\mathcal{S}_j^p) + p(\mathcal{S}_j^p)  \right\}\\
&  \stackrel{1)}{=} \frac{\alpha}{2-\alpha}  p_0(\mathcal{S}_j^p) + \frac{2- 2\alpha}{2-\alpha}\left[-\frac{1}{2} \sum_{r\in[R]} f_r(x)\langle \nabla f_r(x), 1_{\mathcal{S}_j^p}\rangle  + p(\mathcal{S}_j^p) \right] \\
&  \stackrel{2)}{\leq} \frac{\alpha}{2-\alpha}  p_0(\mathcal{S}_j^p) +  \frac{1- \alpha}{2-\alpha} \left[I_p\left(\text{vol}(\mathcal{S}_j^p) - \text{vol}(\partial \mathcal{S}_j^p)  \right) +  I_p\left(\text{vol}(\mathcal{S}_j^p) + \text{vol}(\partial \mathcal{S}_j^p) \right)\right],
\end{align*}
where 1) is due to Lemma~\ref{lem:PR1} and 2) is due to equation~\eqref{eq:PRedge}. This proves the first inequality. By using the concavity of $I_p(\cdot)$, we also have 
\begin{align*}
I_p(\text{vol}(\mathcal{S}_j^p)) \leq  p_0(\mathcal{S}_j^p) \leq  I_{p_0}(\text{vol}(\mathcal{S}_j^p)).
\end{align*}
Moreover, as $I_p$ is piecewise linear, the proof follows. 

\subsubsection{Proof of Theorem~\ref{thm:PRmixing}}
This result can be proved in a similar way as the corresponding case for graphs~\citep{andersen2006local}, by using induction. 

Define $\bar{k} = \min\{k, m-k\}$, $d_{\min} = \min_{i: (p_0)_i > 0}d_i$ and 
\begin{align*}
I^{(t)}(k)= \frac{k}{m} + \frac{\alpha}{2-\alpha}t + \sqrt{\frac{\bar{k}}{d_{\min}}} \left(1- \frac{\Phi_p^2}{8}\right)^t.
\end{align*}
When $t = 0$, $I_p(k)\leq I_{p_0}(k)$ holds due to Lemma~\ref{lem:RPmixing1}. As $I^{(0)}(m)=I^{(0)}(d_{\min})=1$, for $k\in[d_{\min}, m]$, we have
\begin{align*}
I_{p_0}(k) \leq 1\leq I^{(0)}(k).
\end{align*} 
Since for $k\in [0, d_{\min}]$, $I_{p_0}(k)$ is linear, we also have $I_{p_0}(0) = 0 \leq  I^{(t)}(0)$. Hence, for $k\in [0, d_{\min}]$, it also holds that $I_{p_0}(k)\leq I^{(0)}(k)$.

Next, suppose that for step $t$, $I_p(k)\leq I^{(t)}(k)$ holds. We then consider the case $t+1$. For $k= \text{vol}(\mathcal{S}_j^p)$, Lemma~\ref{lem:RPmixing1} indicates that
\begin{align*}
I_p(k)&\leq \frac{\alpha}{2-\alpha} I_{p_0}(k) + \frac{1- \alpha}{2-\alpha}[I_p(k -  \bar{k}\Phi(\mathcal{S}_j^p) )+I_p(k + \bar{k}\Phi(\mathcal{S}_j^p)) ] \\
& \stackrel{1)}{\leq}  \frac{\alpha}{2-\alpha} I_{p_0}(k) + \frac{1- \alpha}{2-\alpha}\left[\frac{2k}{m} + \frac{2\alpha}{2-\alpha}t + \left(\sqrt{\frac{\overline{k-\bar{k}\Phi(\mathcal{S}_j^p)}}{d_{\min}}} +  \sqrt{\frac{\overline{k+\bar{k}\Phi(\mathcal{S}_j^p)}}{d_{\min}}} \right) \left(1- \frac{\Phi_p^2}{8}\right)^t\right] \\
& \stackrel{2)}{\leq} \frac{k}{m} +  \frac{\alpha}{2-\alpha}(t+1) + \frac{1- \alpha}{2-\alpha}\left(\sqrt{\frac{\overline{k- \bar{k}\Phi(\mathcal{S}_j^p)}}{d_{\min}}} +  \sqrt{\frac{\overline{k+ \bar{k}\Phi(\mathcal{S}_j^p)}}{d_{\min}}} \right) \left(1- \frac{\Phi_p^2}{8}\right)^t \\
&  \stackrel{3)}{\leq} \frac{k}{m} +  \frac{\alpha}{2-\alpha}(t+1) + \frac{1- \alpha}{2-\alpha}\left(\sqrt{\frac{\bar{k}- \bar{k}\Phi(\mathcal{S}_j^p)}{d_{\min}}} +  \sqrt{\frac{\bar{k}+ \bar{k}\Phi(\mathcal{S}_j^p)}{d_{\min}}} \right) \left(1- \frac{\Phi_p^2}{8}\right)^t  \\
&  \stackrel{4)}{\leq}  \frac{k}{m} +  \frac{\alpha}{2-\alpha}(t+1) + \frac{2- 2\alpha}{2-\alpha}\sqrt{\frac{\bar{k}}{d_{\min}}} \left(1- \frac{\Phi_p^2}{8}\right)^{t+1} \leq I^{(t+1)}(k),
\end{align*} 
where 1) follows from Lemma~\ref{lem:RPmixing1}; 2) is due to $I_{p_0}(k)\leq 1$; 3) can be verified by considering two cases separately, namely $k\leq \frac{m}{2}$ and $k\geq \frac{m}{2}$; and 4) follows from $\Phi(\mathcal{S}_j^p) \geq \Phi_p$ and the Taylor-series expansion of $\sqrt{x \pm \phi x}$.

At this point, we have shown that $I_p(k)\leq I^{(t)}(k)$ for all $k = \text{vol}(\mathcal{S}_{j}^p)$. Since $\{\text{vol}(\mathcal{S}_{j}^p)\}_{j\in V_p}$ covers all break points of $I_p$ and since $I^{(t)}$ is concave, we have that for all $k\in [0, m]$, it holds that $I_p(k)\leq I^{(t)}(k)$. This concludes the proof. 

\subsubsection{Proof of Theorem~\ref{thm:PRpartition}}
First, we prove that under the assumptions for the vertex-sampling probability $P$ in the statement of the theorem, $pr(\alpha, 1_i)(S)$ can be lower bounded as follows. 
\begin{lemma}\label{lem:PRlowerbound}
If a vertex $v\in S$ is sampled according to a distribution $P$ such that 
$$\mathbb{E}_{i\sim P}[pr(\alpha, 1_i)(\bar{S})] \leq c \; pr(\alpha, \pi_S)(\bar{S}),$$
where $c$ is a constant, then with probability at least $\frac{1}{2}$, one has
\begin{align*}
pr(\alpha, 1_i)(S) \geq  1- \frac{c\, \Phi(S)}{4\alpha}.
\end{align*}
\end{lemma}
\begin{proof}
Let $p = pr(\alpha, \pi_S)$. Then,
\begin{align*}
\alpha p(\bar{S}) &\stackrel{1)}{=}  \alpha  \pi_S(\bar{S}) +(1-\alpha)\left[M(p) - p\right](\bar{S}) \nonumber\\
&  \stackrel{2)}{=}- \sum_{r\in[R]} w_r f_r(x)\langle \nabla f_r(x), 1_{\bar{S}}\rangle   \nonumber\\
& \stackrel{3)}{\leq}  \sum_{r\in[R]}w_r f_r(x) F(S) \leq  \sum_{r\in[R]}w_r F(S)  \max_{i\in S_r^{\uparrow}}x_i \nonumber\\
& \leq I_p(\sum_{r\in[R]}w_r F(S)) = I_p(\text{vol}(\partial S))\nonumber \\
&\stackrel{4)}{\leq} I_{\pi_S}(\text{vol}(\partial S)) = \Phi(S), 
\end{align*}
where 1) is a consequence of Equation~\eqref{eq:PRpoint}; 2) follows from Lemma~\ref{lem:PR1}; 3) holds because for any $x\in \mathbb{R}^N$, $\langle \nabla f_r(x), 1_{\bar{S}}\rangle \geq - F(S)$; and 4) is a consequence of Lemma~\ref{lem:RPmixing1}. Hence, 
\begin{align*}
\mathbb{E}_{v\sim P}[pr(\alpha, 1_i)(\bar{S})] \leq \frac{c}{8} \, pr(\alpha, \pi_S)(\bar{S}) \leq  \frac{c}{8}\, \frac{\Phi(S)}{\alpha}.
\end{align*}
Moreover, sampling according to $v\sim P$ and using Markov's inequality, we have 
\begin{align*}
\mathbb{P}\left[pr(\alpha, 1_i)(\bar{S}) \geq \frac{c}{4}\, \frac{\Phi(S)}{\alpha}\right] \leq \frac{\mathbb{E}_{v\sim P}[pr(\alpha, 1_i)(\bar{S})] }{\frac{c}{4}\frac{\Phi(S)}{\alpha}} \leq \frac{1}{2},
\end{align*}
which concludes the proof.
\end{proof}
As $\Phi(S)\leq \frac{\alpha}{c}$, we have $\mathbb{P}\left[pr(\alpha, 1_i)(S) \geq \frac{3}{4}\right] \geq \frac{1}{2}$. By combining this lower bound on $pr(\alpha, 1_i)(S)$ with the upper bound of Theorem~\ref{thm:PRmixing}, we have 
\begin{align*}
\frac{3}{4} \leq pr(\alpha, 1_i)(S)  \leq I_{pr(\alpha, 1_i)}(\text{vol}(S)) \leq \frac{1}{2} + \frac{\alpha}{2-\alpha} t + \sqrt{\frac{\text{vol}(S)}{d_i}}\left(1- \frac{\Phi_{pr(\alpha, 1_i)}^2}{8}\right)^t. 
\end{align*}
Next, we choose $t= \lceil \frac{8}{\Phi_{pr(\alpha, 1_i)}^2} \ln 8\sqrt{\frac{\text{vol}(S)}{d_i}}\rceil$. Then, the above inequality may be transformed into
\begin{align*}
\frac{3}{4} \leq \frac{1}{2} + \frac{\alpha}{2-\alpha} \lceil \frac{8}{\Phi_{pr(\alpha, 1_i)}^2} \ln 8\sqrt{\frac{\text{vol}(S)}{d_i}}\rceil + \frac{1}{8} \leq \frac{5}{8} + \frac{\alpha}{2-\alpha}\frac{8}{\Phi_{pr(\alpha, 1_i)}^2} \ln 10\sqrt{\frac{\text{vol}(S)}{d_i}}.
\end{align*}
Therefore,
\begin{align*}
\Phi_{pr(\alpha, 1_i)} \leq \sqrt{32\alpha \ln \frac{100\text{vol}(S)}{d_i} }.
\end{align*}

\subsection{Analysis of the Parameter Choices for Experimental Verification}\label{sec:para-dep}
Let $x' = W^{-1/2}x$ and $a' = W^{-1/2}a$. We can transform the objective~\eqref{expobj} into a standard QDSFM problem,
\begin{align*}
\beta \|x'- a'\|_W^2 + \sum_{r\in[R]}\max_{i,j\in S_r}(x_i' - x_j')^2.
\end{align*}
From Theorem~\ref{linearconv}, to achieve an $\epsilon$-optimal solution, one requires $O(R\mu(\beta^{-1}W^{-1}, \beta^{-1}W^{-1})\log \frac{1}{\epsilon})$ iterations in the RCD algorithm (Algorithm 1). According to the particular settings for the experiment (undirected unweighted hypergraphs), we have 
\begin{align}\label{exp-rho}
\rho^2 = \max_{y_r\in B_r, \forall r} \sum_{r\in [R]}\|y_r\|_1^2=  \max_{y_r\in B_r, \forall r}\sum_{r\in [R]}2= 2R.
\end{align}
From the definition of $\mu$~\eqref{defmu}, we may rewrite $\mu(\beta^{-1}W^{-1}, \beta^{-1}W^{-1})$ as
\begin{align}
\mu(\beta^{-1}W^{-1}, \beta^{-1}W^{-1}) &\stackrel{1)}{=} \max \left\{ \frac{N^2}{2}\left(\max_{i,j\in[N]} \frac{W_{ii}}{W_{jj}} + 1\right),  \frac{9}{4}\rho^2 N\beta^{-1}\max_{j\in[N]}\frac{1}{W_{jj}}\right\} \nonumber \\
& \stackrel{2)}{=} \max \left\{ \frac{N^2}{2} \left(\max_{i,j\in[N]} \frac{W_{ii}}{W_{jj}} + 1\right),  \frac{9}{2}\beta^{-1} NR\max_{j\in[N]}\frac{1}{W_{jj}}\right\}\nonumber \\
& \stackrel{3)}{=} \max \left\{ \frac{N^2}{2} \left(\max_{i,j\in[N]} \frac{W_{ii}}{W_{jj}} + 1\right),  \frac{9}{2}\beta^{-1} N^2\max_{i,j\in[N]}\frac{W_{ii}}{W_{jj}}\right\}, \label{exp-para}
\end{align}
where $1)$ holds because half of the values of $W_{ii}$ are set to 1 and the other half to a value in $\{1, 0.1, 0.01, 0.001\}$, $2)$ follows from \eqref{exp-rho} and $3)$ is due to the particular setting $N=R$ and $\max_{i\in[N]}W_{ii} = 1$. Equation~\eqref{exp-para} may consequently be rewritten as $$O(N^2 \max(1, 9/(2\beta))\max_{i, j\in[N]}W_{ii}/W_{jj}),$$ which establishes the claimed statement. 

\section{Acknowledgement}
The authors would like to thank the reviewers for their insightful suggestions to improve the quality of the manuscript. The authors also gratefully acknowledge funding from the NSF CIF 1527636  and the NSF Science and Technology Center (STC) at Purdue University, Emerging Frontiers of Science of Information, 0939370. 


\end{document}